\documentclass[11pt]{article}
\usepackage{xspace}
\usepackage{graphicx}
\usepackage{amsmath}
\usepackage{bbm}
\usepackage{natbib}
\usepackage{amssymb}
\usepackage{amsfonts}
\usepackage{amsthm}
\usepackage{epsfig}
\usepackage{mathrsfs}
\usepackage{multirow}
\usepackage[usenames,dvipsnames]{color}
\usepackage{setspace}
\usepackage{algorithm, algorithmic} 
\usepackage{multicol}
\usepackage[position=top]{subcaption}
\usepackage[margin=2.54cm]{geometry}
\usepackage{lipsum}
\usepackage{bigints}
\usepackage{framed}
\usepackage{colortbl}
\usepackage{calligra}
\usepackage{mathrsfs}
\usepackage{tabularx,hhline}
\usepackage[anchorcolor=blue,
citecolor=blue,
linkcolor=black,
colorlinks=true,
]{hyperref}

\usepackage{caption}
\usepackage{color}
\usepackage{color,soul}
\usepackage{enumitem}
\usepackage{cleveref}
\usepackage{dirtytalk}
\usepackage{comment}
\usepackage{booktabs}
\usepackage{pifont}

\newcommand{\mathc}[1]{\ensuremath{\mathcal{#1}}}
\newcommand{\mb}[1]{\mathbb{#1}}

\newcommand{\htt}{\hat}
\newcommand{\ov}{\text{we}}
\newcommand{\bi}{\begin{itemize}}
\newcommand{\hko}{\hspace{.1in}}

\DeclareMathOperator{\snr}{SNR} 
\DeclareMathOperator{\rank}{rank} 
\DeclareMathOperator{\tin}{in} 
\DeclareMathOperator{\tout}{out} 

\DeclareMathOperator{\F}{F}
\DeclareMathOperator{\diag}{diag}
\DeclareMathOperator{\Diag}{Diag}
\DeclareMathOperator{\trace}{Tr}

\DeclareMathAlphabet{\mathcalligra}{T1}{calligra}{m}{n}
\DeclareMathOperator{\degree}{deg}

\newtheorem{theorem}{Theorem}

\newtheorem{lemma}{Lemma}

\newtheorem{proposition}{Proposition}

\newtheorem{definition}{Definition}

\newtheorem{remark}{Remark}

\newcommand{\K}{\bm K}
\newcommand{\rfr}{\bm R}
\newcommand{\X}{\bm X}
\newcommand{\Z}{\bm Z}
\newcommand{\Y}{\bm Y}
\newcommand{\y}{\bm y}
\newcommand{\z}{\bm z}
\newcommand{\R}{r\xspace}
\newcommand{\N}{N}
\newcommand{\D}{d}
\newcommand{\ocap}{\bm O}
\newcommand{\U}{\bm U}

\newcommand{\smalln}{n}

\newcommand{\bmu}{\bs\mu}
\newcommand{\bSigma}{\bs\Sigma}

\newcommand{\bk}{\color{black}}

\newcommand{\bm}{\mathbf}
\newcommand{\bs}{\boldsymbol}

\title{A Robust Spectral Clustering Algorithm for Sub-Gaussian \\ Mixture Models with Outliers}

\author{
Prateek R. Srivastava \thanks{Graduate Program in Operations Research and Industrial Engineering (ORIE), University of Texas at Austin, Austin, TX, 78712-1591, USA. Email: {\tt prateekrs@utexas.edu, grani.hanasusanto@utexas.edu}.}
\and 
Purnamrita Sarkar \thanks{Department of Statistics and Data Sciences (SDS), University of Texas at Austin, Austin, TX, 78712-1591, USA. Email: {\tt purna.sarkar@austin.utexas.edu}.}%
\and 
Grani A. Hanasusanto \footnotemark[1]
}

\begin{document}

\maketitle

\begin{onehalfspace}

\begin{abstract}
    \noindent We consider the problem of clustering datasets in the presence of arbitrary outliers. Traditional clustering algorithms such as $k$-means and  spectral clustering are known to perform poorly for datasets contaminated with even a small number of outliers. In this paper, we develop a provably robust spectral clustering algorithm that 
    applies a simple rounding scheme to denoise a Gaussian kernel matrix built from the data points and uses vanilla spectral clustering to recover the cluster labels of data points. We  analyze the performance of our algorithm under the assumption that the ``good'' data points are generated from a  mixture of sub-gaussians (we term these ``inliers''), while the outlier points can come from any arbitrary probability distribution. For this general class of models, we show that the mis-classification error decays at an exponential rate in the signal-to-noise ratio, provided the number of outliers is a small fraction of the inlier points.
    Surprisingly, the derived error bound 
    matches with the best-known bound \citep{fei2018hidden,giraud2018partial} for semidefinite programs (SDPs) under the same setting without outliers. We conduct extensive experiments on  a variety of simulated and real-world datasets to demonstrate that our algorithm is less sensitive to outliers compared to other state-of-the-art algorithms proposed in the literature.

\end{abstract}


\graphicspath{ {Images/} }

\section{Introduction}
\label{sec-introduction}
Clustering is a fundamental problem in unsupervised learning with application domains ranging from evolutionary biology, market research, and medical imaging to recommender systems and social network analysis, etc. In this paper, we consider the problem of clustering $n$ independent and identically distributed inlier data points in $d$-dimensional space from a mixture of $\R$ sub-gaussian probability distributions with unknown means and covariance matrices in the presence of arbitrary outlier data points. Given a sample dataset consisting of these inlier and outlier points, the objective of our inference problem is to recover the latent cluster memberships for the set of inlier points, and additionally, to identify the outlier points in the dataset.

Sub-gaussian mixture models (SGMMs) are an important class of mixture models that provide a distribution-free approach for analyzing clustering algorithms and encompass a wide variety of fundamental clustering models, such as (i) spherical and general Gaussian mixture models (GMMs), (ii) stochastic ball models \citep{iguchi2015tightness,kushagra2017provably}, which are mixture models whose components are isotropic distributions supported on unit $\ell_2$-balls,  and (iii) mixture models with component distributions that have a bounded support, as its special cases. 

Taking the clustering objective and tractability of algorithms into consideration, several different solution schemes based on Lloyd's algorithm~\citep{lloyd1982least}, expectation maximization~\citep{dempster1977maximum}, method of moments~\citep{pearson1936method, bickel2011method}, spectral methods \citep{dasgupta1999learning,vempala2004spectral}, linear programming~\citep{awasthi2015relax} and semidefinite programming~\citep{peng2007approximating,mixon2016clustering,yan2016convex} have been proposed for clustering SGMMs. Amongst these different algorithms, Lloyd's algorithm, which is a popular heuristic to solve the $k$-means clustering problem, is arguably the most widely used. When the data lies on a low dimensional manifold, a popular alternative is Spectral Clustering, which applies $k$-means on the top eigenvectors of a suitably normalized kernel similarity matrix~\citep{shi2000normalized,ng2002spectral,von2007tutorial,von2008consistency,schiebinger2015geometry,amini2019concentration}.

\begin{figure}[t!]
    \centering
    \begin{subfigure}[t]{0.45\textwidth}
        \centering
        \includegraphics[width=\linewidth]{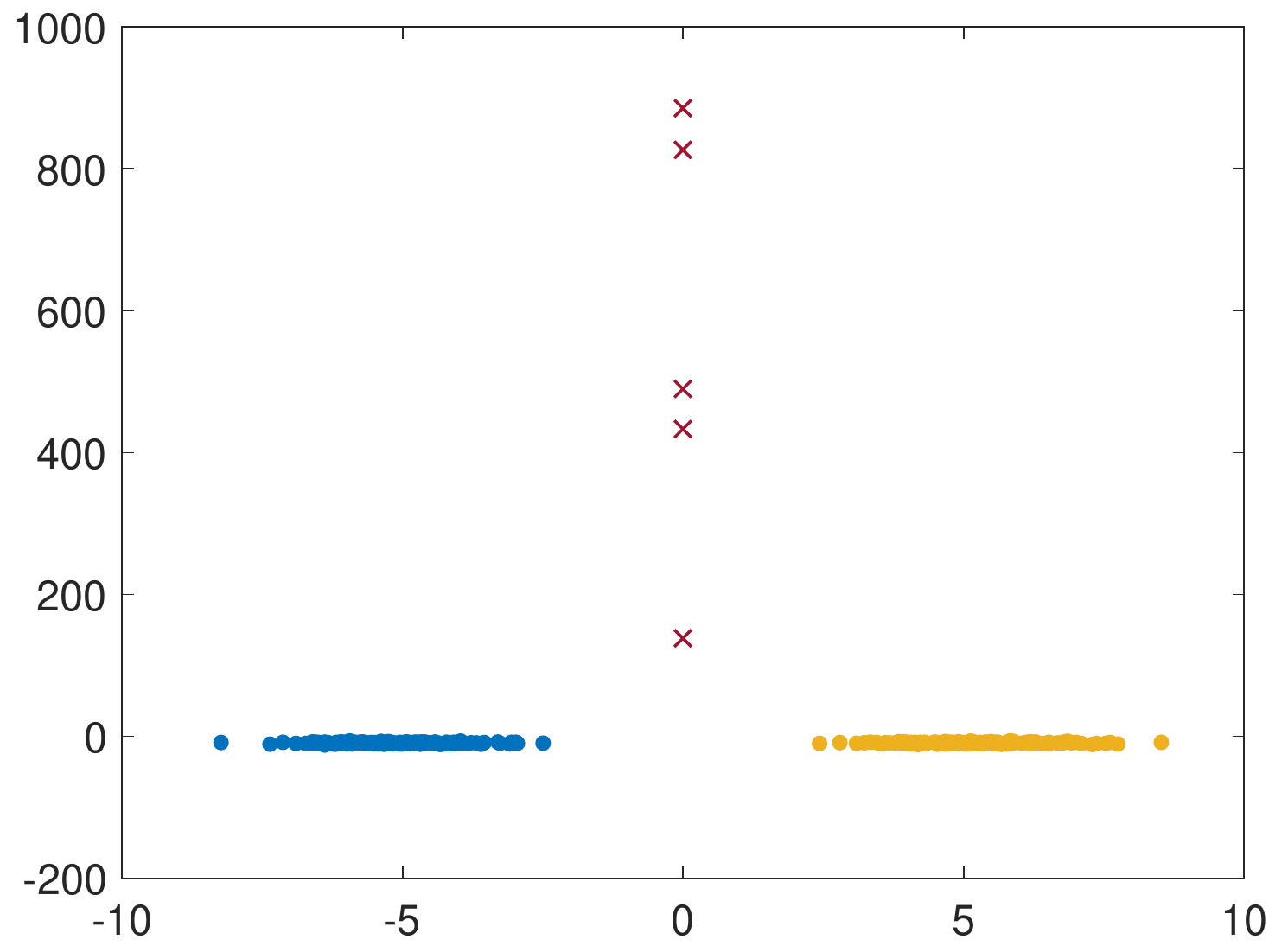}
        \caption{Original Dataset}
    \end{subfigure}
    \begin{subfigure}[t]{0.45\textwidth}
        \centering
        \includegraphics[width=\linewidth]{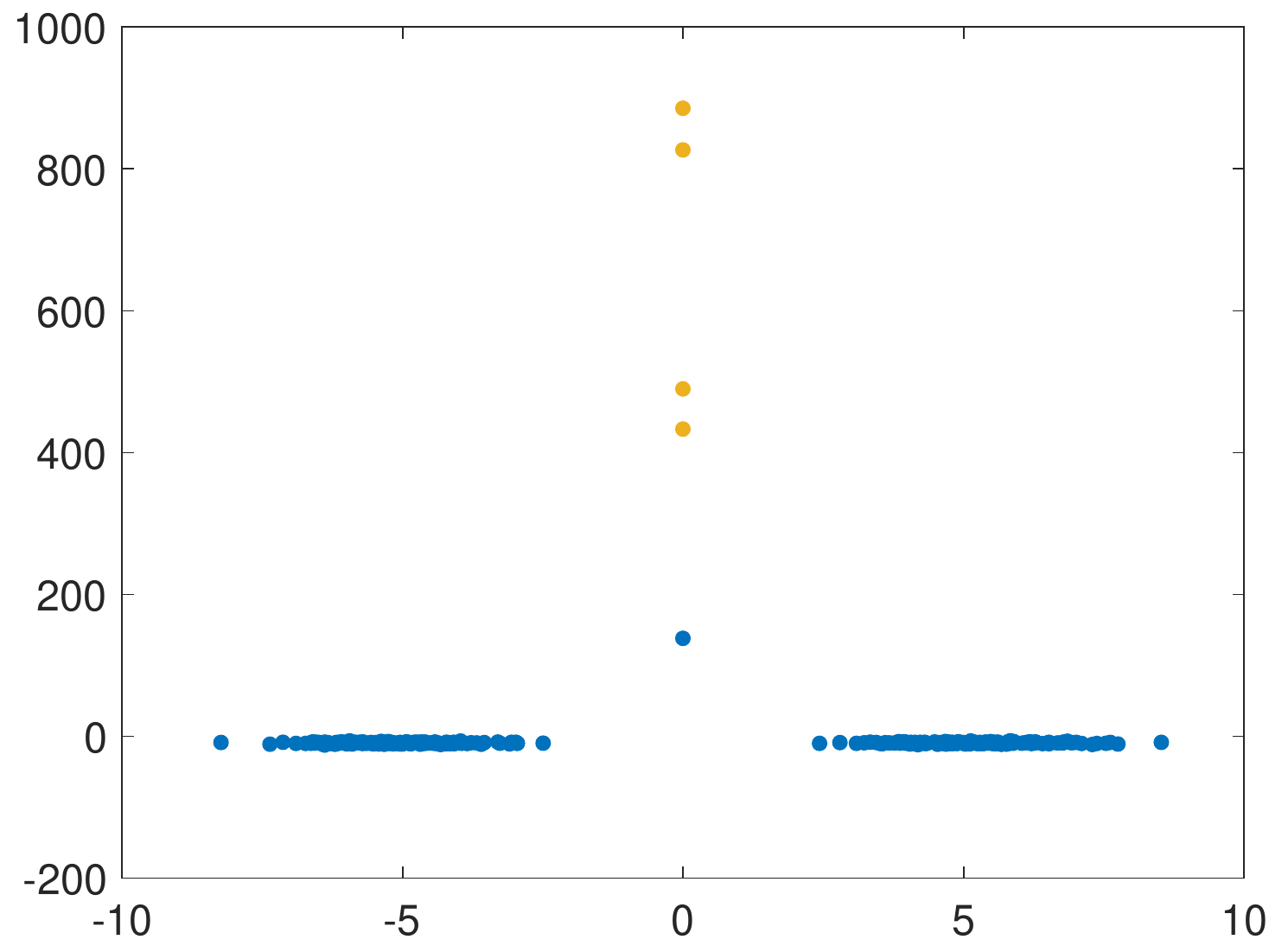} 
        \caption{ Clustering result obtained from $k$-means++ and spectral clustering   \citep{vempala2004spectral}}
    \end{subfigure}
\caption{ $k$-means++ and the spectral clustering algorithm proposed by \citet{vempala2004spectral} are not robust to the outliers. The original dataset consists of inlier data points (marked as solid circles) drawn from a mixture of two Gaussian distributions with means $\bs \mu_1=[-5,0]^\top$, $\bs \mu_2=[5,0]^\top$, covariance matrices $\bs \Sigma_1=\bs \Sigma_2=\bm I_2$, and number of points $n_1=n_2=150$. There are $m=5$ outlier points generated on the y-axis, which are marked as red crosses. In the clustering obtained from both the algorithms, the original clusters are merged into one, and the second cluster comprises entirely of the outlier data points.}
\label{fig:BadClusters}
\end{figure}

Despite their popularity, the performances of vanilla versions of both $k$-means clustering and spectral clustering are known to deteriorate in the presence of noise~\citep{li2007noise,bojchevski2017robust,zhang2018understanding}. \Cref{fig:BadClusters} illustrates a simple example where the two algorithms fail in the presence of outlier points.

\subsection{Our Contributions}

In this paper, we consider the joint kernel clustering and outlier detection problem under a SGMM setting assuming an arbitrary probability distribution for the set of outlier points. First, we formulate the exact kernel clustering problem with outliers and propose a robust SDP-based relaxation for the problem, which is applied after the data has been projected onto the top $\R-1$ principal components (when $d>\R$). This projection step not only helps tighten our theoretical bounds but also yields better empirical results when the dimensionality is large. 

Since SDP formulations do not usually scale well to large problems, we propose a linear programming relaxation that essentially rounds the kernel matrix, on which we apply spectral clustering. In some sense, this algorithm is reminiscent of building a nearest neighbor graph from the data and applying spectral clustering on it. In the literature, $k$-nearest neighbor graphs have found applications in several machine learning algorithms~\citep{cover1967nearest,altman1992introduction,hastie1996discriminant,ding2004k,franti2006fast}, and have been analyzed in the context of density-based clustering algorithms~\citep{Dudensity,verdinelli2018}, and subspace clustering~\citep{heckel2015robust}. 

In general, kernel-based methods are harder to analyze compared to distance-based algorithms since they involve analyzing non-linear feature transformations through the kernel function. In this work, we show that with high probability, our algorithm recovers true cluster labels with small error rates for the set of inlier points, provided that there is a reasonable separation between the cluster centers and the number of outliers is not large. An interesting theoretical result that emerges from our analysis is that the error rate obtained for our spectral clustering algorithm decays exponentially in the square of the signal-to-noise ratio for the case when no outliers are present, which matches with the best-known theoretical error bound for SDP formulations \citep{fei2018hidden} under the SGMM setting.  

Empirically, we observe a similar trend in the performances of robust spectral clustering and our proposed robust SDP-based formulation on real-world datasets, while the first is orders of magnitude faster.  This is quite surprising since, in other model scenarios like the Stochastic Block Model~\citep{holland1983stochastic}, SDPs have been proven to return clusterings correlated to the ground truth in sparse data regimes~\citep{guedon2016community,montanari2016semidefinite}\bk; whereas only regularized variants of spectral clustering \citep{amini2013pseudo,le2015sparse,joseph2016impact,zhang2018understanding} work in these parameter regimes. However, to be fair, empirically we see that SDP is less sensitive to hyperparameter mis-specification. We now summarize the main contributions of this paper. \bk

\begin{enumerate}
    \item We derive an exact formulation for the kernel clustering problem with outliers and obtain its SDP-based convex relaxation in the presence of outliers in the dataset. Unlike previously proposed robust SDP formulations \citep{rujeerapaiboon2017size,yan2016robustness}, our robust SDP formulation does not require prior knowledge of the number of clusters, the number of outliers, or cluster cardinalities.

    \item We propose an efficient algorithm based on rounding and spectral clustering, which is provably robust. Specifically, we show that provided the number of outliers is small compared to the inlier points, the error rate for our algorithm decays exponentially in the square of the signal-to-noise ratio. This error rate is consistent with the best-known theoretical error bound for SDP formulations \citep{fei2018hidden,giraud2018partial}.
    
    Although an extensive amount of work has been done previously to analyze spectral methods in the context of GMMs~\citep{dasgupta1999learning,vempala2004spectral,loffler2019optimality}, to the best of our knowledge, no prior theoretical work has been done to analyze robust spectral clustering algorithms for the non-parametric and more general SGMM setting (with or without outliers). 
\end{enumerate}

\subsection{Related Work}
 Several previous works~\citep{cuesta1997trimmed,li2007noise,forero2012robust,bojchevski2017robust,zhang2018understanding} have proposed robust variants of $k$-means and spectral clustering algorithms; however, they do not provide any recovery guarantees. Recently, there has been a focus on developing robust algorithms based on semidefinite programming and analyzing them for special cases of SGMMs.~\citet{kushagra2017provably} develop a robust reformulation of the $k$-means clustering SDP proposed by \citet{peng2007approximating} and derive exact recovery guarantees under arbitrary (not necessarily isotropic) and stochastic ball model settings using a primal-dual certificate. On a related note,~\citet{rujeerapaiboon2017size} also obtain a robust SDP-based clustering solution by minimizing the $k$-means objective subject to explicit cardinality constraints on the clusters as well as the set of outlier points. Besides the SGMM setting, robust clustering 
algorithms have been proposed for the related problem of subspace clustering where similar theoretical guarantees have been obtained \citep{pmlr-v28-wang13,wang2018theoretical, heckel2015robust,heckel2015dimensionality,soltanolkotabi2012,soltanolkotabi2014} as well as for some other model settings~\citep{vinayak2016similarity,yan2016robustness}. Particularly relevant to us is the work of~\citet{yan2016robustness}, who compare the robustness of kernel clustering algorithms based on SDPs and spectral methods. However, they analyze the algorithms for the mixture model introduced by~\citet{elkaroui2010}, which  assumes the data to be generated from a low dimensional signal in a high dimensional noise setting. Intuitively, in this setting, the signal-to-noise ratio, defined as the ratio of the minimum separation between cluster centers $(\Delta_{\min})$ to the largest spectral norm $(\sigma_{\max})$ of the covariance matrices of the mixture components, grows as $\sqrt{\D}$. The authors show that without outliers, the SDP-based algorithm is strongly consistent, i.e., it achieves exact recovery, while kernel SVD algorithm is weakly consistent, i.e., the fraction of mis-classified data points go to zero in the limit as long as $\D$ increases polynomially in  $N$, the total number of points. Note that, in typical mixture models, the number of dimensions, while arbitrarily large, stay fixed, and there is a possibly small yet non-vanishing Bayes error rate, which is more realistic. 

For the no outliers setting, an extensive amount of work has been done to obtain theoretical guarantees on the performances of various clustering algorithms under different distributional assumptions about the underlying data generation process. For the Gaussian mixture model setting,~\citet{dasgupta1999learning} is amongst the first to obtain theoretical guarantees for a random projections-based clustering algorithm that is able to learn the parameters of mixture model provided the minimum separation between cluster centers $\Delta_{\min}=\Omega(\sqrt{\D} \sigma_{\max})$. Using distance concentration arguments based on the isoperimetric inequality,~\citet{sanjeev2001learning} improve the minimum separation to~$\Delta_{\min}=\Omega(d^{1/4}\sigma_{\max})$. For the special case of a mixture $r$~spherical Gaussians,~\citet{vempala2004spectral} sho  that for their spectral algorithm the separation can be further reduced to $\Delta_{\min}=\Omega((\R\log d)^{1/4}\sigma_{\max})$, which ignoring the logarithmic factor in $d$, is essentially independent of the dimension of the problem. These results are generalized and extended further in subsequent works of~\citet{kumar2010clustering}  and~\citet{awasthi2012improved}. For a distribution-free model described in terms of the proximity conditions considered in~\citet{kumar2010clustering},~\citet{li2020birds} obtain guarantees for the \citet{peng2007approximating} $k$-means SDP relaxation. Under the stochastic ball model setting,~\citet{awasthi2015relax} obtain exact recovery guarantees for linear programming and SDP-based formulations for $k$-median and $k$-means clustering problems using a primal-dual certificate argument. Extending the results of~\citet{awasthi2015relax},~ \citet{mixon2016clustering} show that for a mixture of sub-gaussians, the SDP-based formulation proposed in \citet{peng2007approximating} guarantees good approximations to the true cluster centers provided the minimum distance between cluster centers~$\Delta_{\min}= \Omega(r\sigma_{\max})$. Under a similar separation condition, ~\citet{yan2016convex} also obtain recovery guarantees for a kernel-based SDP formulation under the SGMM setting. Most pertinent to us is the recent result obtained by~\citet{fei2018hidden}, who show that for a minimum separation of $\Delta_{\min}= \Omega(\sqrt{\R} \sigma_{\max})$ the mis-classification error rate of a SGMM with equal-sized clusters  decays exponentially in the square of the signal-to-noise ratio. Another analogous result for the SDP formulation proposed by \citet{peng2007approximating} has been obtained by~\cite{giraud2018partial}. Very recently, we also became aware of the result obtained by \citet{loffler2019optimality}, who obtain an exponentially decaying error rate for a spectral clustering algorithm for the special case of spherical Gaussians with identity covariance matrices. However, in order for their result to hold with high probability, they require the minimum separation between cluster centers to go to infinity. In addition, their proposed algorithm can easily be shown to fail in the presence of outliers, as discussed in greater detail in \Cref{sec-results}. 

In addition to the clustering literature where data is typically drawn i.i.d. from a mixture distribution, spectral and SDP relaxations for hard combinatorial optimization problems have also received significant attention in graph partitioning and community detection literature \citep{goemans1995improved,mcsherry2001spectral,newman2006modularity,rohe2011spectral,sussman2012consistent,fishkind2013consistent,qin2013regularized,guedon2016community,yan2017exact,amini2014semidefinite}.

\subsection{Paper Organization}
The remainder of the paper is structured as follows. In \Cref{sec-ProblemSetup}, we introduce the notation used in the paper and describe the problem setup for sub-gaussian mixture models with outliers. In \Cref{sec-algorithm}, we obtain the formulation for kernel clustering problem with outliers and derive its SDP and LP relaxations that recover denoised versions of the kernel matrix. In addition, we also discuss the details of the clustering algorithm that obtains cluster labels from this denoised matrix. \Cref{sec-results} summarizes the main theoretical findings for our clustering algorithm, provides an overview of the proof techniques used, and contrasts our results with the existing results in the literature. \Cref{sec-experiments} presents experimental results for several simulated and real-world datasets. Technical details of proofs for the main theorems are deferred to the appendix.

\section{Notation and Problem Setup}
\label{sec-ProblemSetup}
In this section, we introduce the notation used in this article and explain the formal setup of the kernel clustering problem for sub-gaussian mixture models with outliers.

\subsection{Notation} For any $n\in \mb N$, we define $[n]$ as the index set $\{1,\ldots,n\}$. We use uppercase bold-faced letters such as $\bm A,\bm B$ to denote matrices and lowercase bold-faced letters such as $\bm u, \bm v$ to denote vectors. For any matrix $\bm A$, $\trace( \bm A)$ denotes its trace, $A_{ij}$ its $(i,j)$-th entry, and $\diag(\bm A)$ represents the column-vector of its diagonal elements. We define $\Diag(\bm v)$ to be a diagonal matrix with vector $\bm v$ on its main diagonal. We consider different matrix norms in our analysis. For a matrix $\bm A\in \mb R^{N\times N}$, the operator norm $\lVert \bm A \rVert_2$ represents the largest singular value of $\bm A$, the Frobenius norm $\lVert \bm A\rVert_{\F}=\big(\sum_{ij}A_{ij}^2\big)
^{1/2}$ and $\ell_1$-norm $\lVert \bm A \rVert_1=\sum_{ij}\lvert A_{ij} \rvert$. For two matrices $\bm A,\bm B$ of same dimensions, the inner product between $\bm A$ and $\bm B$ is denoted by $\langle \bm A, \bm B\rangle:=\trace(\bm A^\top \bm B)=\sum_{ij}  A_{ij}B_{ij}$. We represent the $n$-dimensional vector of all ones by $\bs 1_n$, the $n\times n$ matrix of all ones by $\bm E_n$, the $n\times n$ identity matrix by $\bm I_{n}$ and $n\times m$ matrix of all zeros by $\bm 0_{n\times m}$. We define $\bm e_i$ to be the $i$-th standard basis vector whose $i$-th coordinate is 1 and all other coordinates are 0. We use $\mb S_n^+$ to denote the cone of $n\times n$ symmetric positive semidefinite matrices. Further, we say that a $n\times n$ matrix $\bm X\succeq \bs 0$ if and only if $\bm X\in \mb S_N^+$. 

For the asymptotic analysis, we use standard notations like $o,O,\Omega$ and $\Theta$ to represent rates of convergence. We also use standard probabilistic order notations like $O_p$ and $o_P$ (see~\citet{van2000asymptotic} for more details). We define $x\lesssim y$ to denote $x\leq cy$, where $c$ is some positive constant. We use $\tilde{O}$ to denote $O$ with logarithmic dependence on the model parameters.

\subsection{Problem Setup} 

We consider a generative model that generates a set of $n$ independent and identically distributed inlier points, denoted by $\mathc I$, from a mixture of $r$ sub-gaussian probability distributions \citep{vershynin2010introduction} $\{\mathc D_k\}_{k=1}^r$. The set $\mathc O$ of outlier points can come from arbitrary distributions with $\lvert\mathc O\rvert=m$.  Given the observed data matrix~$\Y=[\y_1,\ldots,\y_\N]^\top \in \mb R^{\N\times \D}$ consisting of these $\N:=n+m$ points  in $\D$-dimensional space, the task is to recover the latent cluster labels for the set of inlier points $\mathc I$, and identify the outliers $\mathc O$ in the dataset. 

For the set of inlier points, let $\bs \pi=(\pi_1,\ldots,\pi_\R)$ where $\bs \pi \geq  \bs 0$ and $\bs\pi^\top\bm 1_\R=1$ denote the mixing weights associated with the $\R$ sub-gaussian probability distributions in the mixture model such that $\pi_{\max}=\max_{k\in[\R]} \pi_k$ and $\pi_{\min}=\min_{k\in[\R]} \pi_k$. Assume that $\bs\mu_1,\ldots,\bs\mu_r\in \mb R^{\D}$ represent the means of $r$ clusters from which the data points are generated. Under the SGMM model, for each point $i\in \mathc I$, first a label $\phi_i\in\{1,\dots, \R\}$ is generated from a Multinomial($\bs\pi$), where $\bs \pi$ is a $\R$-dimensional vector denoting the cluster proportions. We define the true cluster membership matrix~$\bm Z^0\in \{0,1\}^{\N\times r}$ such that $Z^0_{ik}=1$ if and only if point~$i\in \mathcal{I}$ and $\phi_i=k$. Thus, assuming~$Z^0_{ik}=1$, observation $\y_i$ is generated from distribution $\mathc D_k$ with the following form:
\begin{equation*}
    \y_i:=\bs\mu_k+\bs\xi_i,
\end{equation*}
where $\bs \xi_i$ is a mean zero sub-gaussian random vector with $\sigma_k^2$ defined as the largest eigenvalue of its second moment matrix and $\sigma_{\max}:=\max_{k\in[r]}\sigma_k$. We represent the $k$-th cluster by $\mathc C_k :=\{i\in\mathc I:\phi_i=k\}$ and its cardinality by $\smalln_k:=\lvert \mathc C_k \rvert$. The separation between any pair of clusters $k$ and $l$ is defined as~$\Delta_{kl}:=\lVert\bs \mu_k-\bs \mu_l\rVert_2$ with the minimum and maximum separation denoted respectively as $\Delta_{\min}:=\displaystyle \min_{k\neq l} \Delta_{kl}$ and $\Delta_{\max}:=\displaystyle \max_{k\neq l} \Delta_{kl}$. In our analysis, an important quantity of interest is the signal-to-noise ratio, which based on \cite{fei2018hidden} is defined as 
\begin{equation}
    \snr:=\frac{\Delta_{\min}}{\sigma_{\max}}.
\end{equation}

 Without loss of generality, we assume that the points in $\bm Z^0$ are ordered such that the inliers and outliers are indexed together. Within the set of inliers again, we further assume that the points belonging to the same cluster are indexed together. Thus, the true clustering matrix $\bm X^0=\bm Z^0{\bm Z^0}^\top$ is a block diagonal matrix   with $X^0_{ij}=1$ if $i$ and $j$ belong to the same cluster and 0 otherwise. For our algorithm, we use the Gaussian kernel matrix $\K \in [0,1]^{\N\times \N}$ whose $(i,j)$-th entry $K_{ij}:=\mathc \exp\big(-\frac{\lVert \y_i - \y_j\rVert^2}{2\theta^2}\big)$ defines the similarity between points $i$ and $j$ for some scaling parameter $\theta$.

 \section{Robust Kernel Clustering Formulation}
\cite{stella2003multiclass} show that the normalized $k$-cut problem is equivalent to the following trace maximization problem $\trace(\Z^\top\K \Z)$ where $\Z$ is a scaled cluster membership matrix. In their seminal paper, \cite{dhillon2004kernel} prove the equivalence between kernel $k$-means and normalized $k$-cut problem. Based on~\cite{dhillon2004kernel} and~\cite{stella2003multiclass}, \cite{yan2016robustness} propose a SDP relaxation for the kernel clustering problem under the assumption of equal-sized clusters.~\citet{yan2016convex} further extend the kernel clustering formulation to unequal-sized clusters for analyzing the community detection problem in the presence node covariate information. Their formulation, which is derived from the SDP formulation for the $k$-means clustering problem~\citep{peng2007approximating}, however, does not account for possible outliers in the dataset. 

In this section, we first consider an exact formulation for the kernel clustering problem with equal-sized clusters and no outliers. We then extend this formulation to incorporate the case where cluster sizes may be unequal as well as unknown, and outliers are present in the dataset. Finally, we use the idea of ``lifting" and ``relaxing" to obtain two efficient algorithms based on tractable SDP and spectral relaxations for this exact formulation.
\begin{equation}
\label{eq:ExactForm-Z}
\begin{aligned}
& \underset{ \Z}{\text{maximize}} & &  \langle \K,\Z \Z^\top \rangle \\
& \text{subject to}
&&  \Z \in \{0,1\}^{\smalln\times r} \\
&&&  \sum_{k \in [r]} Z_{ik} = 1 &\hko \forall i=1,\ldots,\smalln \\
&&&  \sum_{i \in [\smalln]} Z_{ik} = \frac{\smalln}{\R} &\hko \forall k=1,\ldots,\R
\end{aligned} 
\end{equation}

The optimization formulation in \eqref{eq:ExactForm-Z} represents the kernel clustering problem without outliers that aims to maximize the sum of within-cluster similarities subject to assignment constraints that require each data point $i$ to belong to exactly one cluster and cardinality constraints that assume all clusters to be equal-sized with exactly $\frac{\smalln}{\R}$ (assumed to be integral) data points in each cluster. For the case where the clusters are required to be equal-sized, the cardinality constraints in \eqref{eq:ExactForm-Z} can be equivalently expressed in an aggregated form by requiring $\langle \bm E_\smalln,\Z \Z^\top \rangle = \frac{\smalln^2}{\R}$. 

In general, however, the clusters are seldom equal-sized; in addition, their exact cardinalities are also seldom known in practice. However, if cardinality constraints are dropped from the formulation, the optimal solution $\Z^*$ assigns all points to a single cluster. A natural way to overcome this issue would be to maximize $ \langle \bm K-\gamma\bm E_\smalln,\Z \Z^\top \rangle$ for $\gamma\in (0,1)$. Note that for a valid cluster membership matrix $\Z$, $\langle \bm E_\smalln,\Z \Z^\top \rangle= \frac{\smalln^2}{\R}$ represents its minimum value, which is achieved exactly when all the clusters are equal-sized. Thus, the penalized objective function essentially tries to find clusters that are balanced.

We extend the formulation in~\eqref{eq:ExactForm-Z} to account for possible outliers in the dataset by relaxing the assignment constraint on each data point to belong to either exactly one cluster (if the data point is an inlier) or to no cluster (if the data point is an outlier). The resulting exact formulation for the kernel clustering problem with outliers is a binary quadratic program and is shown in~\eqref{eq:outlier-ExactForm-Z}.
\begin{multicols}{2}
\begin{equation}
\label{eq:outlier-ExactForm-Z}
\begin{aligned}
& \underset{\Z}{\text{maximize}} & &  \langle \K -\gamma \bm E_\N, \Z \Z^\top \rangle \\
& \text{subject to}
&&  \Z \in \{0,1\}^{\N\times r} \\
&&&  \Z \bm 1_r \leq \bm 1_N. \\
\end{aligned} 
\end{equation}
\begin{equation}
\label{eq:outlier-ExactForm-X}
    \begin{aligned}
& \underset{ \X}{\text{maximize}} & &  \langle \K - \gamma \bm E_N, \X\rangle \\
& \text{subject to}
&&  \X \in \{0,1\}^{N\times N} \\
&&& \X \succeq	\bs 0\\
&&&  \text{rank} (\X) \leq r
\end{aligned}
\end{equation}
\end{multicols}


The formulation in \eqref{eq:outlier-ExactForm-Z} involves maximizing a non-convex quadratic objective function over a set of binary matrices $\Z \in \{0,1\}^{\N\times r}$. One way to sidestep this difficulty would be by \say{lifting} the formulation from a low-dimensional space of $\N\times r$ matrices to a high dimensional space of $\N\times \N$ matrices by defining an auxiliary semidefinite matrix~$\X=\Z \Z^\top$ that represents the clustering matrix and expressing the feasible space in terms of the valid inequalities for $\X$.  The resulting formulation is given in \eqref{eq:outlier-ExactForm-X}. In the following proposition, we show that these two formulations are equivalent.

\begin{proposition}
Formulations~\eqref{eq:outlier-ExactForm-Z} and \eqref{eq:outlier-ExactForm-X} are equivalent up to a rotation, i.e., if $\X^*$ is an optimal solution to optimization problem \eqref{eq:outlier-ExactForm-X}, then there exists a decomposition $\X^*={\bm G^*}{\bm G^*}^\top$ and an orthogonal matrix $\ocap \in \mb R^{\R\times \R}$ such that $\Z^*={\bm G^*}\bm O$ is an optimal solution for \eqref{eq:outlier-ExactForm-Z} with the same objective function value.   
\end{proposition}
 We defer the proof to the appendix. Note that in the formulation presented in \eqref{eq:outlier-ExactForm-X}, the rows of $\bm X$ corresponding to outliers are essentially zero vectors. This provides us with a way to identify the outliers. However, even this formulation is a non-convex optimization problem due to the rank and integrality constraints imposed on $\X$. Hence, we obtain tractable reformulations by considering two convex relaxations for the problem. In the first, we relax the binary constraint on~$\bm  X$, and also, drop the rank constraint. This yields the following SDP formulation: 
\begin{equation}
\tag{Robust-SDP}
\label{RelaxedSDPForm}
\begin{aligned}
& \underset{ \X}{\text{maximize}} & &  \langle \K - \gamma \bm E_N,\X\rangle \\
& \text{subject to}
&&  0\leq X_{ij} \leq 1 &\forall~ i,j \\
&&& \bm  X \succeq \bs	 0.
\end{aligned}
\end{equation}

\noindent We note here that similar SDP formulations have also been proposed in the community detection literature \citep{amini2014semidefinite,cai2015robust,guedon2016community}. Next, we consider a second relaxation in which we also allow the SDP constraint to be dropped from the formulation. The resulting formulation is a linear program which is specified below:  
\begin{equation}
\tag{Robust-LP}
\label{RelaxedForm}
\begin{aligned}
& \underset{ \X}{\text{maximize}} & &  \langle \K - \gamma \bm E_N,\X\rangle \\
& \text{subject to}
&&  0\leq X_{ij} \leq 1 &\forall~ i,j
\end{aligned}
\end{equation}

\label{sec-algorithm}
\begin{algorithm}[tb]
   \caption{{Robust Spectral Clustering / Robust-SDP}}
   \label{algorithm}
{\bfseries Input:} Observations $\y_1,\ldots,\y_\N \in \mb R^\D$, number of clusters $\R$, scaling parameter $\theta \in \mb R_+$ and offset parameter $\gamma \in (0,1)$.
\begin{enumerate}[topsep=0pt,itemsep=-1ex,partopsep=1ex,parsep=1ex,leftmargin=.5cm]
    \item Construct Gaussian kernel matrix $\K$ where $K_{ij}=\exp\big(\frac{-\lVert \y_i - \y_j\rVert^2}{2\theta^2}\big)$.
    \item Solve \ref{RelaxedForm} (\ref{RelaxedSDPForm}) to obtain the estimated clustering matrix $\htt \X$ ($\htt \X^{\text{SDP}}$).
    \item Compute the top $\R$ eigenvectors of $\htt \X$ ($\htt \X^{\text{SDP}}$) obtain $\hat{\U}\in \mb R^{\N\times \R}$.
    \item Apply $k$-means clustering on rows of $\hat{\U}$ to estimate the cluster membership matrix $\htt \Z$. 
    \item Use $\htt \X$ ($\htt \X^{\text{SDP}}$) to determine the degree threshold $\tau$. Set $\htt {\mathc  I}=\{i\in [\N] : \degree(i)\geq\tau \} $ and $\htt {\mathc  O}= [\N] \setminus \htt{\mathc I}$.
\end{enumerate}
\end{algorithm}

\noindent For convenience, we denote the feasible region of \ref{RelaxedForm} by set $\mathc X$ and its optimal solution by~$\htt \X$. It is straightforward to see that $\htt \X$ admits a simple analytical solution, which can be expressed below:
\begin{equation*}
    \htt X_{ij}=\begin{cases}
    1 \quad \text{if } K_{ij}-\gamma>0, \\
    0 \quad \text{otherwise}.
    \end{cases}
\end{equation*}

Algorithm~\ref{algorithm} summarizes the robust spectral clustering algorithm. To obtain the SDP variant of the algorithm, in step 2 of the algorithm, we solve the \ref{RelaxedSDPForm} formulation instead of the \ref{RelaxedForm} formulation. We also note here that steps~3~and~4 of the algorithm simply correspond to the application of vanilla spectral clustering to~$\htt \X$. In general, solving the $k$-means clustering problem in~step~4~is a NP-hard problem. Therefore, in our analysis, instead of solving the problem exactly, similar to \cite{lei2015consistency}, we consider the use of a $(1+\epsilon)$-approximate $k$-means clustering algorithm that runs in polynomial time. In the last step, we estimate the set of outliers~$\mathc{\htt O}$.  Based on our derivations of the \ref{RelaxedSDPForm} and \ref{RelaxedForm} formulations, we note that the outlier points in $\mathc{ O}$ correspond to near-zero degree nodes in the true clustering matrix $\X^0$. We make use of this fact to determine a degree threshold $\tau$ from the degree distribution of the nodes in~$\htt \X$, and assign the nodes that have degrees lesser than $\tau$ in $\htt \X$ to the set of outliers $\mathc{\htt O}$. The main idea behind this procedure is that if $\htt \X$ closely approximates~$\X^0$ and the threshold $\tau$ is appropriately chosen, then the low-degree nodes  below the threshold in~$\htt \X$ are good candidates for being outliers. 

It is important to note that properly choosing the parameters $\theta$ and $\gamma$ is central to the performance of the algorithm. For instance, if we choose the value of $\gamma$ to be arbitrarily close to 0 or 1, then $\htt \X$ obtained after rounding is either an all ones matrix or an all zeros matrix, thereby rendering the denoising step useless. In \Cref{sec-results}, we derive theoretical values for $\theta$ and $\gamma$ in terms of $\sigma_{\max}$ and $\Delta_{\min}$.

\section{Main Results}
\label{sec-results}
In this section, we summarize our main results and provide an overview of the approach used to obtain these results. Our main theoretical result is a finite sample guarantee on the estimation error for $\htt \X$. Specifically, we show the relative estimation error for $\htt \X$ decays exponentially in the square of the signal-to-noise ratio with probability tending to one as $N\rightarrow \infty$, provided there is sufficient separation between cluster centers and the number of outliers $m$ are a small fraction of the number of inliers points $\smalln$ (\Cref{theorem-ErrorRateX}). Using the result, we show that provided the clusters are approximately balanced, the error rate for $\htt \X$ translates into an error rate for $\htt \Z$, and hence, the fraction of mis-classified data points per cluster also decays exponentially in the square of the signal-to-noise ratio (\Cref{theorem-ErrorRateZ}). 

For analyzing semidefinite relaxations of clustering problems, a rather useful direction is the approach described in \cite{guedon2016community}, which is in  the context of stochastic block models. The main idea in the analysis of~\citet{guedon2016community} and~\citet{mixon2016clustering} is to come up with a suitable reference matrix $\bm R$, and then use concentration of measure to control the deviation of the input matrix (adjacency matrix $\bm A$ for~\citet{guedon2016community}, the matrix of pairwise squared Euclidean distances in~\citet{mixon2016clustering}, and the kernel matrix $\K$ for us) from the reference matrix.  However, there are some important differences between our setting and theirs. SGMMs and SBMs are fundamentally different because the kernel matrix $\K$ constructed for a SGMM arises from $n$ i.i.d.~datapoints, leading to  entries that are statistically dependent on each other. In contrast, the adjacency matrix of a random graph for a SBM has $n\choose 2$ Bernoulli random variables, which are conditionally independent given the latent cluster memberships. Therefore, the analytical techniques required to analyze SGMMs are completely different compared to SBMs. Both~\citet{mixon2016clustering} and~\citet{yan2016convex} use suitable reference matrices for related but different SDP relaxations. The proof techniques that we develop in this section are new and involve coming up with a new reference matrix that allows us to carefully bound the tail probabilities. In addition, the resulting error bound that we get from our analysis is also tighter than that of~\citet{mixon2016clustering} and~\citet{yan2016convex}. 

We now provide an overview of our proof approach. Our constructed reference matrix~$\rfr \in [0,1]^{\N\times \N}$~satisfies two properties:
\begin{enumerate}[label=(\roman*)]
    \item $\rfr$ is close to $\K$ with high probability in the $\ell_1$-norm sense.
    \item The solution to the reference optimization problem \eqref{ReferenceOptimizationProblem} defined below corresponds to the true clustering matrix $\X^0$~(\Cref{lem-TrueReferenceMaximizer}).
    \begin{equation}
    \label{ReferenceOptimizationProblem}
        \begin{aligned}
    & \underset{ \X}{\text{maximize}} & &  \langle \rfr - \gamma \bm E_\N,\X\rangle \\
    & \text{subject to}
    &&  0\leq X_{ij} \leq 1 & \forall~ i,j 
    \end{aligned}
    \end{equation}
    In other words, the reference matrix $\rfr$ is chosen in a way such that the true clustering matrix~$\X^0$ solves the reference optimization problem, which is obtained by replacing kernel matrix~$\K$ in \ref{RelaxedForm} with $\rfr$.
\end{enumerate}
We show that if (i) holds, then with high probability $\htt \X \in \mathc X$ approximately solves the reference optimization problem in $\eqref{ReferenceOptimizationProblem}$, i.e., $\langle \rfr - \gamma \bm E_\N, \htt \X\rangle \approx \langle \rfr - \gamma \bm E_\N, \X^0\rangle$ (see \Cref{lem-EstimatedAlmostReferenceMaximizer}).
Using this result, we then prove that if (ii) holds, and the number of outliers is a small fraction of the number of inliers in the dataset, then the estimated clustering matrix $\htt \X$ is close to the true clustering matrix~$\X^0$. In other words, the relative estimation error, $\frac{\lVert \htt \X-\X^0 \rVert_1}{\lVert \X^0 \rVert_1} \leq \epsilon$ (small) with probability tending to one as $\N\rightarrow \infty$ (see \Cref{theorem-ErrorRateX}). 
Next, using the Davis-Kahan theorem \citep{yu2014useful}, we show that provided the clusters are relatively balanced in sizes, the error rates obtained for $\htt \X$ also hold for the clustering membership matrix $\htt \Z$ obtained by applying spectral clustering on $\htt \X$ (see \Cref{theorem-ErrorRateZ}).

For our analysis, we assume the reference matrix $\rfr$ to be a random matrix whose $(i,j)$-th entry is defined as below: 
\begin{equation}
\label{ReferenceDefinition}
    R_{ij}= 
    \begin{cases}
    \max\{K_{ij}, \tau_{\tin} \} &\text{ if both } i \text{ and } j \in \mathc C_k \\
    \min\{K_{ij}, \tau_{\tout}^{(k,l)}\} & \text{ if } i\in \mathc C_k, j\in \mathc C_l \ (l\neq k)\\
    \gamma & \text{ if either } i\in \mathc O \text{ or } j \in \mathc O  \\
    \end{cases}
\end{equation}
Here, $\tau_{\tin}:= \exp{\big(-\frac{{r_{\tin}}^2}{\theta^2}}\big)$ and $\tau_{\tout}^{(k,l)}:=\exp\big(-\frac{{r^{(k,l)}_{\tout}}^2}{\theta^2}\big)$ are threshold quantities defined respectively for the diagonal and off-diagonal blocks of reference matrix over the set of inlier points. For $i,j\in \mathc C_k$, we obtain $R_{ij}$ by thresholding $K_{ij}$ to $\tau_{\tin}$ if $K_{ij}<\tau_{\tin}$. Similarly, for any $i\in \mathc C_k$ and  $j\in \mathc C_l$, $R_{ij}$ thresholds the value to $\tau_{\tout}^{(k,l)}$ if $K_{ij}>\tau_{\tout}^{(k,l)}$. The values of parameters $r_{\tin}$ and $r_{\tout}^{(k,l)}$, which we specify later in the section, are determined such that with high probability only a few kernel entries violate the thresholds defined for their respective blocks, and thus, property $(i)$ is satisfied. 

To ensure that our constructed reference matrix $\rfr$ satisfies property $(ii)$, we impose a \textit{strong assortativity} condition (similar to the analysis used for SBMs) which assumes that for the set of inlier points, the smallest entry $R^{\tin}_{\min}$ on the diagonal blocks of $\rfr$ is strictly greater than the largest entry $R^{\tout}_{\max}$ on any of its off-diagonal blocks, i.e., 
\begin{equation}
\label{StrongAssortativity}
    R^{\tin}_{\min}=\min_{i,j \in \mathc C_k: k\in [\R] } R_{ij} > \max_{ i\in \mathc C_k, j\in \mathc C_l:k,l \in [\R]} R_{ij} = R^{\tout}_{\max}.
\end{equation}
Based on the definition of the reference matrix, it is clear that~$R^{\tin}_{\min}\geq \tau_{\tin}$ and~$R^{\tout}_{\max}\leq \tau_{\tout} :=\max_{k\neq l}\tau_{\tout}^{(k,l)} $. Thus, the strong assortativity condition in \eqref{StrongAssortativity} is immediately implied if we require that $\tau_{\tin} > \tau_{\tout}$. We now use the strong assortativity condition in $\eqref{StrongAssortativity}$ to show that the true clustering matrix $\X^0$ is the solution to the reference optimization problem in $\eqref{ReferenceOptimizationProblem}$ as required by property $(ii)$.
\begin{lemma}
\label{lem-TrueReferenceMaximizer}
Suppose that the strong assortativity condition in \eqref{StrongAssortativity} holds and $R^{\tout}_{\max}<\gamma<R^{\tin}_{\min}$, then the true clustering matrix $\X^0$ maximizes the reference optimization problem in \eqref{ReferenceOptimizationProblem}.
\end{lemma}
\noindent
\begin{proof}
Set $R^{\tout}_{\max}<\gamma<R^{\tin}_{\min}$. Then, for the set of inlier points, all entries on the diagonal blocks of $\rfr - \gamma \bm E_\N$ are strictly positive, while those on the off-diagonal blocks are strictly negative. Thus, $\X^0=\arg\max\limits_{\X\in[0,1]^{N\times N}} \langle \rfr -\gamma \bm E_\N, \X \rangle$, i.e., $\X^0$ maximizes the reference objective function over the feasible region comprising of all $[0,1]^{\N\times \N}$ matrices. 
\end{proof}

\begin{remark}
 Note that even though we do not have SDP constraints, $\X^0=\Z^0 {\Z^0}^\top \in \mathc S_\N^+$  which implies $\X^0 \in \mathc X$ and $\X^0\in\arg\max\limits_{\X\in\mathc X} \langle \rfr -\gamma \bm E_\N, \X \rangle$. And thus, Lemma~\ref{lem-TrueReferenceMaximizer} also applies to Robust-SDP.
\end{remark}

\noindent Next, we present \Cref{lem-EstimatedApproximatesTrue}, which provides a bound on the estimation error for the inlier parts of~$\X^0$ and~$\htt \X$ in terms of the difference in their corresponding objective function values for the reference optimization problem. 
\begin{lemma}
\label{lem-EstimatedApproximatesTrue}
Suppose that the strong assortativity condition in \eqref{StrongAssortativity} holds and $R^{\tout}_{\max}<\gamma<R^{\tin}_{\min}$, then the estimation error for $\X^0$  over the set of inlier data points is 
$$\lVert \htt \X_{\mathc I} - \X^0_{\mathc I} \rVert_1 \leq \frac{\langle \rfr - \gamma \bm E_\N, \X^0-\htt \X\rangle}{\min(R^{\tin}_{\min}-\gamma,\gamma-R^{\tout}_{\max})}.$$
Additionally, if the penalty parameter $\gamma\in \big(R^{\tout}_{\max},R^{\tin}_{\min} \big)$ is expressed as  $\gamma= \upsilon \tau_{\tin}+(1-\upsilon)\tau_{\tout}$ for some constant~$\upsilon~\in~(0,1)$, then the above bound simplifies to $$ \lVert \htt \X_{\mathc I} - \X^0_{\mathc I} \rVert_1 \leq \frac{\langle \rfr - \gamma \bm E_\N, \X^0-\htt \X\rangle}{\min\{\upsilon,1-\upsilon\}(\tau_{\tin}-\tau_{\tout})}.$$
\end{lemma}

\noindent In the next lemma, we show that if the kernel matrix is close to the reference matrix in a $\ell_1$-norm sense, then the difference in the objective values of the reference optimization problem is also small. 
\begin{lemma}
\label{lem-EstimatedAlmostReferenceMaximizer}
Let $\K_{\mathc I},\rfr_{\mathc I}\in [0,1]^{\smalln \times \smalln} $ denote respectively the parts of the kernel and reference matrices with each $(i,j)$-th entry restricted to the set of inlier points, then 
$$\langle \rfr - \gamma \bm E_\N, \X^0-\htt \X\rangle \leq 2 \lVert \K_{\mathc I} -\rfr_{\mathc I} \rVert_{1}.$$
\end{lemma}

Based on the definition of the reference matrix in \eqref{ReferenceDefinition}, we note that for the $(i,j)$-th entry on the diagonal block of reference matrix where both $i,j\in \mathc C_k$, $R_{ij}$ deviates from its corresponding kernel value $K_{ij}$ only if $K_{ij}$ is below the threshold value $\tau_{\tin}$. Similarly, for the $(i,j)$-th entry on the off-diagonal block where $i\in \mathc C_k$ and $j\in \mathc C_l$, $R_{ij}$ differs from $K_{ij}$ only if $K_{ij}$ is above the threshold value $\tau_{\tout}^{(k,l)}$ for that block. Therefore, we obtain a bound on $\lVert \K_{\mathc I} - \rfr_{\mathc I} \rVert_1$ by bounding the number of kernel entries which deviate from their respective threshold values on the diagonal and off-diagonal blocks. In particular, we can bound the $\ell_1$-loss in
 \Cref{lem-EstimatedAlmostReferenceMaximizer}
 by the following:
\begin{align}
2\cdot\underbrace{\sum_{k\in[\R]}\sum_{\substack{i,j\in\mathc C_k: i< j}}\mathbbm{1}_{\{ K_{ij}<\tau_{\tin}\}}}_{A}+\underbrace{\sum_{k\neq l}\sum_{i\in \mathc C_k, j\in\mathc C_l}\mathbbm{1}_{\{ K_{ij}>\tau_{\tout}^{(k,l)}\}}}_{B}
\end{align}
If the entries of the kernel matrix were independent, a straightforward application of standard concentration inequalities would have provided us a bound. However, because of the dependence between them,  we use properties of the concept of U-statistics~\citep{hoeffding1963probability}. In particular, we write the first part ($A$) of the above decomposition in terms of the following sum of one-sample U-statistics:

\begin{equation}
    \displaystyle A=\sum_k {n_k\choose 2} U_{kk},\qquad U_{kk}=\frac{\sum_{\{(i,j):i,j\in \mathc C_k,i <j \}}\mathbbm{1}_{\{ K_{ij}<\tau_{\tin}\}}}{n_k(n_k-1)/2}.
    \label{eq:ukk}
\end{equation}

Similarly, we write the second part ($B$) of the decomposition in terms of the following sum of two-sample U-statistics:
\begin{equation}
    B=\sum_{k\neq l}n_kn_l U_{kl},\qquad U_{kl}=\frac{\sum_{i\in\mathc C_k, j \in \mathc C_l}\mathbbm{1}_{\{ K_{ij}>\tau_{\tout}^{(k,l)}\}}}{n_k n_l}.
    \label{eq:ukl}
\end{equation}

A U-statistic of degree $m$ is an unbiased estimator of some unknown quantity $\mb E[h(w_1,\dots, w_m)]$ (where $w_1,\dots,w_n$ are i.i.d.~observations drawn from some underlying probability distribution). It can be written as an average of the $h$ function (also known as the kernel function) applied on~${n\choose m}$~size~$m$ subsets of the data. It is not hard to see that $U_{kk}$ defined in~\eqref{eq:ukk} is a U-statistic created from $\bm y_i, i\in \mathc C_k$, where $\bm y_i$ are drawn i.i.d.~from the $k$-th SGMM mixture component. On the other hand, $U_{kl}$ defined in~\eqref{eq:ukl} is a two-sample U-statistic created from two i.i.d.~datasets drawn from the $k$-th and $l$-th SGMM mixture component.
Finally, using concentration results for U-statistics from \cite{hoeffding1963probability} and \cite{arcones1995bernstein}, we obtain a probabilistic bound on the number of corrupt entries. This leads to the bound on the estimation error for $\htt \X$ in Theorem~\ref{theorem-ErrorRateX}, which we present in the next sub-section.

\subsection{Estimation error}
\label{sec-EstimationError}
We are now in a position to present our first main result, which states that if the number of outlier points is much smaller than the number of inlier points in the dataset,  then with probability tending to one, the error rate obtained is small provided there is enough separation between the cluster centers and the sample size is sufficiently large. We state this result formally in the theorem below.

\begin{theorem}[Estimation error for \ref{RelaxedForm} solution $\htt \X$] 
\label{theorem-ErrorRateX} Let $\tau_{\tin}=\exp\big(-\frac{5\Delta_{\min}^2}{32\theta^2}\big)$ and $\tau_{\tout}^{(k,l)}=\exp\big(-\frac{\Delta_{kl}^2}{2\theta^2}\big)$. Choose $\gamma \in (\tau_{\tout},\tau_{\tin})$, where $\tau_{\tout}:=\max_{k\neq l}\tau_{\tout}^{(k,l)}=\exp\big(-\frac{\Delta_{\min}^2}{2\theta^2}\big)$. Suppose~$\theta~=~\Theta(\Delta_{\min})$ and the minimum separation between cluster centers $\Delta_{\min}\geq 8\sigma_{\max} \sqrt{\D}$, then with probability at least $1-2r/n_{\min}$, we have that the estimation error for the inlier part of $\htt \X$ is
\begin{equation}
\label{equation-ErrorRateInlierX}
        \lVert \htt \X_{\mathc I} -  \X^0_{\mathc I} \rVert_1
        \leq Cn^2\cdot \max \bigg\{ \exp\bigg(-\frac{\Delta_{\min}^2}{64\sigma_{\max}^2}\bigg),\frac{\log n_{\min}}{n_{\min}}\bigg\}.
\end{equation}
In addition, the relative estimation error for $\htt \X$ is  
\begin{equation}
\label{equation-ErrorRateX}
    \frac{\lVert \htt \X - \X^0\rVert_1}{\lVert \X^0 \rVert_1}\leq
    C' r \exp\bigg(-\frac{\Delta_{\min}^2}{64\sigma_{\max}^2}\bigg)  + C'' r \max\bigg\{\frac{\log n_{\min}}{n_{\min}},\frac{m}{\smalln}\bigg\}.
\end{equation}
\noindent Here, $C, C',C''>0$ are universal constants, and $n_{\min}:=\min_{k\in[\R]} n_k>r$ denotes the cardinality of the smallest cluster.
\end{theorem}

\begin{remark}
\label{rem-separationX}
 In Section~\ref{sec:highdim}, we prove that if one does a suitable dimensionality reduction to first project the data on the top 
 $\R-1$ principal components, then with probability tending to one, the projected data becomes a SGMM in~a $\R-1$ dimensional space with minimum cluster separation~$\Delta_{\min}/2$ as $N$ goes to $\infty$. As a result, the new separation condition for applying 
 Algorithm~\ref{algorithm} to this projected dataset becomes  \begin{equation}\label{eq:sep}
 \Delta_{\min}\geq 16\sigma_{\max} \sqrt{\min\{\D,\R\}}.
\end{equation}
\end{remark}

\begin{remark}
 It is possible to show the error rate for $\htt \X$ in \Cref{theorem-ErrorRateX} can be improved upon by avoiding the $\frac{\log n_{\min}}{n_{\min}}$ term in \eqref{equation-ErrorRateX}. However, this improved error bound is obtained at the expense of requiring a considerably larger sample size for a fixed  $1-\delta$ success probability of \eqref{equation-ErrorRateX}. 
\end{remark}

From \Cref{theorem-ErrorRateX}, we have that if there are no outliers in the dataset, i.e., $m=0$ or if the number of outliers grow at a considerably slower rate compared to the number of inlier points, i.e., $m=o_P(n)$, then asymptotically the error rate for $\htt \X$ decays exponentially with the square of the signal-to-noise ratio. To analyze this result in terms of prior theoretical work that has been done in the context of sub-gaussian mixture models without any outliers, we note that \cite{mixon2016clustering} show that for the $k$-means clustering SDP proposed by  \cite{peng2007approximating} which assumes that the number of clusters $\R$ is known, the estimation error (obtained after re-scaling) in a Frobenius norm sense $\lVert \htt \X -\X^0\rVert_{\F}^2$ decays at a rate of $\frac{\R^2\smalln_{\max}^2}{\snr^2}$ provided the minimum separation $\Delta_{\min} \gtrsim \R \sigma_{\max}$. In more recent work,  \cite{fei2018hidden} show that for their SDP formulation that minimizes the $k$-means objective assuming all clusters to be equal-sized, the relative estimation error decays exponentially in the square of the signal-to-noise ratio provided $\Delta_{\min}\gtrsim \sqrt{r} \sigma_{\max}$. \citet{giraud2018partial} obtain a similar error rate for the $k$-means clustering SDP proposed by \citet{peng2007approximating} that does not assume clusters to be equal-sized. Similar to \cite{fei2018hidden} and \citet{giraud2018partial}, our result in \Cref{theorem-ErrorRateX} also guarantees a theoretical error bound that decays as $\exp(-\Omega(\snr^2))$. The obtained bound is strictly better compared to \cite{mixon2016clustering} as shown below:   
\begin{equation*}
    \lVert \htt \X -\X^0\rVert_{\F}^2 \leq \lVert \htt \X -\X^0\rVert_1 \lesssim n^2 \exp(-\Omega(\snr^2)).
\end{equation*}

A key point to note in our results is that, in contrast to \cite{fei2018hidden} and \cite{mixon2016clustering}, our proof does not assume any prior knowledge about the number and sizes of clusters. In addition, \Cref{theorem-ErrorRateX} generalizes the analysis to incorporate outliers in the mixture of sub-gaussians setting. However, the separation condition $\Delta_{\min} \gtrsim \sqrt \D \sigma_{\max}$ does not generalize well to high dimensional settings where $\D \gg \R$. To overcome this, later in this section, we propose a simple dimensionality reduction procedure that allows us to obtain the error rate in \eqref{equation-ErrorRateX} for a reduced separation of $\Delta_{\min} \gtrsim \sqrt{\min\{\R,\D\}} \sigma_{\max}$ when $\R$ is known.  

Very recently, \citet{loffler2019optimality} obtain an exponentially decaying bound in the square of the signal-to-noise ratio for the spectral clustering algorithm proposed by \citet{vempala2004spectral}. However, for their analysis, they assume the data to be generated from a mixture of spherical Gaussians with identity covariance matrices. Furthermore, for their result to hold with high probability, the minimum separation $\Delta_{\min}$ needs to go to infinity. Based on the simple example considered in~\Cref{fig:BadClusters}, we also note that this algorithm is not robust to outliers.  .


We conclude this subsection with a comment on outliers. In our analysis so far, we have not made any specific assumptions on the distribution of the outliers points. However, one may have stronger theoretical results if such assumptions can be made; in particular, the following discussion shows that our algorithm can in fact tolerate $O(n)$ outlier points under suitable assumptions.
\begin{remark}
\label{remark-outlier}
Based on the distance of each outlier point to its closest cluster center, we divide the set of outlier points into two sets consisting of ``good" and ``bad" outlier points. Intuitively, the ``good" outlier points are far away from all the clusters, whereas the ``bad" outlier points may be arbitrarily close to one or more clusters. It can be easily shown that any outlier point that is ``bad" and close to a cluster center, can potentially have as many as $\Omega(\frac{n}{r})$ neighbors with high probability. For this reason, the first assumption that we make about the outlier points requires that the cardinality of the set of bad outlier points is at most $o(n)$. On the other hand,  if the outlier points are good, i.e., if they are far away from the clusters, then the set of good outlier points is potentially allowed to have a cardinality of $O(n)$. However, these good outlier points must be either isolated points or occur in small ``bunches" or clusters so the cardinality of any one cluster, comprising entirely of outlier points, is not too large (of the order $ \Omega(\frac{n}{r})$). One can ensure this by restricting the number of outlier points within a small neighborhood of each good outlier $i\in \mathc{O}_g$ to $o(n)$. 
\end{remark}
\noindent We now mathematically formalize these notions. 
\begin{definition}\label{def:goodout}
We denote the set of good outlier points by $\mathc O_g:=\{i\in \mathc O: \displaystyle\min_{k\in [r]}\lVert \bm y_i -\bmu_k \rVert \geq \sqrt{2} \Delta_{\min} \}$, which consists of outlier points whose distance from its closest cluster center is at least above the threshold $\sqrt{2} \Delta_{\min}$. 
In addition, we also assume that for all $i \in \mathc O_g$, the set of outlier neighboring points $\mathc N_{\mathc O}{(i)}:=\{ j\in\mathc O : \lVert \bm y_i - \bm y_j \rVert \leq  \frac{\Delta_{\min}}{\sqrt{2}}  \}$ has cardinality $o(n)$.
\end{definition}
\noindent We summarize our main result in the proposition below.

\begin{proposition}
\label{prop-outliers}
Let $\mathc O$ denote the set of outlier points. Let $\mathcal{O}_g\subset \mathc O$ be good outliers satisfying Definition~\ref{def:goodout}. 
Let $\mathc O_b:=\mathc O \setminus \mathc O_g$. 
Suppose the parameters $\gamma$ and $\theta$ are chosen as described in \Cref{theorem-ErrorRateX} and the minimum separation between cluster centers $\Delta_{\min}\geq 8\sigma_{\max} \sqrt{\D}$, then provided the size of $\mathc O_g$ is $O(n)$, with probability at least $1-3r/n_{\min}$, we have that the relative estimation error for $\htt \X$ is  
\begin{equation}
\label{equation-outliers}
\begin{aligned}
\frac{\lVert \htt \X -  \X^0 \rVert_1}{\lVert\X^0 \rVert_1}
            &\leq C' r\cdot \max \bigg\{ \exp\bigg(-\frac{\Delta_{\min}^2}{64\sigma_{\max}^2}\bigg),\sqrt{\frac{\ \log n_{\min} }{n_{\min}}}\bigg\} + \frac{2r \lvert \mathc O_b \rvert}{n} 
\end{aligned}
\end{equation}
\noindent Here, $C'>0$ is a universal constant and $n_{\min}:=\min_{k\in[\R]} n_k>r$ denotes the cardinality of the smallest cluster.
\end{proposition}
\noindent The proof of the theorem is deferred to the Appendix.
\begin{remark}
In Proposition~\ref{prop-outliers}, we have $|\mathcal{O}_g|=O(n)$ and as long as $|\mathc O_b|/n$ is smaller than the first term, we have the same asymptotic rate as Theorem~\ref{theorem-ErrorRateX}.
\end{remark}

 \subsection{Rounding error}
 As detailed in \Cref{algorithm}, we recover cluster labels $\htt\Z$ from the estimated clustering matrix $\htt \X$ by applying spectral clustering on the columns of $\htt \X$. Our proof technique for analyzing the spectral clustering step is inspired by the approach discussed in \cite{lei2015consistency}, where the authors rely on a polynomial-time solvable $(1+\epsilon)$-approximate $k$-means clustering algorithm to cluster the rows of the matrix~$\htt \U\in \mb R^{\N\times\R}$, whose columns consist of the $\R$ principal eigenvectors of $\htt \X$ that correspond to an embedding of each point in $\R$-dimensional space. In the next theorem, we derive theoretical guarantees on the mis-classification rate for the solution $\htt \Z$ obtained from this rounding procedure. 
 
 \begin{theorem}[Clustering error for rounded solution $\htt \Z$]
\label{theorem-ErrorRateZ}
Let $\htt \Z$ be the estimated cluster membership matrix obtained by applying spectral clustering on $\htt \X$ using a $(1+\epsilon)$-approximate $k$-means clustering algorithm. Define $\bar{\epsilon}$ to denote the bound on the relative estimation error of $\htt \X$ in the right hand side of \eqref{equation-ErrorRateX}. Suppose $\frac{64(2+\epsilon)\bar{\epsilon}}{n_{\min}^2}\frac{n^2}{r}\leq 1$ and the separation condition $\Delta_{\min}\geq 8\sigma_{\max} \sqrt{\D}$ hold, then with probability at least $1-2r/n_{\min}$, the cardinality of the set of misclassified data points $\mathc S_k \subset \mathc C_k$ for each  $k\in[\R]$ is upper bounded  as 
\begin{equation}
\label{eq:misclassrate}
\begin{aligned}
    \sum_{k\in[\R]} \frac{\lvert \mathc S_k \rvert} {\smalln_k}
    \leq 64(2+\epsilon) \frac{\lVert \X^0-\htt \X  \rVert_1}{n_{\min}^2}, 
\end{aligned}
\end{equation}
where $n_{\min}:=\min_{k\in[\R]} n_k>r$ denotes the cardinality of the smallest cluster.
\end{theorem}
 
\begin{remark}
 Based on our discussion in \Cref{rem-separationX}, if we adopt the dimensionality reduction procedure described in Section~\ref{sec:highdim} to first project the data on the top 
 $\R-1$ principal components, then the new separation condition for  \Cref{theorem-ErrorRateZ} to hold for the projected dataset becomes Eq~\eqref{eq:sep} as before.
\end{remark}


We note that the added condition on $\bar{\epsilon}$ is required to translate the error of $\htt \X$ to mis-classification error, and is easily satisfied. If the clusters are balanced, i.e. $n_{\min}=\Theta(n/r)$, then it will be satisfied as long as $\snr=\Omega(\log r)$, $n$ is large, and $m/n$ is small. It can also be satisfied for an unbalanced setting at the expense of a larger $\snr$ and large enough $n_{\min}$. 
Thus, from \eqref{eq:misclassrate}, we see that the average mis-classification rate per cluster for inlier data points decays exponentially in the signal-to-noise ratio as well as $N$ tends to infinity, provided the clusters are balanced and $m/n$ is sufficiently small. In our proof, we first analyze the approximate $k$-means clustering step and show that the average fraction of mis-classified data points per cluster is upper bounded by $\lVert \htt {\bm U} - \bm U^0 \bm O \rVert_{\F}$, where $\U^0\in \mb R^{\N\times\R}$ represents the $\R$ principal eigenvectors of $\X^0$ and $\bm O \in \mb R^{\R \times \R}$ is the optimal rotation matrix. Next, using the Davis-Kahan theorem~\citep{yu2014useful}, we obtain a bound on the deviation $\lVert \htt {\bm U} - \bm U^0 \bm O \rVert_{\F}$  in terms of $\|\X^0-\htt \X\|_1$.

\begin{remark}
\label{rem-ExpConst} 
Based on the minimax results obtained in \cite{lu2016statistical}, we note that for the SGMM setting in which there are no outliers, i.e., $m=0$, the error rate derived in \eqref{eq:misclassrate} is optimal up to a constant factor in the exponent. Specifically, in \citet{lu2016statistical}, the optimal rate has a factor of $1/8$ within the exponent as opposed to the $1/64$ factor that we obtain from~\eqref{equation-ErrorRateX}~and~\eqref{eq:misclassrate}. In Appendix~\ref{RemarkProof}, we show that by narrowing down the range of values that $\gamma$ can take, the $1/64$ factor in~\eqref{equation-ErrorRateX}
can be reduced to $1/33$ to obtain a tighter bound. 
\end{remark}

 \begin{remark}
It is easy to show that with minor modifications, the results in Theorems~\ref{theorem-ErrorRateX} and~\ref{theorem-ErrorRateZ} also hold respectively for the solutions $\htt \X^{\text{SDP}}$ and $\htt \Z^{\text{SDP}}$ obtained from the \ref{RelaxedSDPForm} formulation.
 \end{remark}

\subsection{Dimensionality reduction for large $\D$}
\label{sec:highdim}
In this section, we extend our analysis to high dimensional problems where $\D\gg \R$. Without loss of generality, we make the assumption that the inlier part of the data (data matrix excluding the outlier points) is centered at the origin, i.e., mean~$\bmu=\sum_{k\in[\R]} \pi_k \bmu_k=0$ for the sub-gaussian mixture model. Under this assumption, since the $\R$~mean vectors can lie in at most $\R-1$ dimensional space, we apply Algorithm~\ref{algorithm} after dimensionality reduction. 
This is similar to previous works of \cite{vempala2004spectral} on Gaussian mixture models. In order to maintain the independence of data points, similar to \cite{chaudhuri2009multi} and \cite{yan2016convex}, we split the data into two random parts. One part is used to compute the directions of maximum variance using principal component analysis (PCA) on its covariance matrix. The data points in the other part are projected along these principal directions to obtain their representations in a low-dimensional space.

In this procedure, we first randomly split the data matrix~$\Y$ into two disjoint sets $P_2$ and $P_1$ with their respective cardinalities $N_2$ and $N_1:=N-N_2$. Using the points in $P_2$, we construct the sample covariance matrix $\htt{\bs\Sigma}_2=\frac{\sum_{i\in P_2}(\y_i-\overline{\y}_2)(\y_i-\overline{\y}_2)^\top}{N_2}$ where ~$\overline{\y}_2=\frac{\sum_{i\in P_2}\y_i}{N_2}$ and obtain the matrix $\bm V^{(2)}_{\R-1} \in \mb R^{\D\times {(\R-1)}}$ whose columns consist of the top $\R-1$ eigenvectors of $\htt{\bs\Sigma}_2$ that represent the $\R-1$ principal components. We obtain the projection $\y_i'$ of each data point $i \in P_1$ by projecting $\y_i$ onto the subspace spanned by the top $\R-1$ eigenvectors of $\htt {\bs\Sigma}_2$, i.e., $\y_i'={\bm V_{\R-1}^{(2)\top}} \y_i$. Sample splitting ensures that the projection matrix is independent of the data matrix that is being projected.  Hence, the projected data points $\y_i'$ in the split $P_1$ of dataset are independent of each other. This ensures that the key assumption of independence of data points that underlies Theorems \ref{theorem-ErrorRateX} and \ref{theorem-ErrorRateZ} is satisfied. 

Next, we show that provided the number of outliers is small in comparison to the number of inlier data points, the original pairwise distances between cluster centers are largely preserved with high probability after projection. We state this result formally in the proposition below. In our result, we assume that the $\R$ cluster means span the $\R-1$ dimensional space.

\begin{proposition}
\label{lem-projection}
Assume that $\sum_{k}\pi_k \bs \mu_k=0$ and $N_2=\N^{\alpha}$ for some $0<\alpha<1$. Let $\Y^{\mathc O}\in \mb R^{m\times \D}$ denote the outlier part of the data matrix, and $\bm H:=\sum_k \pi_k \bs\mu_k \bs\mu_k^\top$ such that its smallest positive eigenvalue 
 $\eta_{r-1}(\bm H)>5\bigg(\sigma_{\max}^2+C_1\sqrt{ \frac{ 2\alpha\D N^{1-\alpha} \log N}{n  }} + C_2\bigg( \frac{m}{N}+\sqrt{\frac{\alpha\log N}{N^{\alpha}}}\bigg) \max\big\{\Delta_{\max}^2, \lVert \Y^{\mathc{O}}\rVert^2_{2,\infty}\big\}\bigg)$
for some universal constants $C_1$ and $C_2$. Then, the projections $\y_i'$ obtained for inlier data points in $P_1$ are independent sub-gaussians in $\R-1$ dimensional space. In addition, suppose $\Delta_{\min}$ denotes the minimum separation between any pair of cluster centers in the original $\D$-dimensional space, then the minimum separation after projection in the reduced space is $\Delta_{\min}/2$ with probability at least 
$1-\Tilde{O}(\R^2 N^{-\alpha})$.
\end{proposition}

The proof can be found in the Appendix. 
The condition on $\eta_{r-1}$ essentially lower bounds the separation between the cluster means. For a simple symmetric equal-sized two-component mixture model, it is easy to see that $\eta_{r-1}$ is proportional to the square of the distance between the cluster centers.
It is important to note here that the sample splitting procedure discussed in this section is mainly for theoretical convenience to ensure that the projected data points are obtained independently of each other; in practice, as discussed in \cite{chaudhuri2009multi}, this step is usually not required.  We note that the cardinality of set $P_2$ is a $N^{-(1-\alpha)}$ fraction of the total number of points in $\Y$, and hence, it vanishes for large $N$. On the other hand, the mis-classification rate for our algorithm for the balanced clusters setting is upper bounded as $\sum_{k\in[\R]} \frac{\lvert \mathc S_k \rvert} {\smalln_k}\lesssim Cr^2\exp\big(-\frac{\Delta_{\min}^2}{64\sigma_{\max}^2}\big)+C'\frac{mr}{\smalln}$, which is asymptotically non-vanishing. Therefore, the asymptotic error rate remains unaffected by sample splitting. If we make $\alpha$ very large, for example, use $N_2=N/\log N$, then the condition on the smallest eigenvalue is less restrictive, but we only label $N(1-1/\log N)$ data points.  

\subsection{Extension to weakly separated clusters}
\label{sec-OverlapCluster}
\noindent In this section, we consider the problem setup in which not all clusters have a minimum separation of~$\Delta_{\min}=8\sigma_{\max}\sqrt{d}$ between them, which is the condition required in \Cref{theorem-ErrorRateX} for the results to hold. Specifically, we extend the theoretical results obtained in Theorems~\ref{theorem-ErrorRateX} and \ref{theorem-ErrorRateZ} to show that if the separation between a pair of clusters is small, then with probability tending to one, it is possible to recover the ``weakly separated" clusters as a single merged cluster with low error rate. 

To achieve this, we define the threshold on the minimum separation to be $\Delta_0:= 8 \sigma_{\max}\sqrt{d}$. 
We classify each cluster pair $(k,l)$ as ``weakly" or ``well" separated  based on whether $\Delta_{kl}<\Delta_0$ or $\Delta_{kl}\geq \Delta_0$ respectively. Let $\mathc S_{\ov} :=\{(k,l) : \Delta_{kl}< \Delta_0\text{ for } k,l \in [r] \}$  denote the set of all weakly separated pair of cluster pairs, then we redefine the reference matrix to incorporate for weakly separated clusters as below:
\begin{equation}
\label{eq:NewReference}
    R_{ij}= 
    \begin{cases}
    \max\big\{K_{ij}, \exp\big(-\frac{r_{\tin}^2}{\theta^2}\big) \big\} &\text{ if } i,j \in \mathc C_k \text{ or if } i \in \mathc C_k,j \in \mathc C_l \text{ with } (k,l) \in S_{\ov}  \\
    \min\big\{ K_{ij}, \exp\big(-\frac{{r^{kl}_{\tout}}^2}{\theta^2}\big) \big\}&\text{ if } i\in \mathc C_k, j\in \mathc C_l \text{ with }  (k,l) \in S_{\ov}^c \\
    \gamma &\text{ if either } i\in \mathc O \text{ or } j \in \mathc O  \\
    \end{cases}
\end{equation}
\noindent Clearly, if all clusters are well separated, the reference matrix defined above reduces to the reference matrix $\bm R$ in \eqref{ReferenceDefinition}. However, under weak separation, we note that the solution~$\tilde{\X}$ obtained from the reference optimization problem~\eqref{ReferenceOptimizationProblem} corresponds to the solution where the weakly separated clusters form a single merged cluster and is of the form given below:
\begin{equation}
\label{eq:Xtilde}
    \tilde{X}_{ij}=\begin{cases}
    \ 1 & \text{if } i,j \in \mathc C_k \text{ or if } i \in \mathc C_k, j\in \mathc C_l \text{ with }  (k,l) \in S_{\ov} \\
    \ 0 & \text{otherwise}.
    \end{cases}
\end{equation}

\begin{proposition}
\label{proposition-OverlapCluster}
Let $\tilde{\X}$ be the true solution defined in \eqref{eq:Xtilde} and $\htt \X$ be the solution obtained from the \ref{RelaxedForm} formulation. Suppose $\displaystyle\Delta':=\max_{k\neq l} \{\Delta_{kl}:\Delta_{kl} < \Delta_0\}$ and $\displaystyle\tilde{\Delta}_{\min}:=\min_{k\neq l}\ \{\Delta_{kl} : \Delta_{kl}\geq \Delta_0\}$ denote respectively the maximum cluster separation below threshold $\Delta_0$ and the minimum cluster separation above~$\Delta_0$. Fix $\gamma \in \bigg(\exp\big(\frac{-5\tilde{\Delta}_{\min}^2}{32\theta^2}\big),\exp\big(\frac{-\tilde{\Delta}_{\min}^2}{2\theta^2}\big) \bigg)$ and set~$\theta~=~\Theta(\tilde{\Delta}_{\min})$. Assume that $\Delta' <\min\{\tilde{c} \tilde{\Delta}_{\min},\Delta_0 \}$, then with probability at least $1-2r/n_{\min}$, the estimation error for the inlier part of $\htt\X$ is upper bounded as 
\begin{equation}
\label{eq:overlap_inlier_thm}
        \lVert \htt \X_{\mathc I} -  \tilde{\X}_{\mathc I} \rVert_1
        \leq Cn^2\cdot \max \bigg\{ \exp\bigg(-\frac{(\tilde{\Delta}_{\min}-\Delta'/\tilde{c})^2}{64\sigma_{\max}^2}\bigg),\frac{\log n_{\min}}{n_{\min}}\bigg\}.
\end{equation}
 In addition, the relative estimation error for $\htt \X$  is
\begin{equation}
\label{eq:overlap_inlier_thm2}
    \frac{\lVert \htt \X - \tilde{\X}\rVert_1}{\lVert \tilde{\X} \rVert_1}\leq
    C' r \exp\bigg(-\frac{(\tilde{\Delta}_{\min}-\Delta'/\tilde{c})^2}{64\sigma_{\max}^2}\bigg)  + C'' r \max\bigg\{\frac{\log n_{\min}}{n_{\min}},\frac{m}{\smalln}\bigg\},
\end{equation}
Here $C,C'$ and $\tilde{c}=\frac{\sqrt{10}}{8}$ are positive constants.
\end{proposition}
To understand the result, we consider a simple example where we have a mixture model consisting of six  spherical Gaussians, each having unit variance and a between cluster separation of five units. We incrementally reduce the mean separation between the first two clusters $\Delta_{12}$, while keeping the separation between the remaining clusters as fixed. The clustering matrices $\htt \X$ obtained from the rounding step are shown in Figure~\ref{fig:Overlap_X}. As the mean separation between the first two clusters is decreased, we note that they get gradually merged in $\htt{\X}$, while the remaining part of $\htt{\X}$ corresponding to the ``well" separated clusters remains unchanged. To obtain the final clustering of points from~$\htt \X$, we first determine the number of clusters by adopting the procedure described in \Cref{sec-clusters_num} based on the multiplicity of 0 eigenvalue(s) for the normalized graph Laplacian matrix. The corresponding clustering results obtained by applying the Robust-SC algorithm are shown in Figure~\ref{fig:Overlap}. 
\begin{figure}[h!] 
    \centering
    \begin{subfigure}{\linewidth}
    \centering
    \includegraphics[width=.7\linewidth,height=1.4cm]{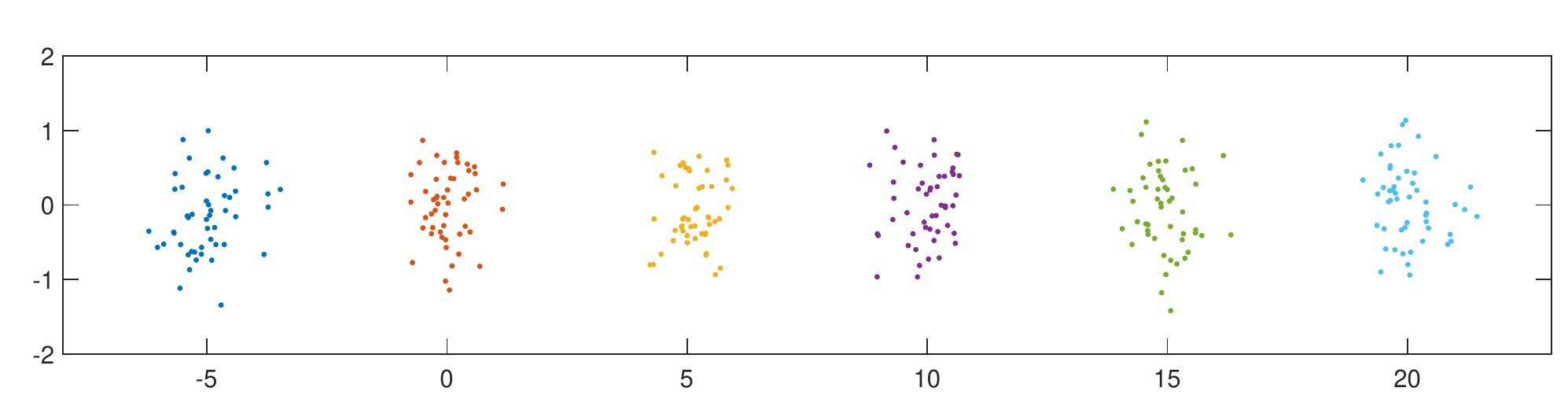}
    \end{subfigure}
    \begin{subfigure}{\linewidth}
    \centering
    \includegraphics[width=.7\linewidth,height=1.4cm]{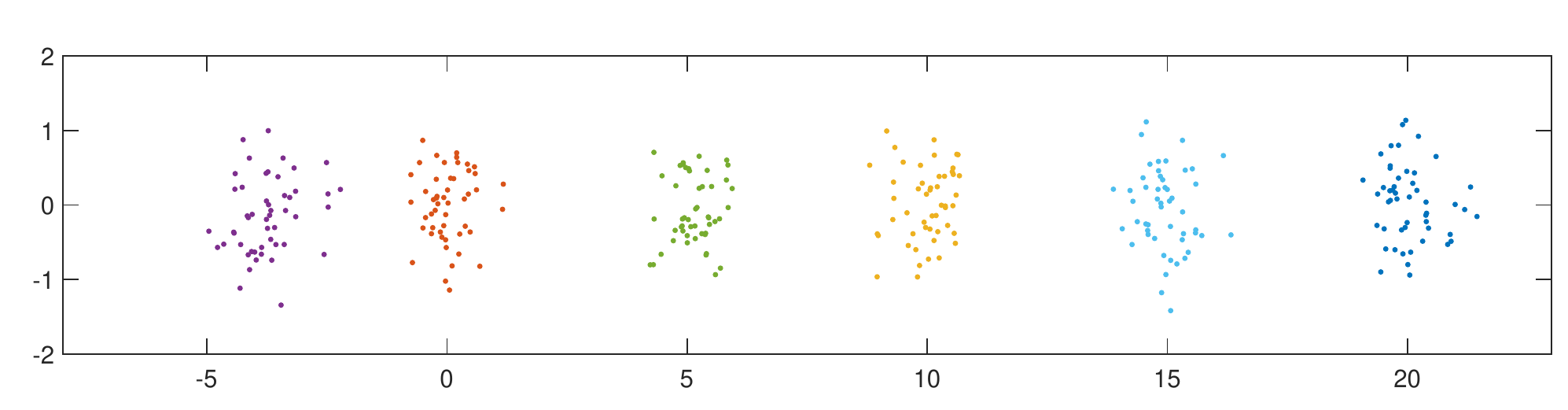}
    \end{subfigure}
    \begin{subfigure}{\linewidth}
    \centering
    \includegraphics[width=.7\linewidth,height=1.4cm]{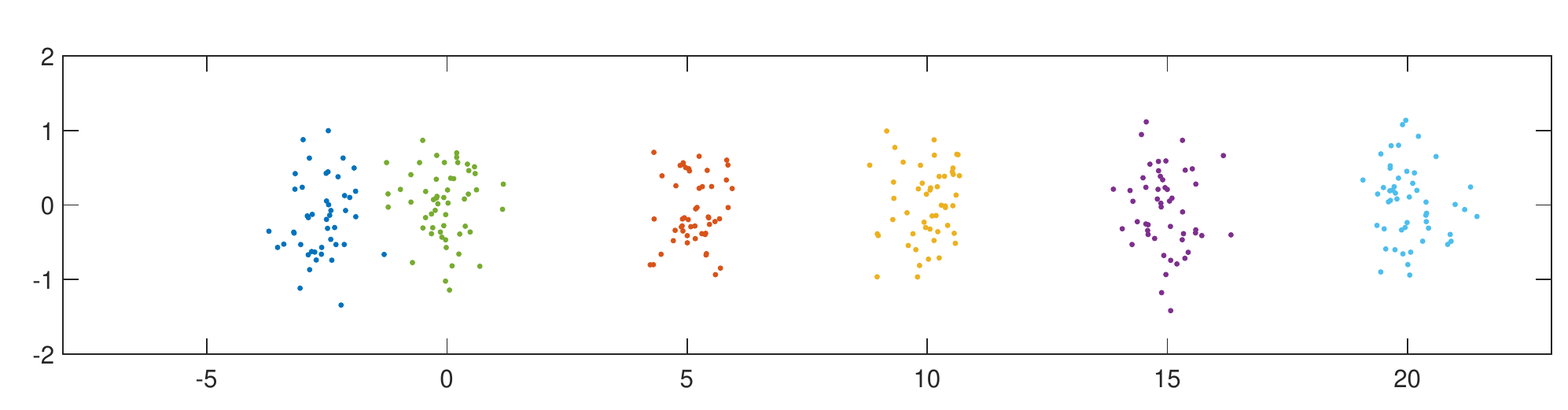}
    \end{subfigure}
    \begin{subfigure}{\linewidth}
    \centering
    \includegraphics[width=.7\linewidth,height=1.4cm]{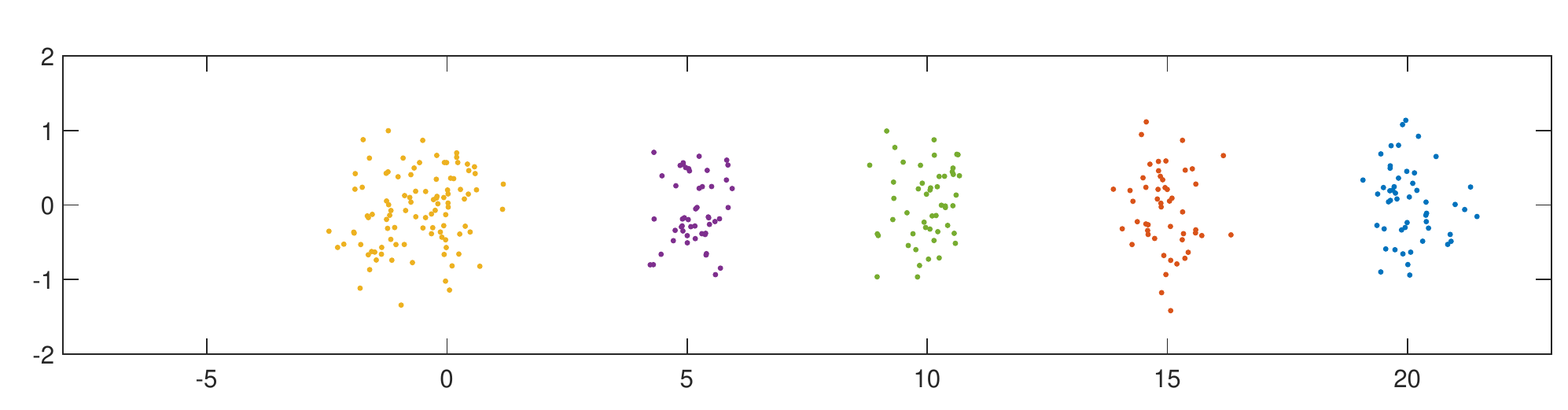}
    \end{subfigure}
    \begin{subfigure}{\linewidth}
    \centering
    \includegraphics[width=.7\linewidth,height=1.4cm]{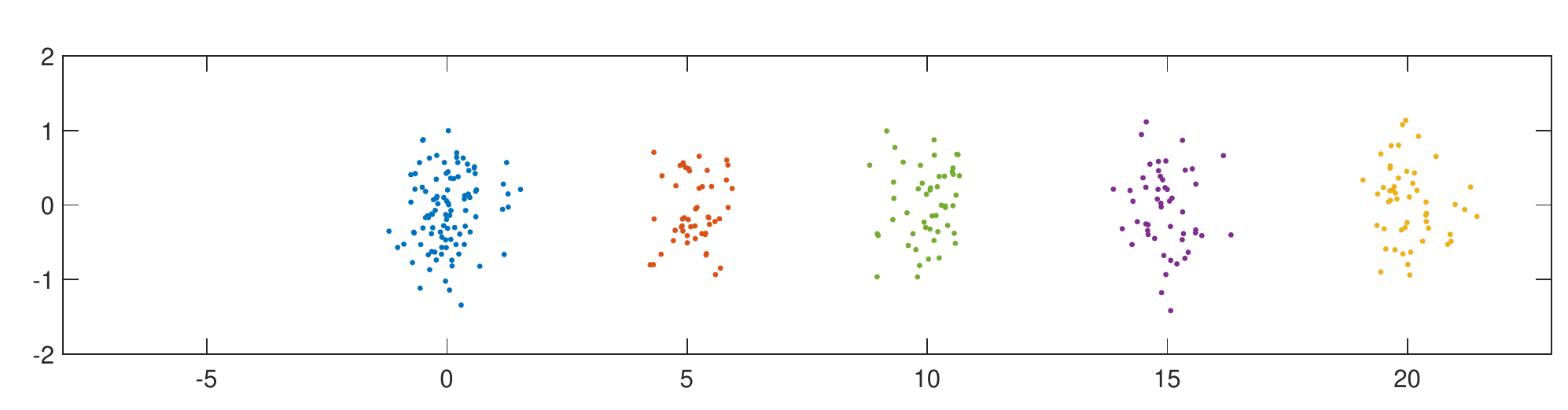}
    \end{subfigure}
    \caption{Example shows the effect of reducing the mean cluster separation below the threshold $\Delta_0$. The original dataset is obtained from a mixture of six spherical Gaussians with unit variances and a mean separation of 5 units. The separation between the first two clusters $\Delta_{12}$ is then incrementally reduced while keeping the separation between other clusters as fixed. The figures show the final clustering obtained by applying the Robust-SC algorithm. As the overlap increases, the algorithm merges the first two clusters together.}
    \label{fig:Overlap}
\end{figure}
\begin{figure}[h!]
    \centering
    \begin{subfigure}{.18\linewidth}
    \centering
    \includegraphics[width=\linewidth,height=2.6cm]{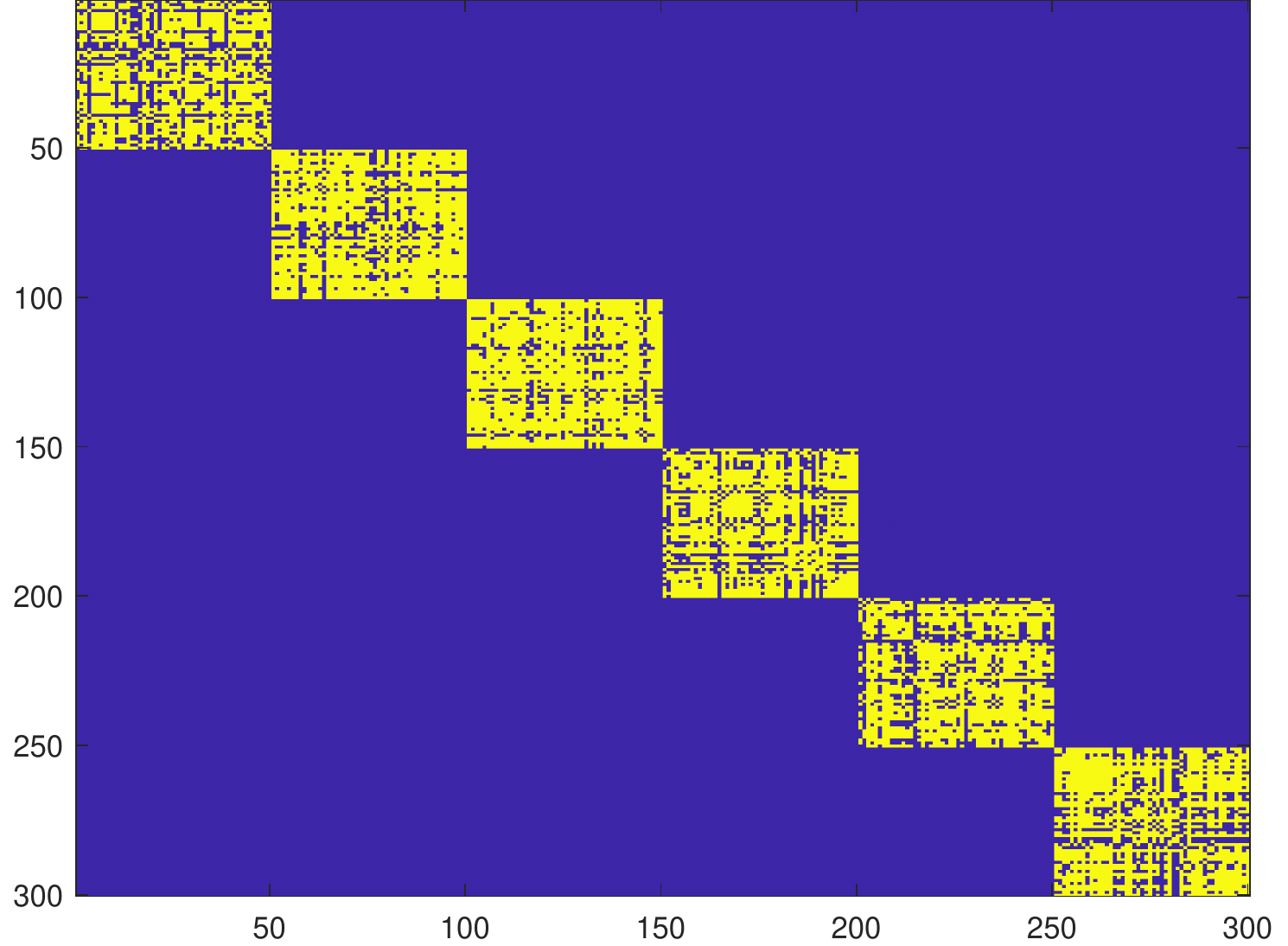}
    \end{subfigure}
    \begin{subfigure}{.18\linewidth}
    \centering
    \includegraphics[width=\linewidth,height=2.6cm]{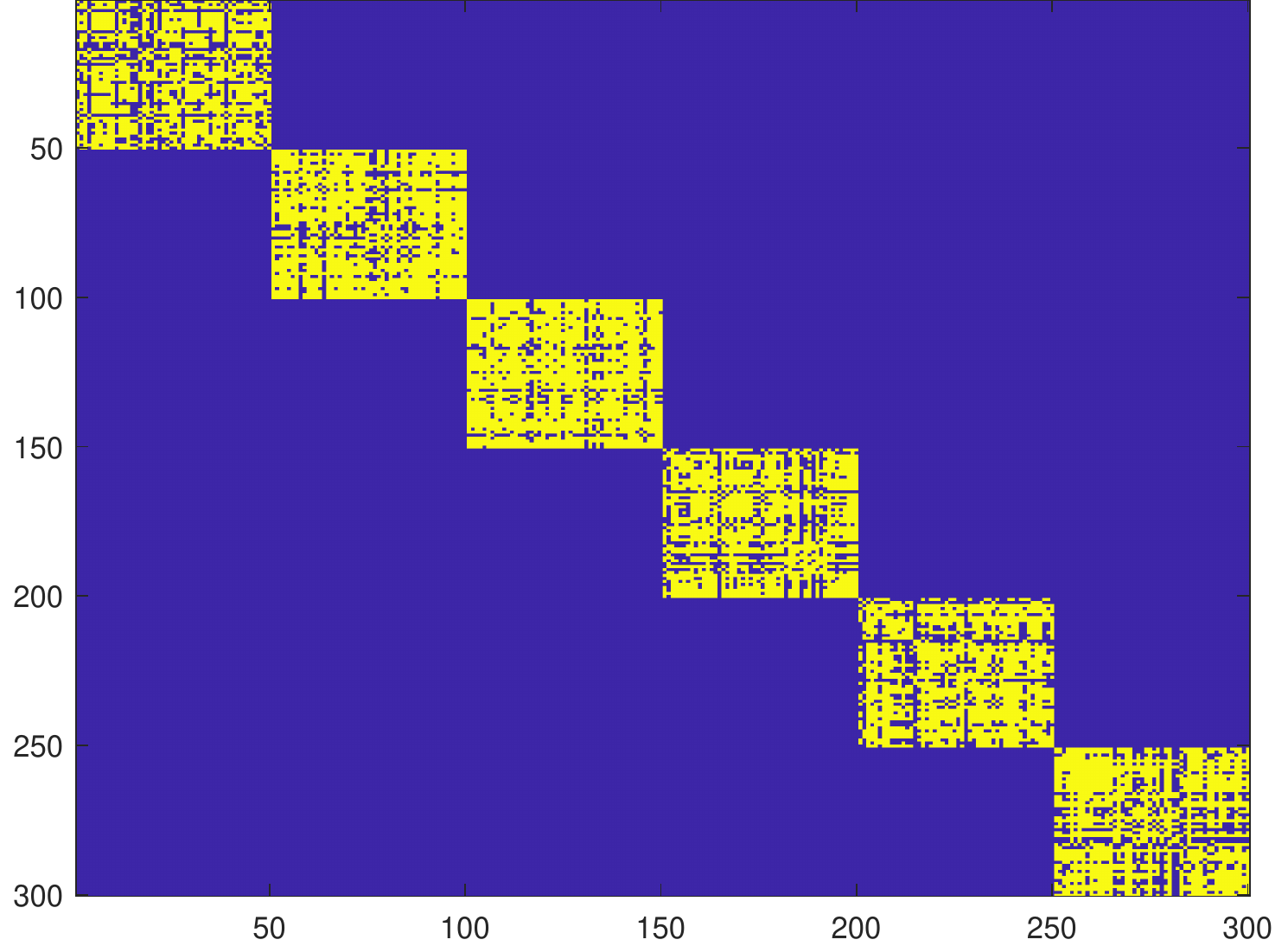}
    \end{subfigure}
    \begin{subfigure}{.18\linewidth}
    \centering
    \includegraphics[width=\linewidth,height=2.6cm]{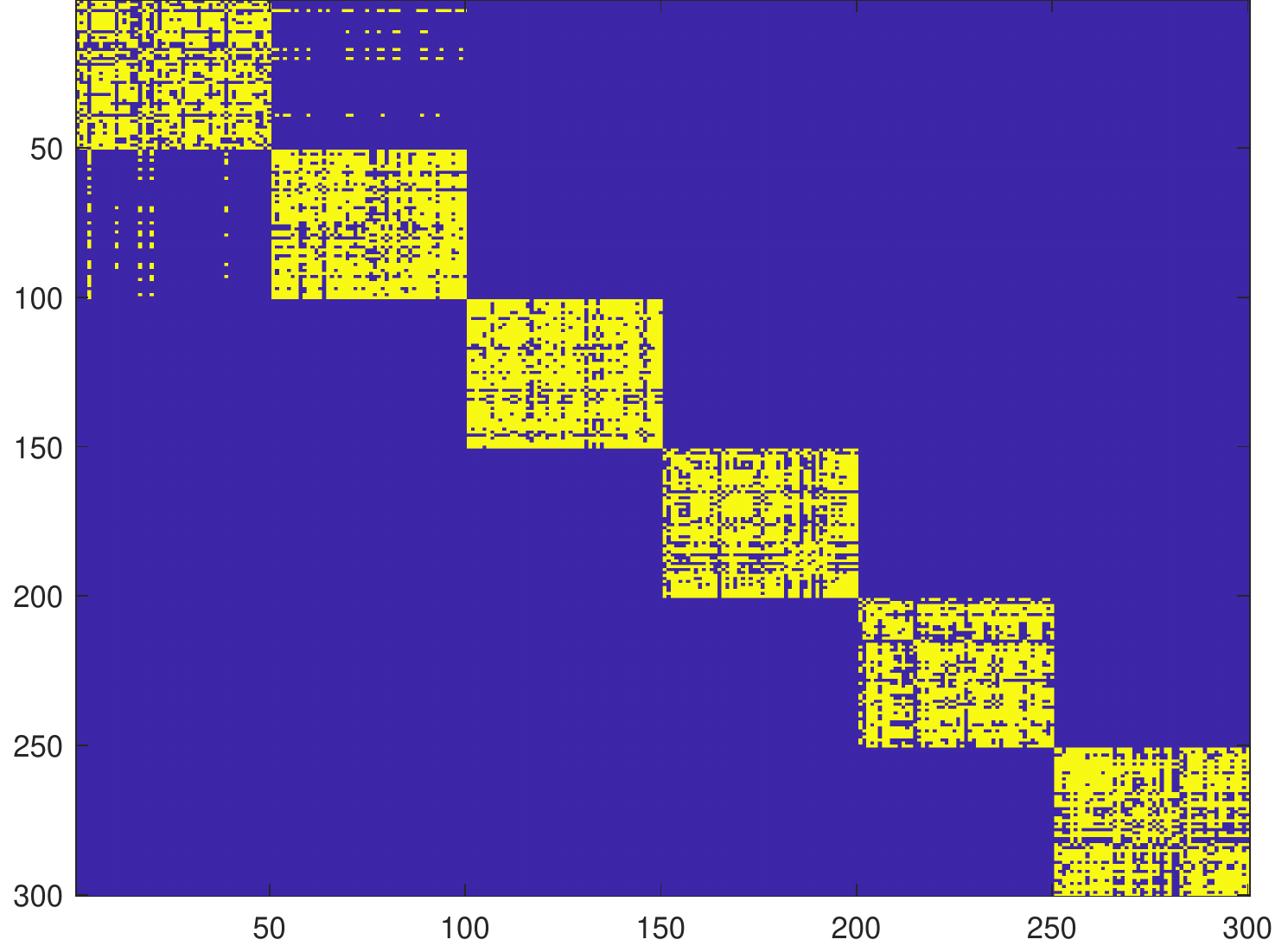}
    \end{subfigure}
    \begin{subfigure}{.18\linewidth}
    \centering
    \includegraphics[width=\linewidth,height=2.6cm]{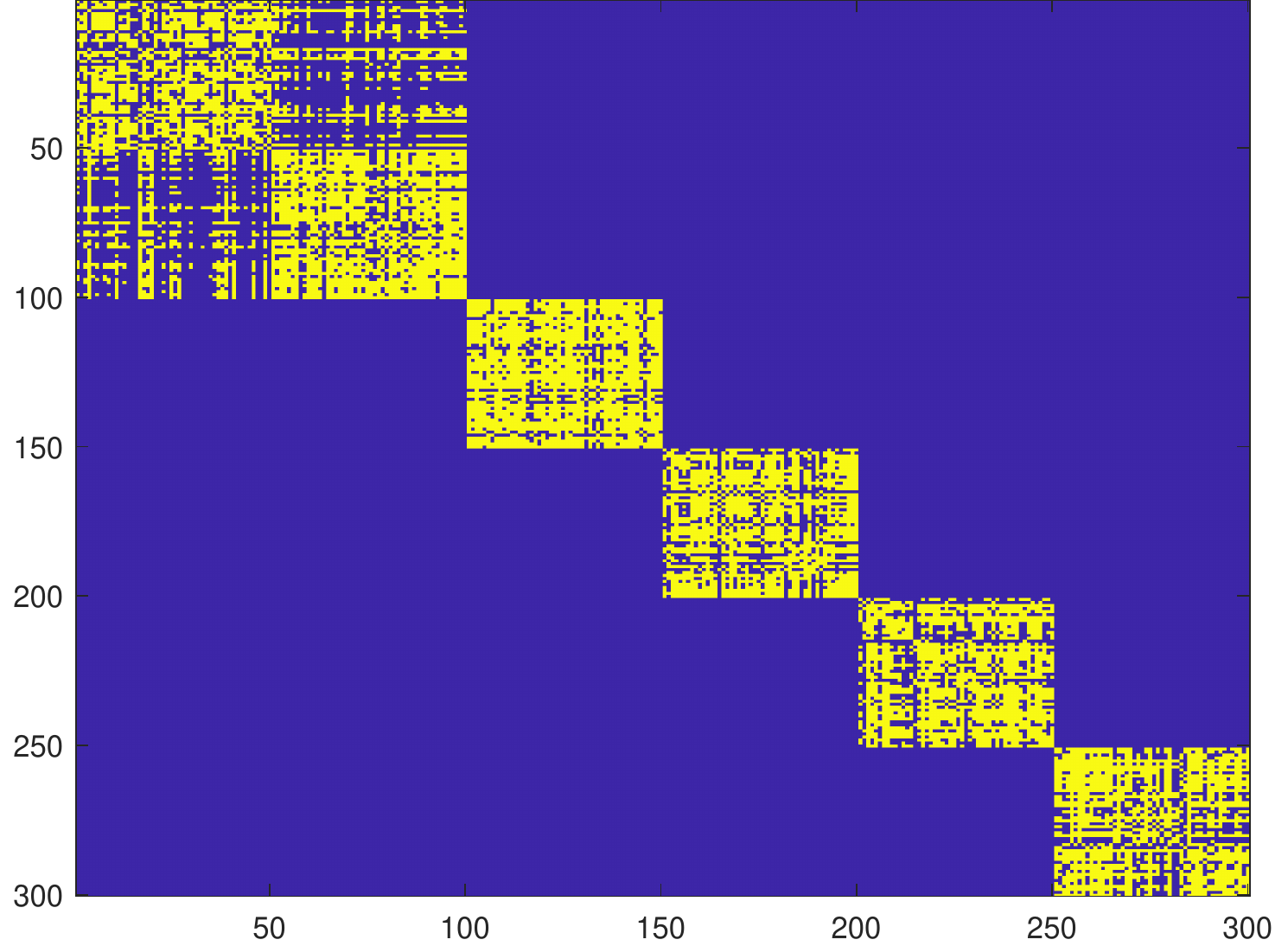}
    \end{subfigure}
    \begin{subfigure}{.18\linewidth}
    \centering
    \includegraphics[width=\linewidth,height=2.6cm]{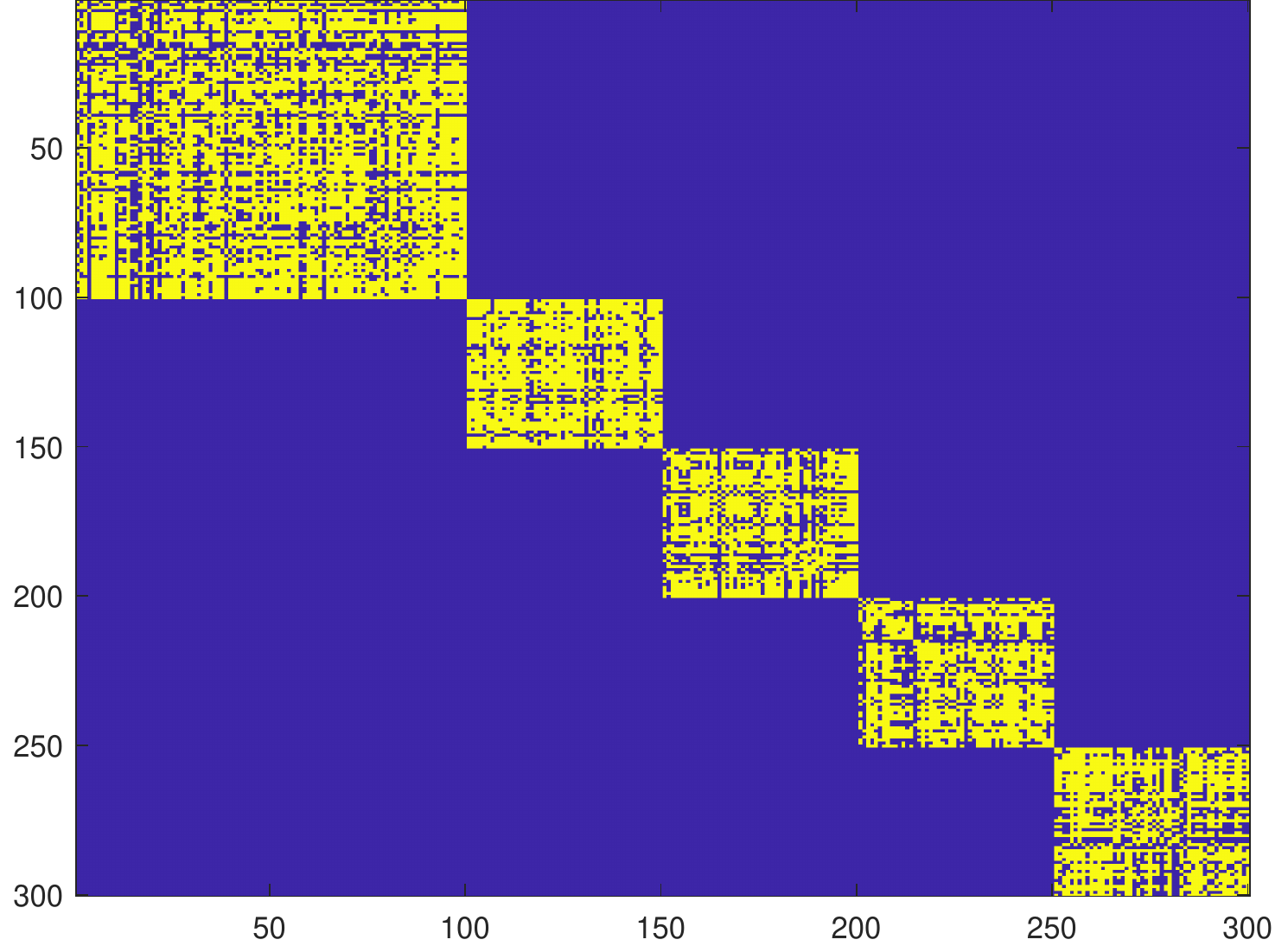}
    \end{subfigure}
    \caption{Clustering matrices $\htt\X$ obtained for different values of $\Delta_{12}$ considered in the example in~\Cref{fig:Overlap}. As $\Delta_{12}$ is decreased, the overlap between the first two clusters in $\htt \X$ increases. However, the remaining part of $\htt \X$ remains unaffected. }
    \label{fig:Overlap_X}
\end{figure}

\section{Experiments}
\label{sec-experiments}

In this section, we study the performance of our Robust-LP based spectral clustering algorithm (Robust-SC) on both simulated and real-world datasets. For our simulation studies, we conduct two different experiments. In the first experiment, we compare Robust-SC with three SDP-based clustering algorithms - (1) Robust-SDP, which is our proposed kernel clustering algorithm based on the Robust-SDP formulation; (2) Robust-Kmeans proposed by~\cite{kushagra2017provably}, which is a regularized version of the $k$-means SDP formulation in~\cite{peng2007approximating}; and (3) CC-Kmeans proposed by~\cite{rujeerapaiboon2017size}, which is another SDP-based algorithm that recovers robust solutions by imposing explicit cardinality constraints for the clusters and the outliers points. Similar to Robust-SC and Robust-SDP algorithms, the formulations for both Robust-Kmeans and CC-Kmeans are capable of identifying outliers in datasets in addition to being robust to them. Therefore, we evaluate the performance of these algorithms in terms of both the inlier clustering accuracy and the outlier detection accuracy. 

However, the SDP-based algorithms are computationally intensive to implement, and therefore, do not scale well to large scale datasets. For this reason, in the second simulation experiment, we evaluate the performance of Robust-SC on larger datasets and compare it with three additional algorithms:~(1)~$k$-means++, (2)~vanilla spectral clustering (SC), and (3) regularized spectral clustering (RegSC) \citep{joseph2016impact,zhang2018understanding}. Finally, for real-world data sets, we compare Robust-SC with all the above-mentioned algorithms.

\subsection{Implementation}
We carry out all our experiments on a quadcore 1.9 GHz Intel Core i7-8650U CPU with 16GB RAM. 
For solving different SDP instances, we use the MATLAB package SDPNAL+~\citep{yang2015sdpnal}, which is based on an efficient implementation of a provably convergent ADMM-based algorithm. 

\subsection{Performance Metric}
 We measure the performance of algorithms in terms of clustering accuracy for the inliers and the percentage of outliers we can detect. We also report the overall accuracy, which is the total number of correctly clustered inliers and correctly detected outliers divided by $N$. \bk 

\subsection{Parameter selection}
\textbf{\emph{Choice of $\bs \theta$:}} It is well known that a proper choice of scaling parameter $\theta$ in the Gaussian kernel function plays a significant role in the performance of both spectral as well as SDP-based kernel clustering algorithms. We adopt the procedure prescribed by \cite{shi2009data} for choosing a good value of $\theta$ for low-dimensional problems. The main idea is to select $\theta$ in a way such that for $(1-\alpha)\times 100\%$ of the data points, at least a small fraction $\beta$ (say around 5-10\%) of the points in the neighborhood are within the \say{range} of the kernel function. In general, the value of selected $\beta$ should be sufficiently high so that points that belong to the same cluster form a single component with relatively high similarity function values between them. Based on this idea, we choose $\theta$ as follows:
\begin{equation*}
    \theta= \frac{(1-\alpha)\text{ quantile of } \{q_1,\ldots,q_N\}}{\sqrt{(1-\alpha) \text{ quantile of }  \chi_\D^2}},
\end{equation*}
where for all points $1,\ldots,N$, each $q_i$ equals the $\beta$ quantile of the $\ell_2$-distances $\{\lVert \bm y_i - \bm y_j\rVert, j=1,\ldots,\N\}$ of point $i$ from other points in the dataset. Depending on the fraction of outlier points in the dataset, we usually choose a small value of $\alpha$ so that for a majority of inlier points, the points in the neighborhood have a considerably higher similarity value. In all our experiments, we set $\beta=0.06$ and $\alpha=0.2$. For high-dimensional problems, we use the dimensionality reduction procedure described in \Cref{sec-results} to first project the data points onto a low-dimensional space and then apply the above procedure to choose $\theta$.\\
 \noindent \textbf{\emph{Choice of $\bs \gamma$:}} Based on our discussion in \Cref{sec-algorithm}, the parameter $\gamma$ plays an equally important role in the performance of the \ref{RelaxedForm} formulation. For our experiments on simulated datasets, we choose the following value of $\gamma$:
 \begin{equation*}
    \gamma=  \exp{\bigg(-\frac{t_\alpha}{2}\bigg)},
\end{equation*}
where $t_\alpha= (1-\alpha)$ quantile of $\chi^2_\D$. This value is obtained by setting the distance in the Gaussian kernel function to equal the $(1-\alpha)\text{ quantile value of } \{q_1,\ldots,q_N\}$.

\subsection{Simulation studies}

\subsubsection{Comparison with SDP-based algorithms:} For the experiments in this section, we construct three synthetic datasets - (1) Balanced Spherical GMMs, (2) Unbalanced Spherical GMMs, and (3) Balanced Ellipsoidal GMMs. These datasets have been obtained from a mixture of linearly separable Gaussians, and explore the effect of varying different model parameters like $\bs\pi$, $\{\bs \mu_1,\dots,\bs \mu_\R\}$, and $\{\bs\Sigma_1,\dots,\bs\Sigma_\R\}$  on the performance of the algorithms. In all of these datasets, we add outlier points in the form of uniformly distributed noise to the clusters. \Cref{table-ModelSpecifications} lists out the model specifications for these synthetically generated datasets. \Cref{fig:OriginalDatasets} depicts these datasets; in each figure, the clusters formed by the inlier points are represented in different colors by solid circles, while the outlier points are marked with red crosses. 

\begin{table}[htb!]
\centering
\renewcommand{\arraystretch}{1.1}
{
    \begin{tabular}{ l l} \hline
    \multicolumn{1}{c}{\textbf{  Dataset}} & \multicolumn{1}{c}{\textbf{Model Specifications}}  \\ \hline 
         1. Balanced Spherical GMMs & $\bmu_1=[0,0]^\top,\bmu_2=[6,3]^\top,\bmu_3=[6,-3]^\top$\\
         &$\bs\Sigma_1=\bs\Sigma_2=\bs\Sigma_3=\Diag([1,1])$\\
         & $\smalln_1=\smalln_2=\smalln_3=150,m=50$\\
         2. Unbalanced Spherical GMMs & 
         $\bmu_1=[0,0]^\top,\bmu_2=[20,3]^\top,\bmu_3=[20,-3]^\top$\\
         &$\bs\Sigma_1=\Diag([5,5]), \bs\Sigma_2=\bs\Sigma_3=\Diag([0.5,0.5])$\\
         & $\smalln_1=500,\smalln_2=\smalln_3=150,m=50$\\
         3. Balanced Ellipsoidal GMMs & $\bmu_1=[0,5]^\top,\bmu_2=[0,-5]^\top,\bs\Sigma_1=\bs\Sigma_2=\Diag([20,1])$ \\
         &  $\smalln_1=\smalln_2=200$, $m=25$\\
         \hline
    \end{tabular}}
\caption{Model specifications for synthetic datasets.\label{table-ModelSpecifications}}
\end{table}

\begin{figure}[h]
\centering
    \begin{subfigure}[t]{0.28\textwidth}
        \centering
        \includegraphics[width=3.7cm, height=2.9cm]{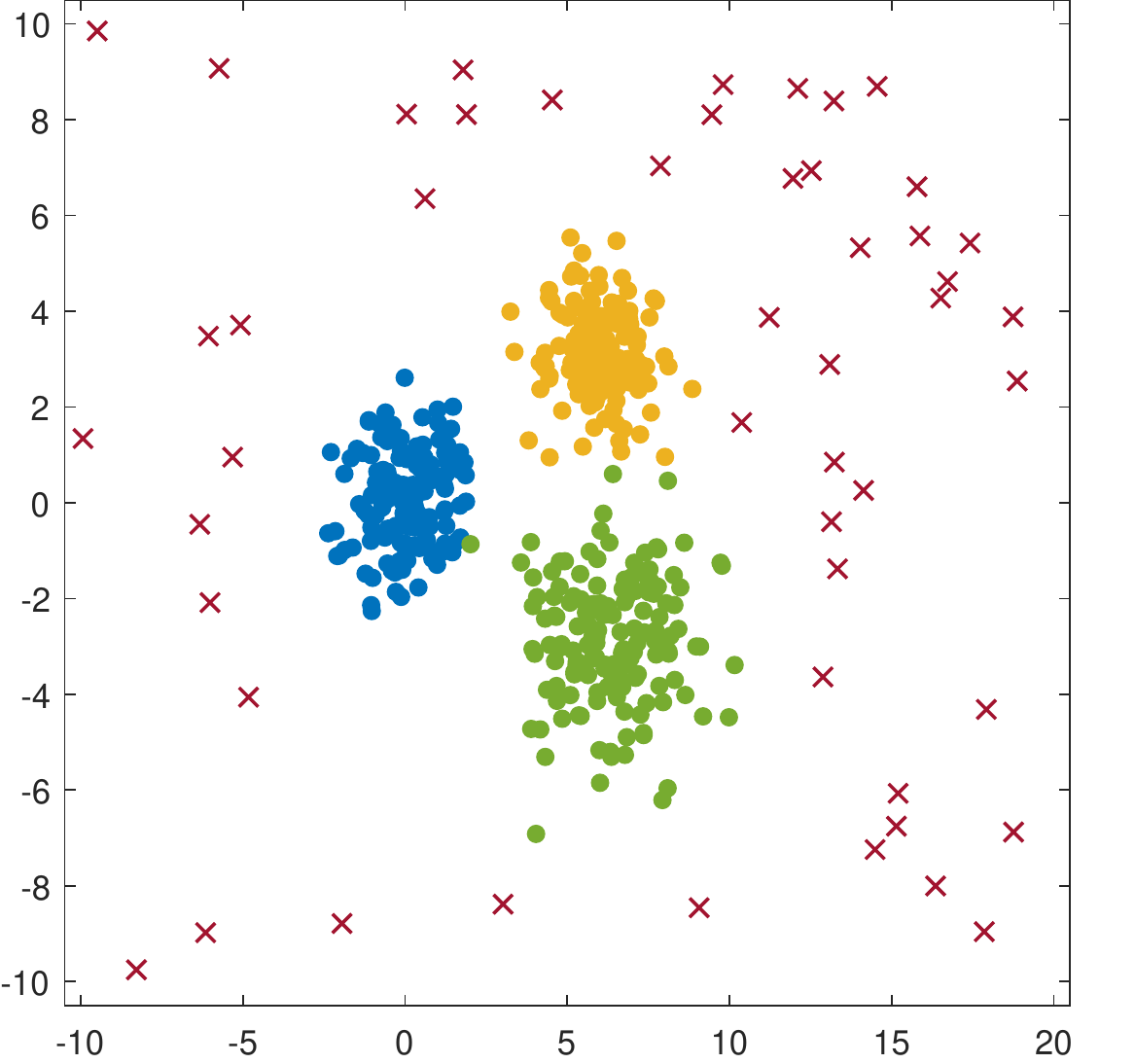}
        \caption{Balanced Spherical GMMs}
    \end{subfigure}
    \begin{subfigure}[t]{0.28\textwidth}
        \centering
        \includegraphics[width=3.7cm, height=2.95cm]{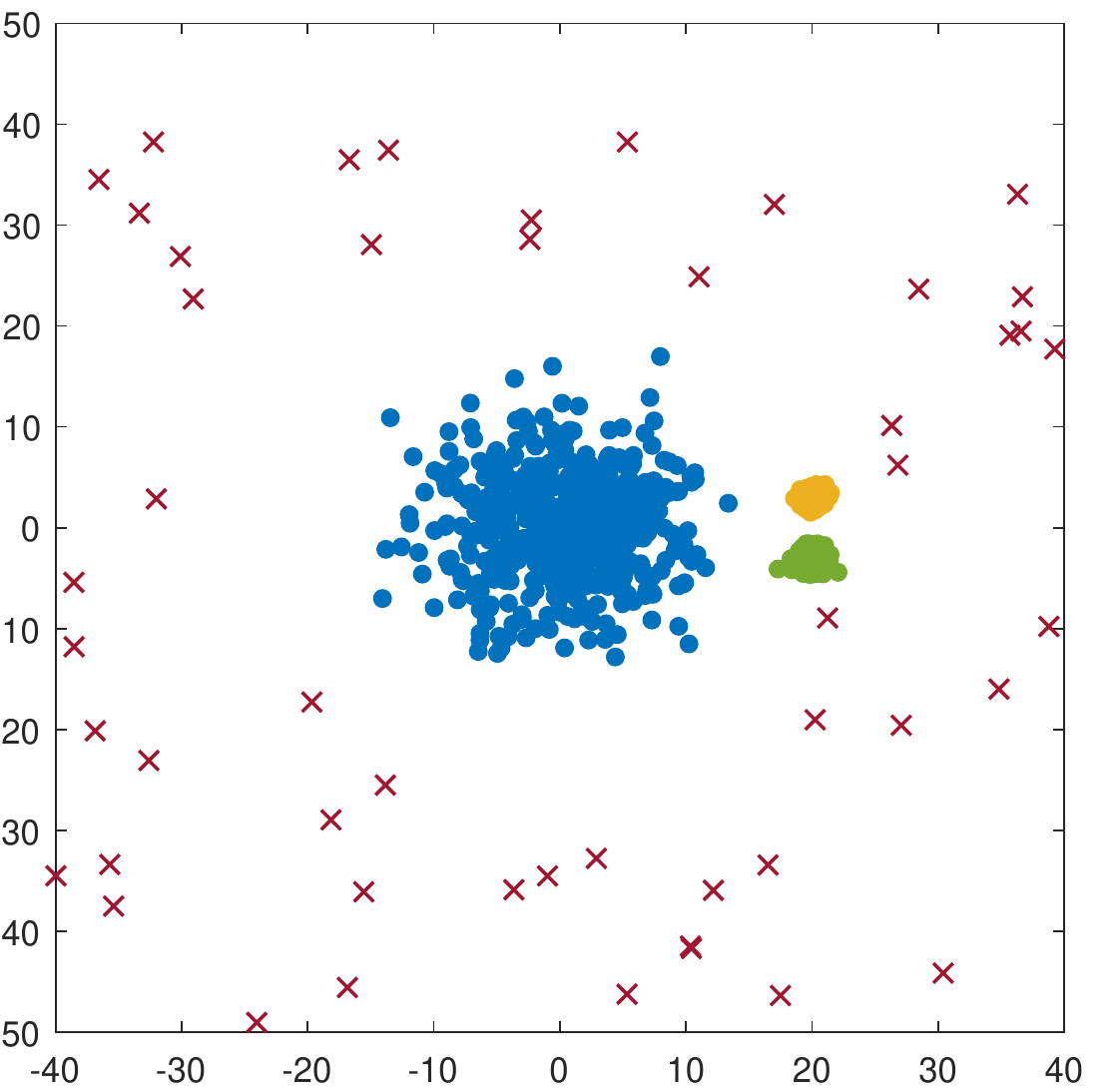}
        \caption{Unbalanced Spherical GMMs}
    \end{subfigure}
    \begin{subfigure}[t]{0.3\textwidth}
        \centering
        \includegraphics[width=4.5cm,height=3.2cm]{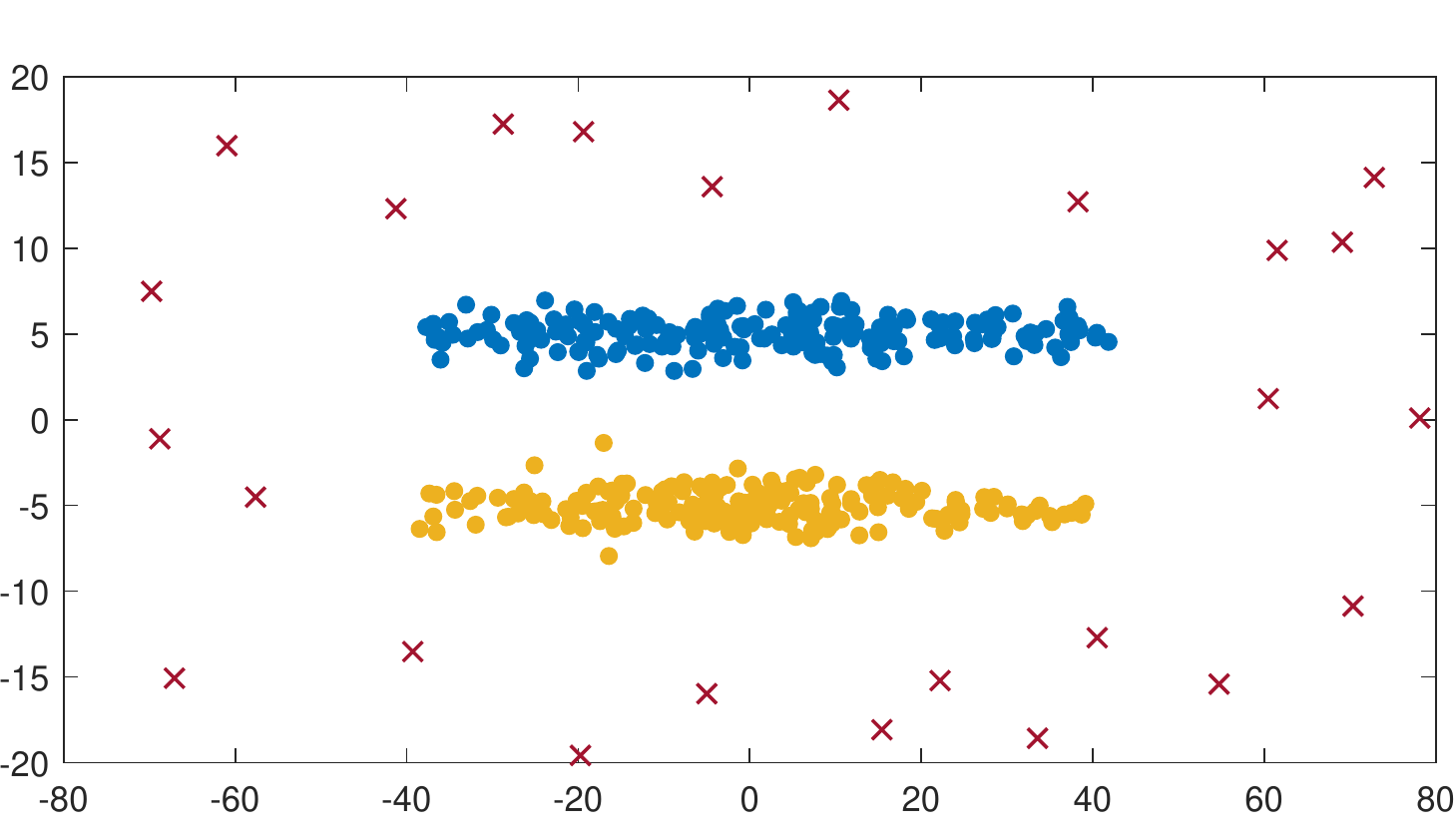}
        \caption{Balanced Ellipsoidal GMMs}
    \end{subfigure}
    \caption{Synthetic datasets generated for evaluating the performance of clustering algorithms.}
    \label{fig:OriginalDatasets}
\end{figure}

As discussed earlier in this section, we compare the performance of our Robust-SC and Robust-SDP algorithms with two other SDP-based robust formulations, namely Robust-Kmeans and CC-Kmeans. In addition to  explicitly requiring the number of outliers and cardinalities for all clusters as inputs, the CC-Kmeans algorithm suffers from several drawbacks. First, in contrast to both Robust-SDP and Robust-Kmeans, the algorithm requires solving the SDP formulation twice - once, to identify the outliers, and second, to recover the clusters after the outliers have been removed. Secondly and more importantly, the CC-Kmeans formulation for $r$ clusters, in general, requires defining $\R $ separate matrix decision variables of dimensions $(\N+1)\times(\N+1)$, each with a positive semidefinite constraint. Due to extensive memory and computational requirements, the CC-Kmeans SDP could not be implemented on the synthetic datasets for the listed model specifications in \Cref{table-ModelSpecifications}. However, despite its several shortcomings, CC-Kmeans does provide us with a benchmark on the solution quality provided the clustering problem has been entirely specified. Therefore, we try to evaluate the performance of CC-Kmeans algorithm by considering a smaller dataset with a total of around 150-200 data points in each dataset, obtained by sampling an equal number of points from each cluster. We deliberately choose the clusters to be equal-sized for CC-Kmeans because when the clusters are equal-sized, the number of SDP variables per problem instance can be reduced (although each instance does need to be solved $\R$ times), thereby making the problem computationally tractable.

For each dataset in \Cref{table-ModelSpecifications}, we generate 10 samples for the stated model specification and obtain clustering results for each algorithm except CC-Kmeans, for which we perform a single simulation run. Based on the implementation times in \Cref{table-ImplementationTimes}, it is quite evident that the CC-Kmeans algorithm is considerably slower (at least 10-20 times) compared to the other SDP algorithms even for a down-sampled dataset, and therefore, we do not show further experiments on CC-Kmeans in our simulation study. 

We summarize the results obtained in Table \ref{table-SyntheticResults}. For each dataset, we report the performance of the algorithms with respect to three metrics: (i) inlier clustering accuracy, (ii) outlier detection accuracy, and (iii) overall accuracy. On the Balanced Spherical GMMs dataset, all the algorithms perform equally well with more than $95\%(\pm 2\%)$ overall accuracy. For the Unbalanced Spherical GMMs dataset, Robust-SC and Robust-SDP are comparable with about $98\%(\pm 0.6 \%)$ overall accuracy, whereas Robust-Kmeans performs poorly with about $56\% (\pm 2\%)$ overall accuracy. Similarly, for the Balanced Ellipsoidal GMMs dataset, Robust-SC and Robust-SDP have similar accuracy values of $97.31\%(\pm 0.6\%)$ and $93.86\%(\pm 5 \%)$, whereas Robust-Kmeans has a poor accuracy of $50.52\%(\pm 1\%)$.

Based on the high accuracy values for inlier and outlier data points, Robust-SC and Robust-SDP consistently provide high quality solutions in terms of recovering the true clusters for inlier data points as well as identifying outliers in the dataset. On the other hand, while Robust-Kmeans and CC-Kmeans perform well for 
the Balanced Spherical GMMs dataset, they fail either on the Unbalanced Spherical GMMs dataset, where the clusters are unbalanced in terms of their cluster cardinalities (refer to \Cref{subfig:UnbalancedSphericalGMMs}), or the Balanced Ellipsoidal GMMs dataset, where the clusters have significantly different variances along different directions (refer to \Cref{subfig:BalancedEllipsoidalGMMs}). 
\begin{table}[h!]
    \centering
    \renewcommand{\arraystretch}{1}
    \begin{tabular}{ll l ll ll ll} \hline
    \textbf{Dataset} & \multicolumn{2}{c}{\textbf{Robust-SC}} & \multicolumn{2}{c}{\textbf{Robust-SDP}} & \multicolumn{2}{c}{\textbf{Robust-Kmeans}} & \multicolumn{2}{c}{\textbf{CC-Kmeans}} \\ \hline
         Balanced Spherical 
         & Inlier & 0.9902 & Inlier &  0.9836 & Inlier & 0.9660 & Inlier & 1.0000\\
         & Outlier& 0.9840 & Outlier & 0.9080  & Outlier & 0.7540 & Outlier & 1.0000
         \\
         & Overall & 0.9896 & Overall & 0.9760 & Overall & 0.9448 & Overall & 1.0000 \\
         Unbalanced Spherical   
         & Inlier & 0.9914   & Inlier & 0.9908  & Inlier & 0.5360 & Inlier & 0.9667       \\
         & Outlier  & 0.9680 & Outlier & 0.8840 & Outlier & 0.9240 & Outlier & 0.9600 \\
         & Overall & 0.9900  & Overall & 0.9845 & Overall & 0.5588 & Overall & 0.9650\\ 
         Balanced Ellipsoidal & Inlier & 0.9468 & Inlier & 0.9840 & Inlier & 0.5038 &  Inlier & 0.4933   \\ 
         & Outlier &	0.8080 & Outlier & 0.8000 &  Outlier & 0.5280 &  Outlier &  0.6800\\
         & Overall & 0.9386 & Overall & 0.9731 & Overall & 0.5052 & Overall & 0.5200\\
         \hline
    \end{tabular}
    \caption
    {Performance of clustering algorithms on synthetic datasets. The table reports the performance of Robust-SC, Robust-SDP, and Robust-Kmeans algorithms in terms of their inlier clustering accuracy, outlier detection accuracy, and overall accuracy for synthetic datasets, averaged over 10 simulation runs. For CC-Kmeans, the algorithm could not be implemented for the entire dataset due to memory and computational limitations. Therefore, for comparison, we specify the results for a single simulation on a down-sampled dataset with an equal number of points from each cluster.
    \label{table-SyntheticResults}}
\end{table}


In addition, we note that while there is very little difference between Robust-SC and Robust-SDP in terms of solution quality, Robust-SC is orders of magnitude faster than Robust-SDP and other SDP-based algorithms in terms of solution times (refer to \Cref{table-ImplementationTimes}).

\begin{table}[t!]
\centering
\begin{tabular}{m{5cm}cccc} \hline
      \textbf{Dataset} &\textbf{Robust-SC}  &\textbf{Robust-SDP} & \textbf{Robust-Kmeans} & \textbf{\textbf{CC-Kmeans}}\\ 
     \hline
     Balanced Spherical GMMs
     & 3.24 & 265.62 &  355.65 & 3718 \\
     Unbalanced Spherical GMMs
     & 3.18 & 828.56 & 1064.11 & 5726 \\
     Balanced Ellipsoidal GMMs
     & 2.71  & 273.52 &  123.74 & 1944 \\
     \hline
    \end{tabular}
    \caption{Solution times (in seconds) for different clustering algorithms on synthetic datasets. For Robust-SC, Robust-SDP, and Robust-Kmeans, the solution times are specified for the entire dataset, averaged over 10 simulation runs. For CC-Kmeans, the algorithm could not be implemented for the entire dataset due to memory and computational limitations. Therefore, for comparison, we specify the run-time for a single simulation on a down-sampled dataset with an equal number of points from each cluster.\label{table-ImplementationTimes}}
\end{table}

\begin{figure}[h!]
    \centering
    \begin{subfigure}[t]{\textwidth}
    \centering
    \begin{minipage}[t]{0.18\textwidth}
    \centering
    \includegraphics[width=\linewidth]{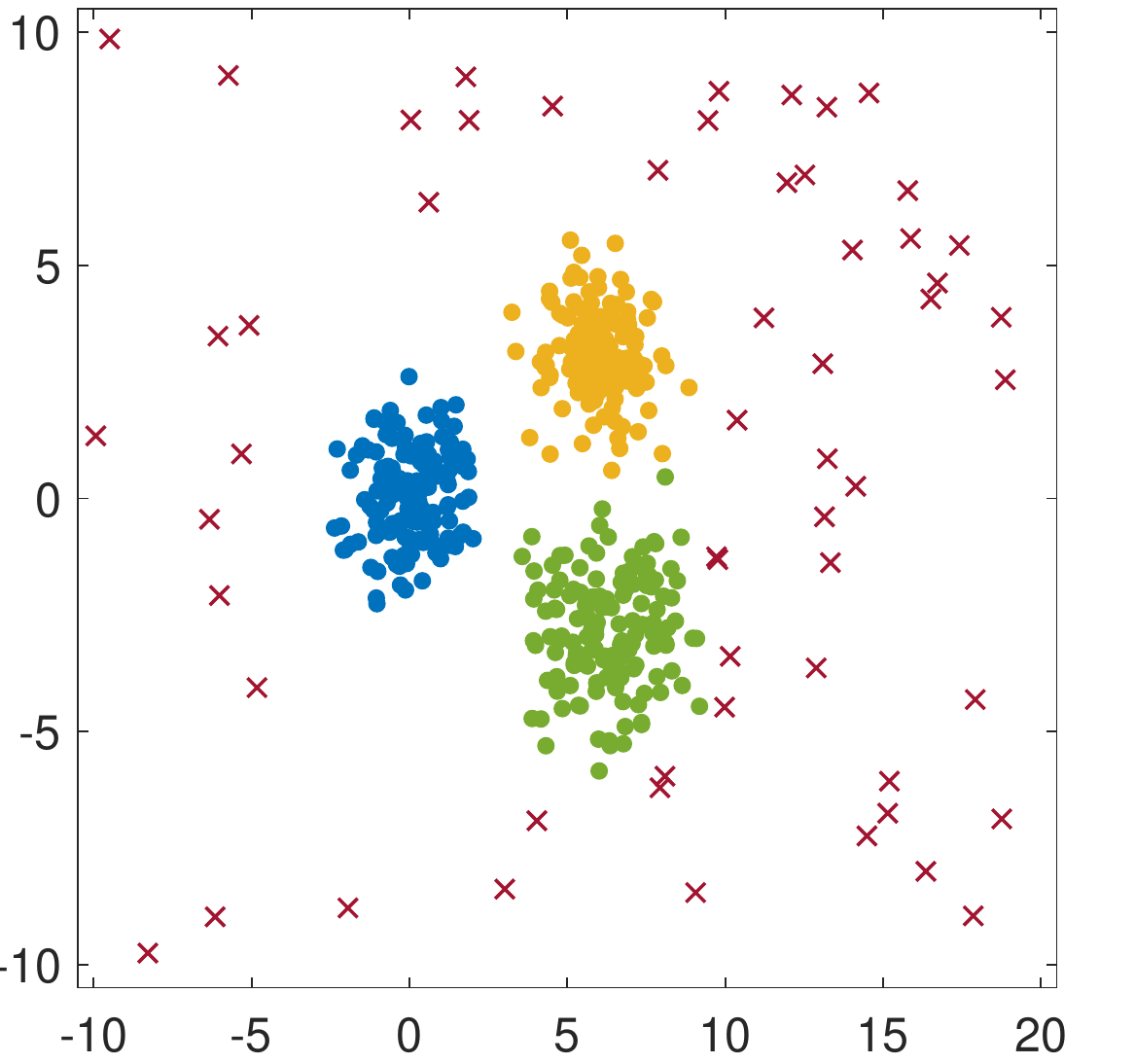}
    \caption*{Robust-SC}
    \end{minipage}
    \begin{minipage}[t]{0.18\textwidth}
    \includegraphics[width=\linewidth]{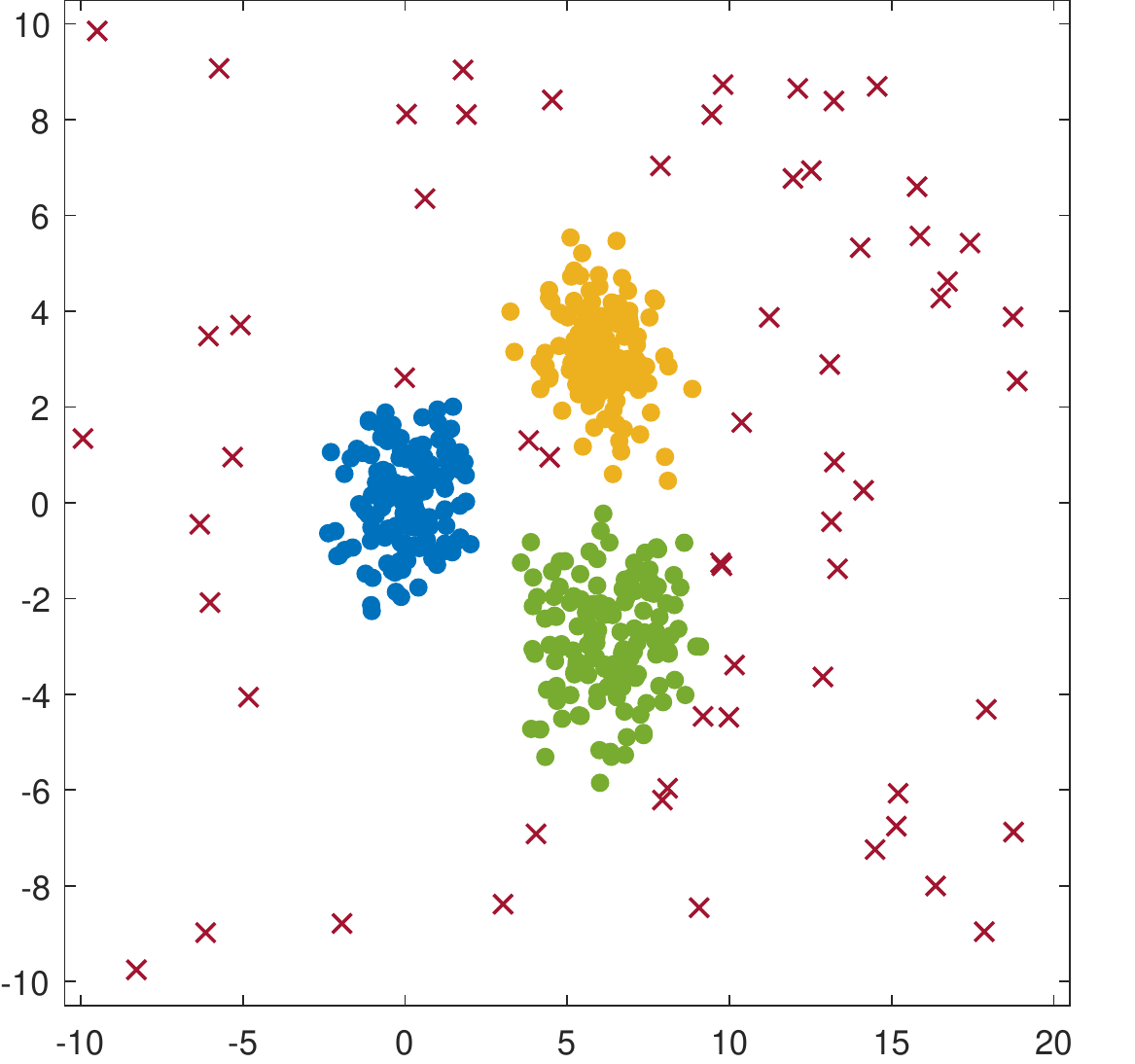}
    \caption*{Robust-SDP}
    \end{minipage}
    \begin{minipage}[t]{0.18\textwidth}
    \includegraphics[width=\linewidth]{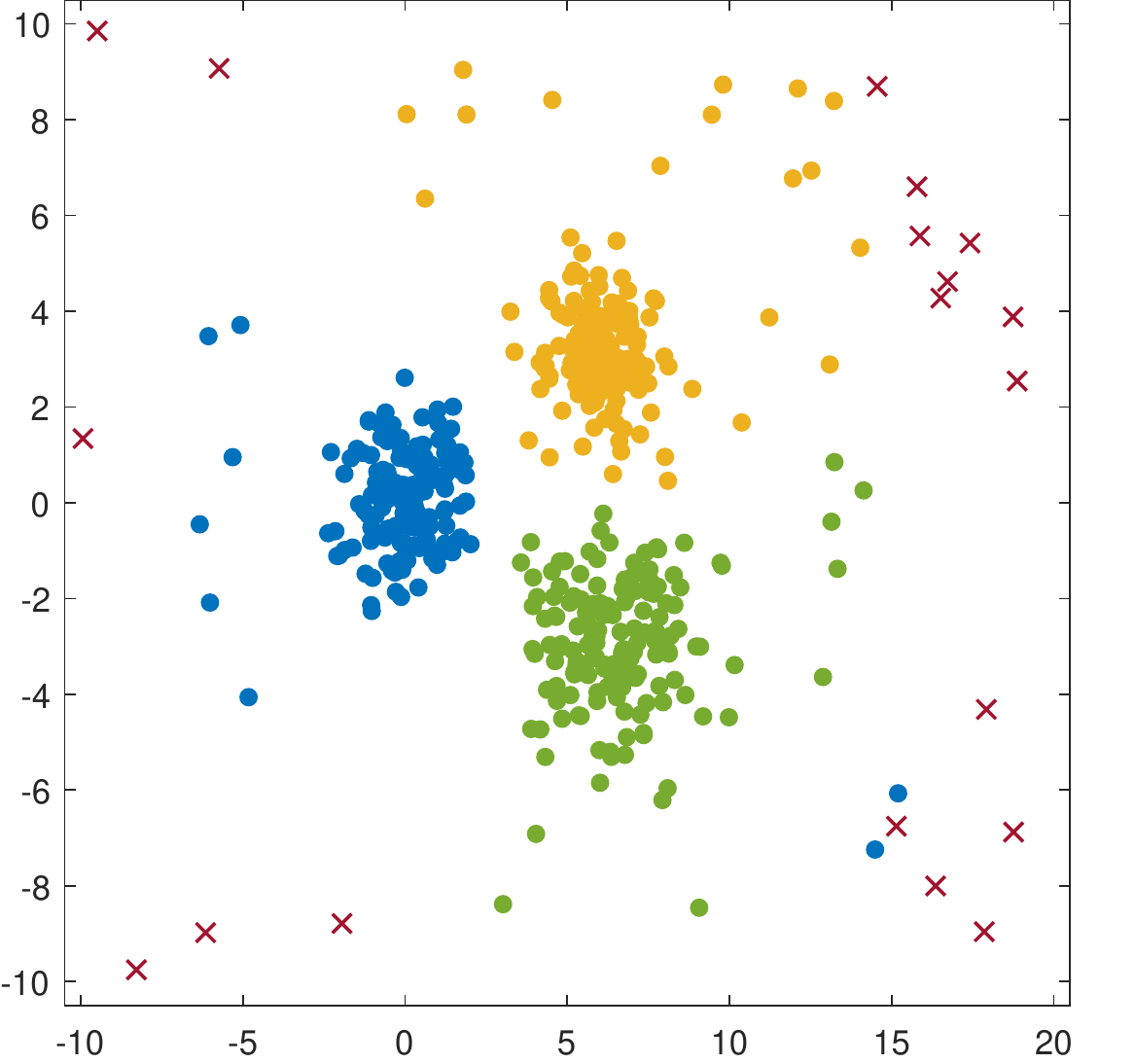}
    \caption*{Robust-Kmeans}
    \end{minipage}
    \begin{minipage}[t]{0.18\textwidth}
    \includegraphics[width=\linewidth]{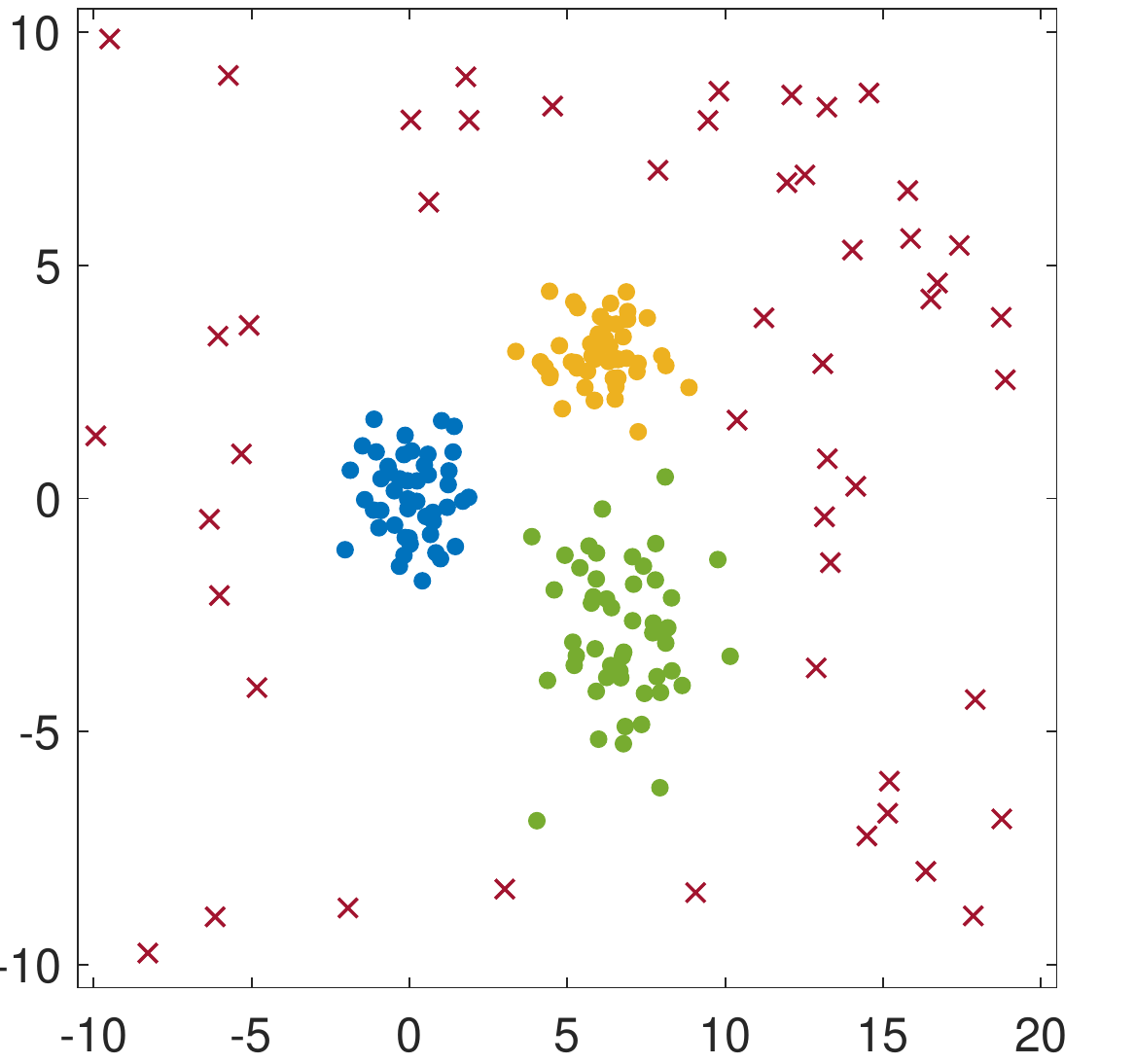}
    \caption*{CC-Kmeans on down-sampled dataset}
    \end{minipage}
    \caption{Balanced Spherical GMMs}
    \label{subfig:BalancedSphericalGMMs}
    \end{subfigure}
    
    \begin{subfigure}[t]{\textwidth}
    \centering
    \begin{minipage}[t]{0.18\textwidth}
    \centering
    \includegraphics[width=\linewidth]{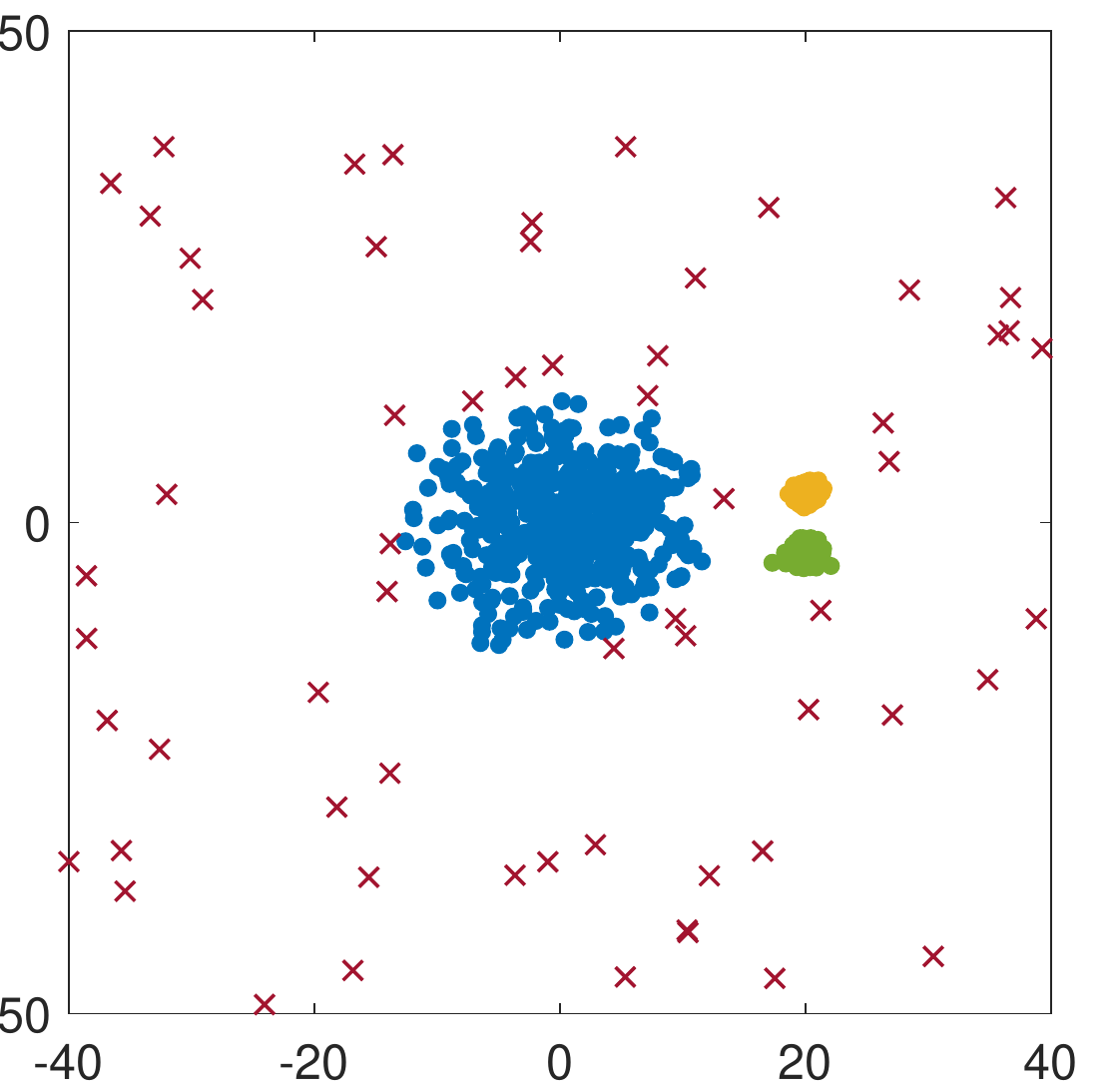}
    \caption*{Robust-SC}
    \end{minipage}
    \begin{minipage}[t]{0.18\textwidth}
    \includegraphics[width=\linewidth]{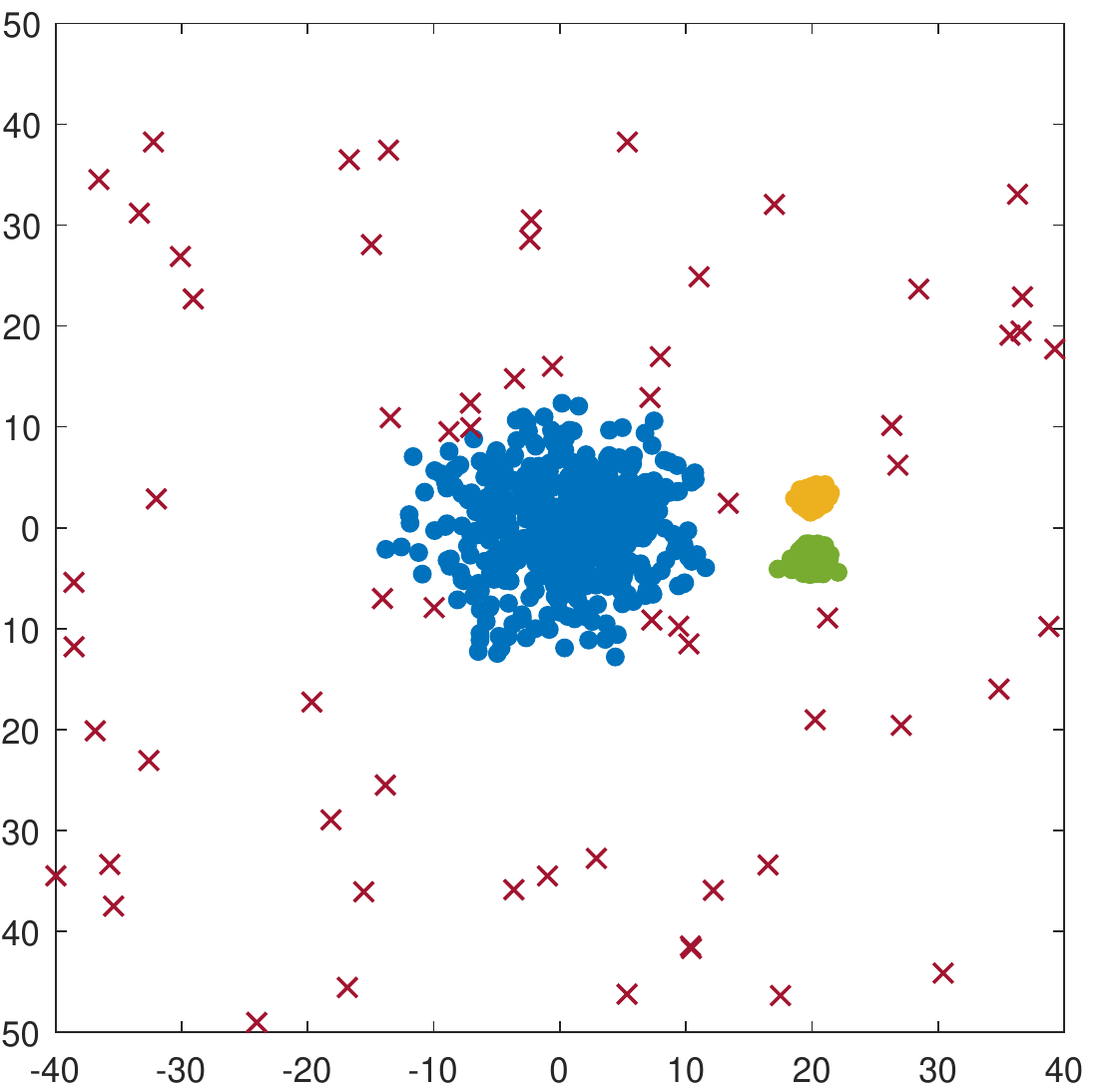}
    \caption*{Robust-SDP}
    \end{minipage}
    \begin{minipage}[t]{0.18\textwidth}
    \includegraphics[width=\linewidth]{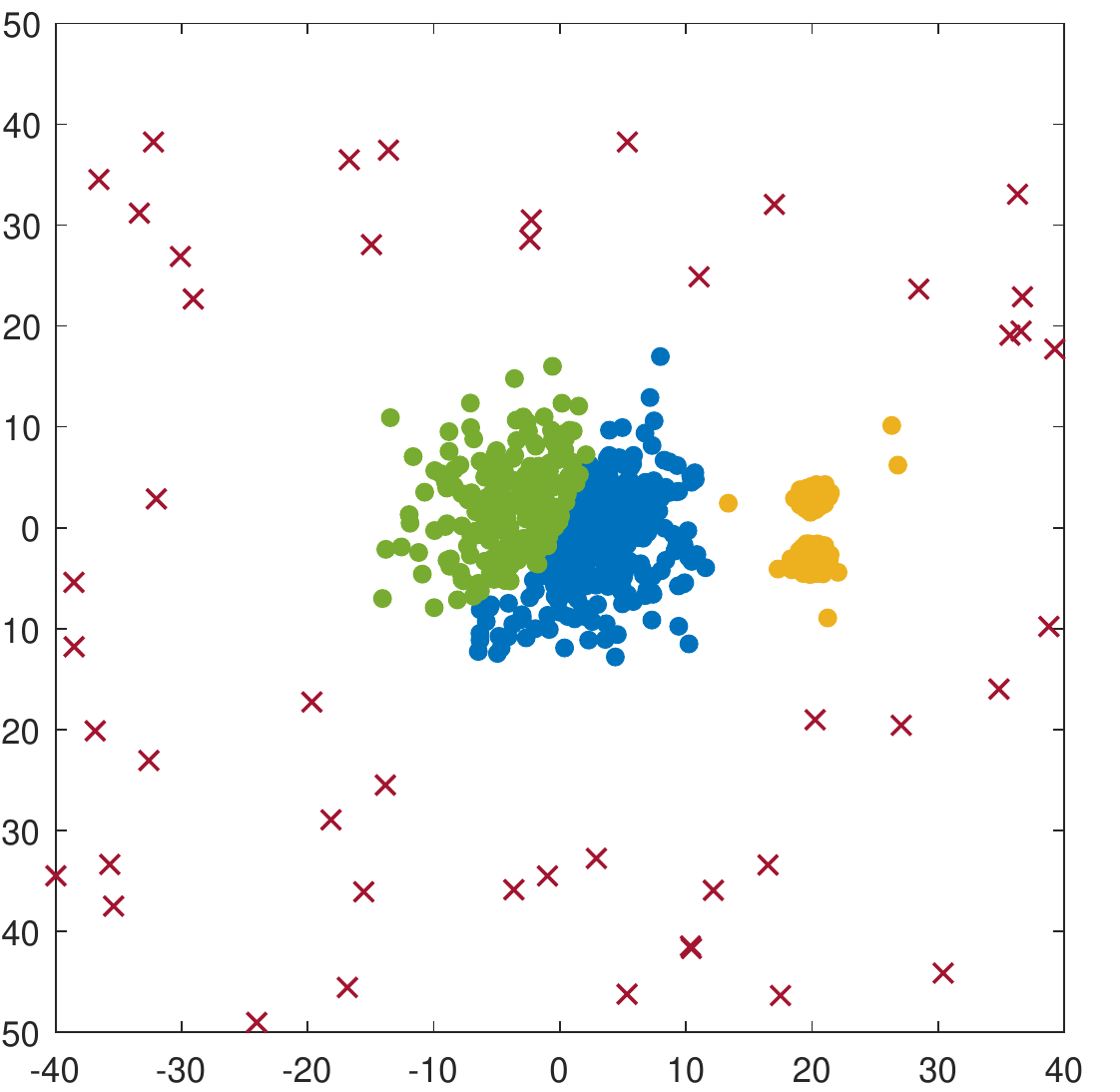}
    \caption*{Robust-Kmeans}
    \end{minipage}
    \begin{minipage}[t]{0.18\textwidth}
    \includegraphics[width=\linewidth]{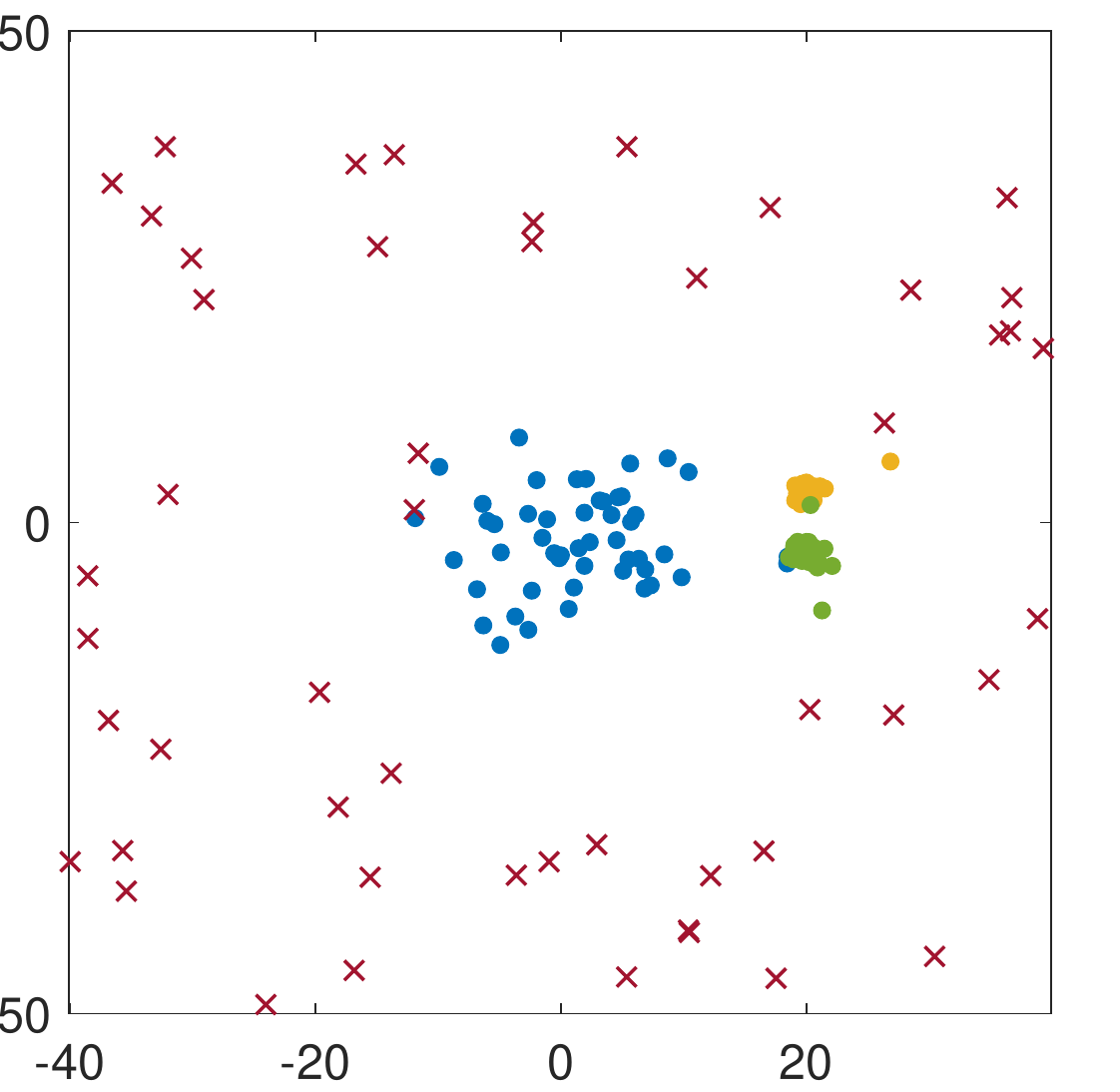}
    \caption*{CC-Kmeans on down-sampled dataset}
    \end{minipage}
    \caption{Unbalanced Spherical GMMs}
    \label{subfig:UnbalancedSphericalGMMs}
    \end{subfigure}
    
    \begin{subfigure}[t]{\textwidth}
    \centering
    \begin{minipage}[t]{0.18\textwidth}
    \centering
    \includegraphics[width=\linewidth,height=3cm]{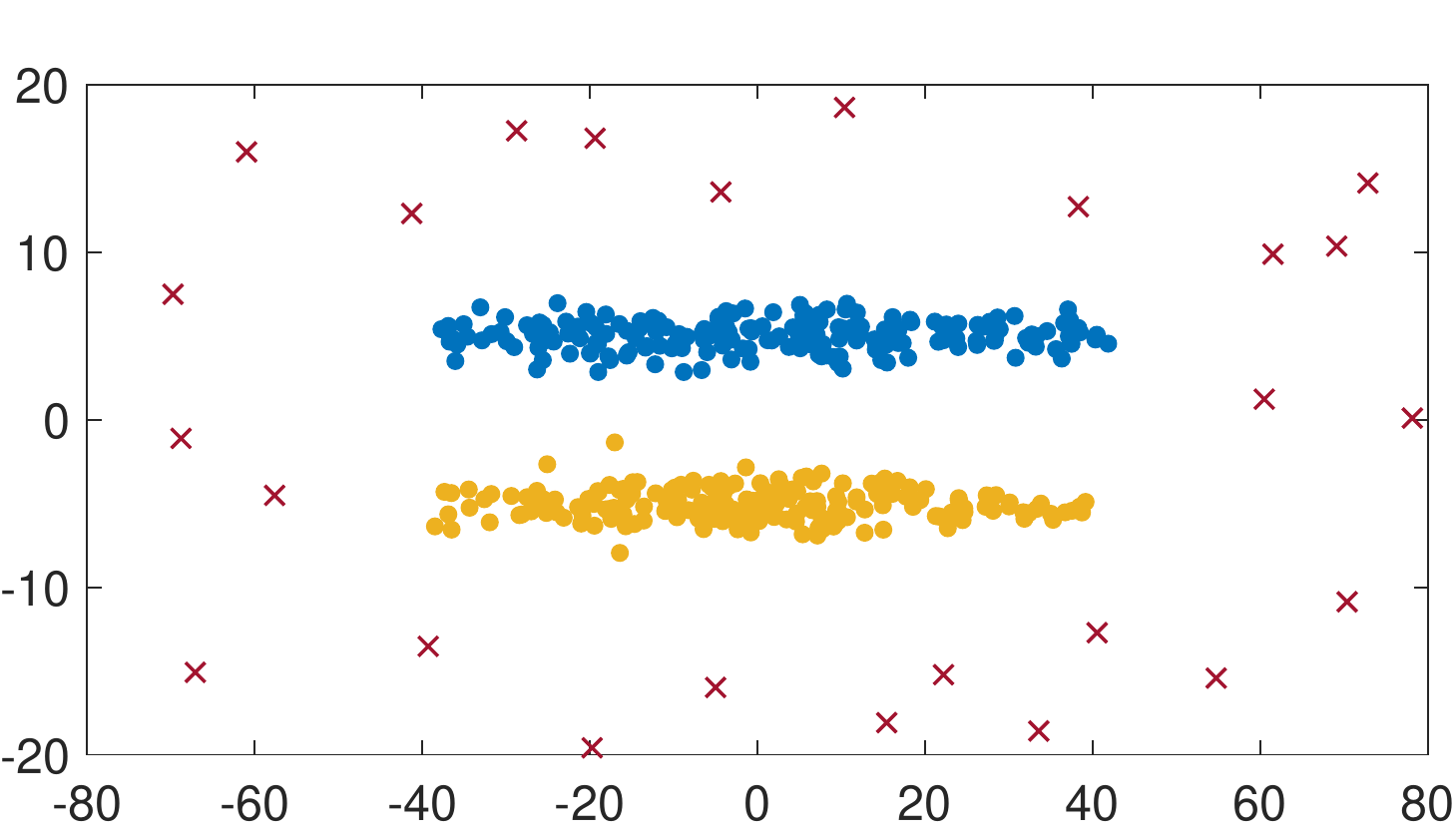}
    \caption*{Robust-SC}
    \end{minipage}
    \begin{minipage}[t]{0.18\textwidth}
    \includegraphics[width=\linewidth,height=3cm]{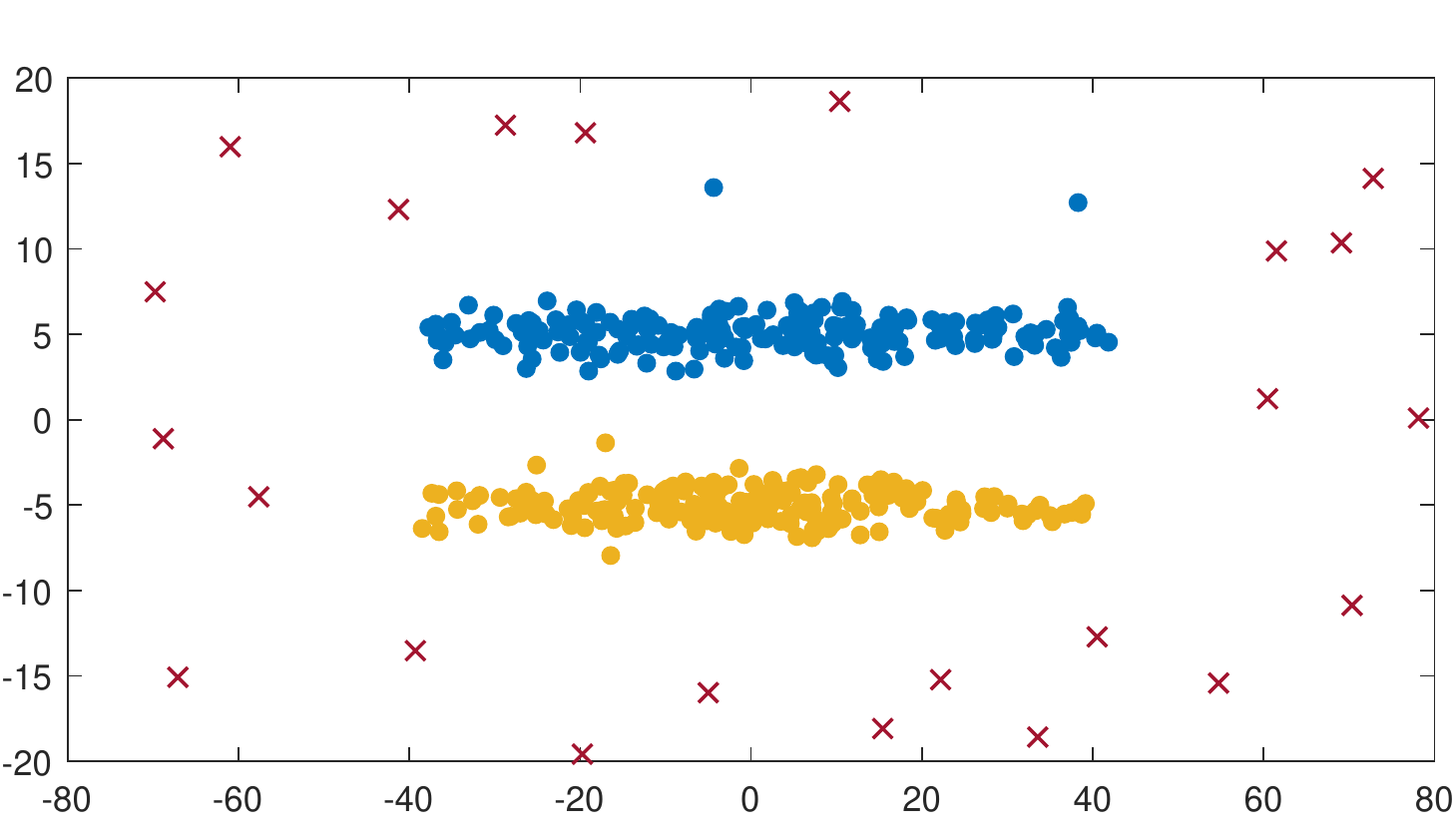}
    \caption*{Robust-SDP}
    \end{minipage}
    \begin{minipage}[t]{0.18\textwidth}
    \includegraphics[width=\linewidth,height=3cm]{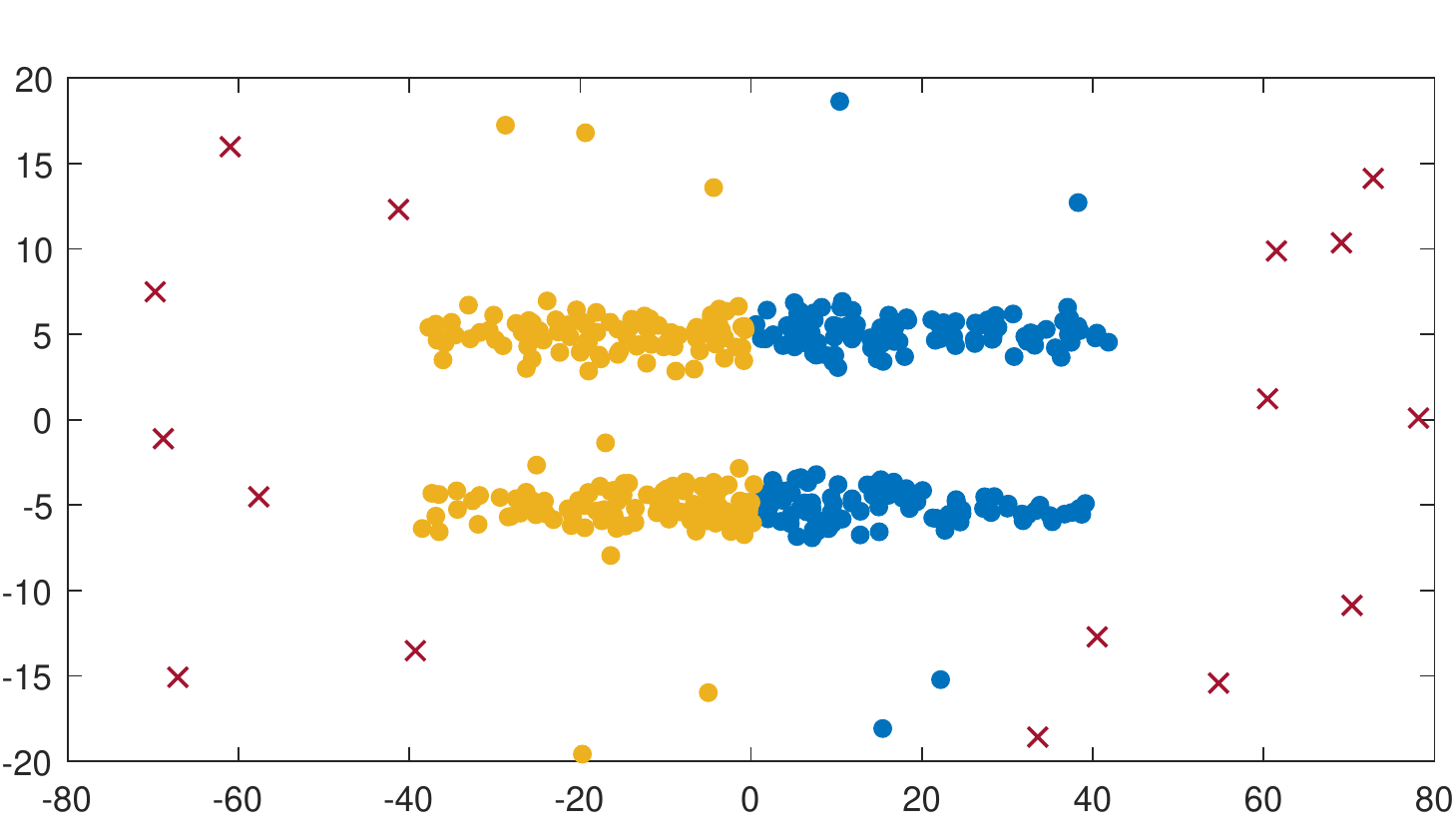}
    \caption*{Robust-Kmeans}
    \end{minipage}
    \begin{minipage}[t]{0.18\textwidth}
    \includegraphics[width=\linewidth,height=3cm]{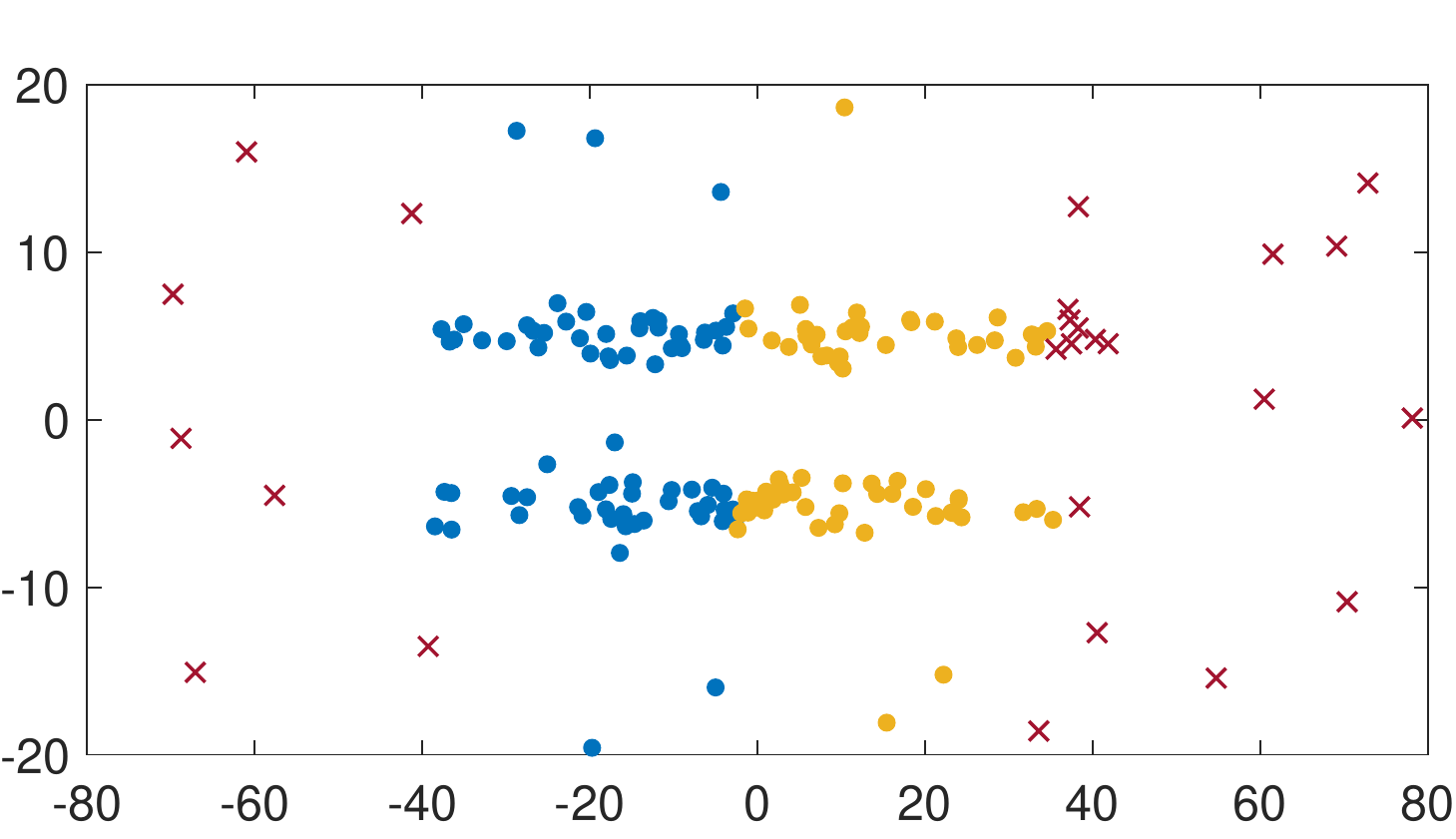}
    \caption*{CC-Kmeans on down-sampled dataset}
    \end{minipage}
    \caption{Balanced Ellipsoidal GMMs}
    \label{subfig:BalancedEllipsoidalGMMs}
    \end{subfigure}
    
    \caption{Clustering results for different algorithms on synthetic datasets. CC-Kmeans could not be implemented on the entire dataset due to memory and computational limitations. Therefore, for comparison, we show the clustering results for a down-sampled dataset with equal number of points from each cluster.}
    \label{fig:Results}
\end{figure}

\subsubsection{Comparison with $k$-means++ and spectral clustering algorithms:} 
\label{sec-LargeScale}
From the solution times reported in~\Cref{table-ImplementationTimes}, it is quite evident that the SDP-based algorithms are intractable for large scale experiments. Therefore, in this section, we consider a much larger experiment setting and compare Robust-SC with more scalable $k$-means++ and spectral clustering algorithms. 

In the experimental setup for this section, we assume that the $n$~inlier points  are generated in $r$-dimensional space from $r$ equal-sized spherical Gaussians, which are centered  at the vertices of a suitably scaled standard $(r-1)$-dimensional simplex and have identity covariance matrices. Thus, for all clusters $k\in[r]$, $\bmu_k=s \cdot \bm e_k$, for some scale parameter $s$ and $\bSigma_k=\bm I_r$. The $m$~outlier points are generated from another spherical Gaussian centered at the origin, i.e., $\bmu_{\mathc O}=\bm 0$, and having a much larger variance $(\bSigma_{\mathcal{O}}= 100\cdot {\bm I}_r)$.

\begin{figure}
    \centering
    \begin{subfigure}[t]{0.32\textwidth}
        \centering
        \includegraphics[width=\linewidth]{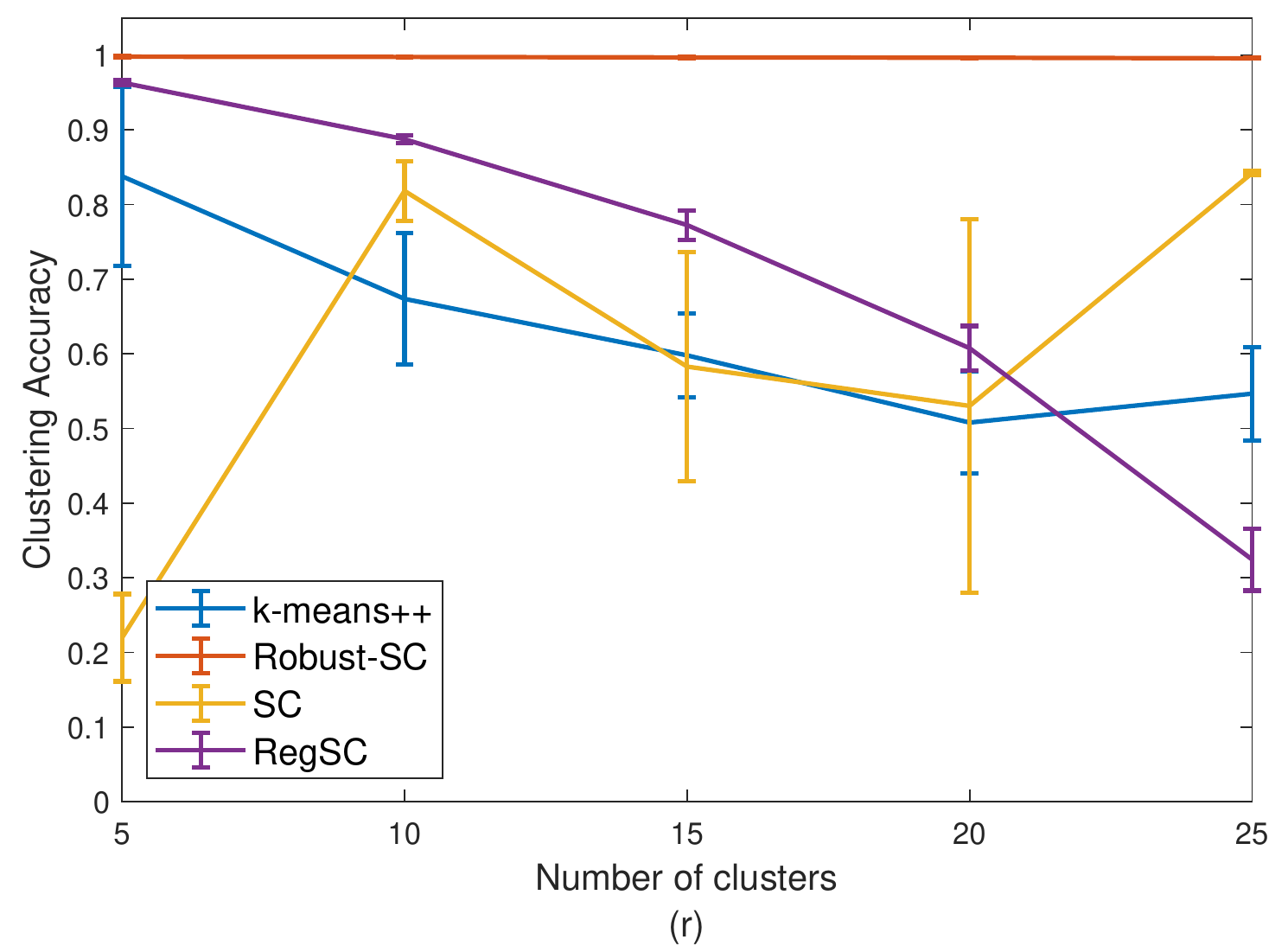}
        \caption{}
    \end{subfigure}
    \begin{subfigure}[t]{0.32\textwidth}
        \centering
        \includegraphics[width=\linewidth]{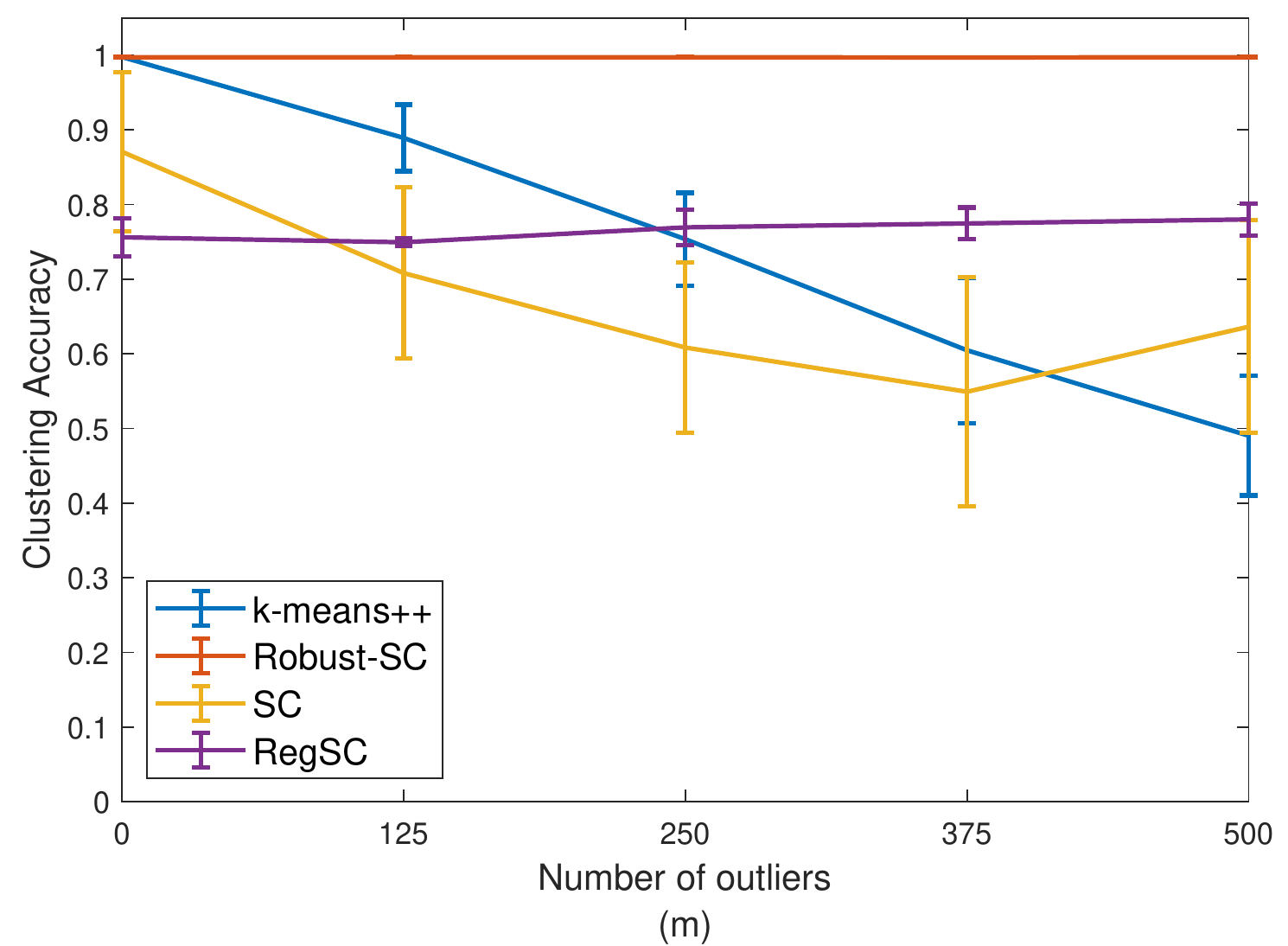} 
        \caption{} 
    \end{subfigure}
        \begin{subfigure}[t]{0.32\textwidth}
        \centering
        \includegraphics[width=\linewidth]{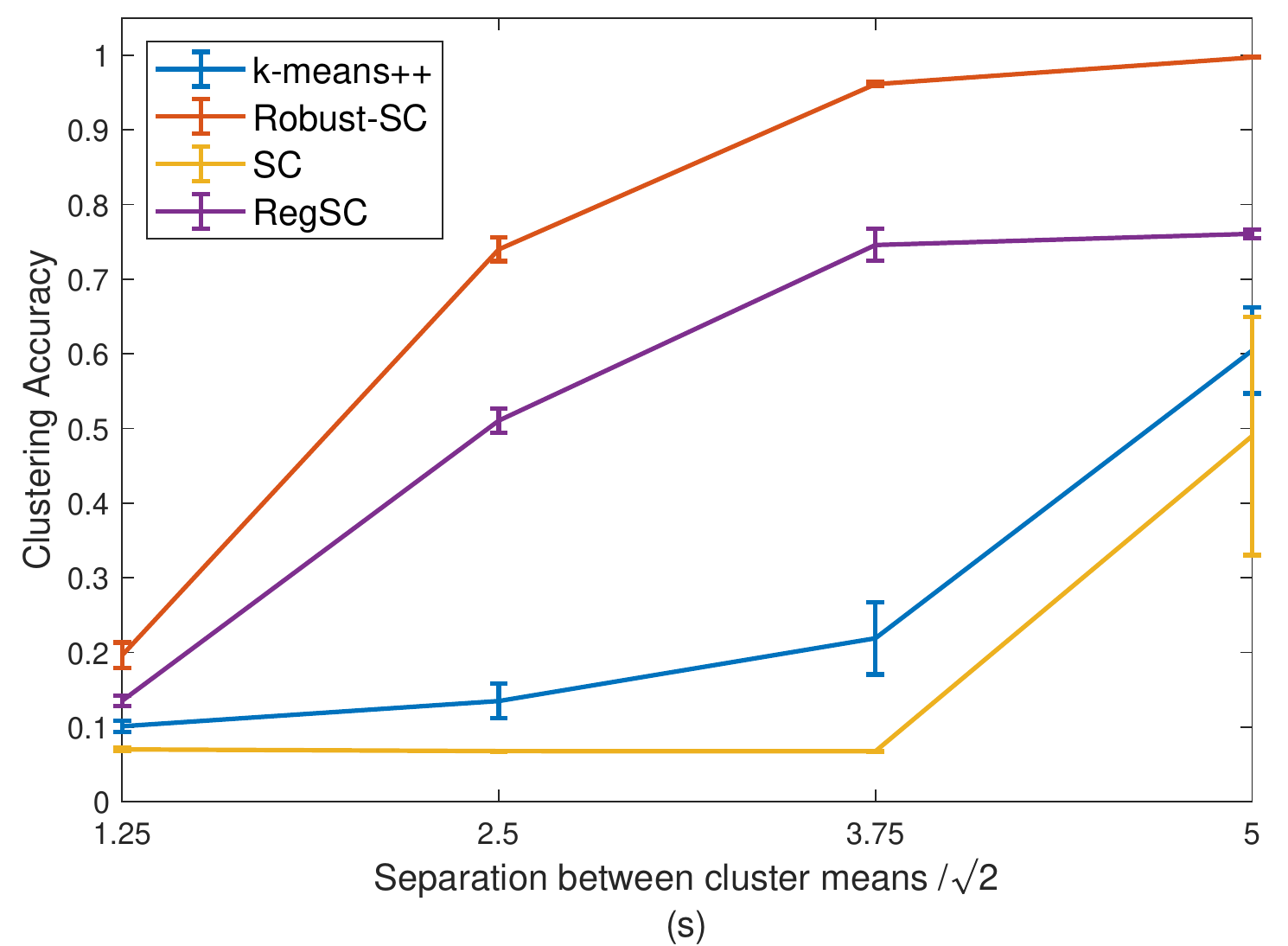} 
        \caption{} 
    \end{subfigure}
\caption{ Figure shows the effects of varying the model parameters on the inlier clustering accuracy for different algorithms. The default parameter values are set to $r=15,s=5,m=400$ and $n/r=400$. In each plot, apart from the parameter that is being varied, the other parameters are set to their default values.}
\label{fig:Test_large_scale}
\end{figure}

We analyze the robustness of the Robust-SC algorithm under different model settings by varying the number of clusters $(r)$, the number of outliers points $(m)$, and the separation between cluster centers $(\Delta := \sqrt{2}s)$. We compare Robust-SC with $k$-means++ and popular variants of the spectral clustering using clustering accuracy for inlier points as the evaluation metric. Figure \ref{fig:Test_large_scale} shows the results obtained. For this set of experiments, we assume that the default parameter values are set to $r=15$, $s=5$, $m=400$ and $n/r=400$. In each experiment, we assume that except for the parameter that is varied, the other parameters are set to their default values. From the plots, we note that Robust-SC clearly outperforms the other clustering algorithms in terms of performance. We further demonstrate the scalability of the Robust-SC algorithm by repeating the experiment for $r=50$ equal-sized clusters with $n=50,000$ inlier points and $m=1000$ outlier points. For 10~simulation runs of this experiment, we achieve an average inlier clustering accuracy of $0.9926$ and an average solution time of $525.34$s with standard deviation values of $5.44\times 10^{-4}$ and $17.8$s respectively.

\subsection{Real world datasets}
For evaluating the performance of different algorithms on real world datasets, we standardize the dataset by applying a $z$-score transformation to each attribute of the dataset. For high dimensional datasets, we adopt the dimensionality reduction procedure described in \Cref{sec-results}, which involves first computing the covariance matrix $\bSigma$, projecting the data points on to the subspace spanned by the $\R-1$ principal eigenvectors of $\bSigma$, and then applying the $z$-score transformation to each attribute in the reduced space. All of these datasets were obtained from the UCI Machine Learning repository \citep{Dua:2019}. We provide below a brief description of these datasets and summarize their main characteristics in \Cref{RealDatasets}.
\begin{itemize}
    \item \textbf{MNIST dataset:} Handwritten digits dataset comprising of 1000 samples of  $8\times 8$ grayscale images (represented as a $64$-dimensional vector) of digits from 0 - 9. 
    
    \item \textbf{Iris dataset:}  Dataset consists of a total of 150 samples from 3 clusters, each representing a particular type of Iris plant. The four attributes associated with each data instance represent the sepal and petal lengths and widths of each flower in centimeters.   
    \item \textbf{USPS dataset:} A subset of the original USPS dataset consisting of 500 random samples, each representing a $16 \times 16$ greyscale image of one of the following four digits - 0, 1, 3, and 7. 
    \item \textbf{Breast cancer dataset:} Dataset consists of 683 samples of benign and malignant cancer cases. Every data instance is described by 9 attributes, each having ten integer-valued discrete levels.
\end{itemize}

\begin{table}[h]
    \centering
    {\begin{tabular}{l c  c c c} \hline
     \textbf{Dataset} & $N$ - \# of datapoints & $\D$ - \# of dimensions & $\R$ - \# of clusters\\ 
        \hline
         MNIST  
         & 1000 & 64 & 10 \\
         Iris   
         & 150 & 4 & 3\\ 
         USPS
         & 500 & 256 & 4 \\
         Breast Cancer 
         & 683 & 9 & 2 \\
         \hline
    \end{tabular}}
    \caption
    {Real-world datasets with their main characteristics.}\label{RealDatasets}
\end{table}

\begin{table}[b]
\centering
\begin{tabular}{ l c  c c c c} \hline
 \textbf{Algorithm}&\textbf{MNIST} & \textbf{Iris} &\textbf{USPS}&\textbf{Breast Cancer} \\ 
    \hline
    Robust-SDP
    & 0.8450 & \textbf{0.8933} & \textbf{0.9720} & 0.9649  \\
    Robust-SC
    & \textbf{0.8630} & 0.8800 & 0.9620 & \textbf{0.9722}\\
    Robust-Kmeans 
    & 0.8040 & 0.8267 & 0.8320 & 0.9575 \\
    Robust-Kmeans-NoDR
    & 0.6680 & 0.8267 & 0.6420 & 0.9575 \\ 
    CC-Kmeans 
    & - & 0.8400 & - & - \\ 
    SC
    & 0.8580 & 0.6600 & 0.3280 & 0.6471\\
    RegSC
    & 0.7320 & 0.5200 & 0.6000 & 0.8873 \\
    $k$-means++ & 0.7850 & 0.8133 & 0.6080 & 0.9575 \\
    \hline
    \end{tabular}
    \caption
{Performance of different clustering algorithms on real-world datasets.\label{Table-RealWorld}. This table reports the performance of different clustering algorithms on real-world datasets in terms of their overall clustering accuracy. Entry with `-' indicates that the algorithm failed to terminate within the specified time limit of 2 hours.}
\end{table}

  
\begin{figure}[h]
    \centering
    \begin{subfigure}[t]{0.45\textwidth}
        \centering
        \includegraphics[width=\linewidth]{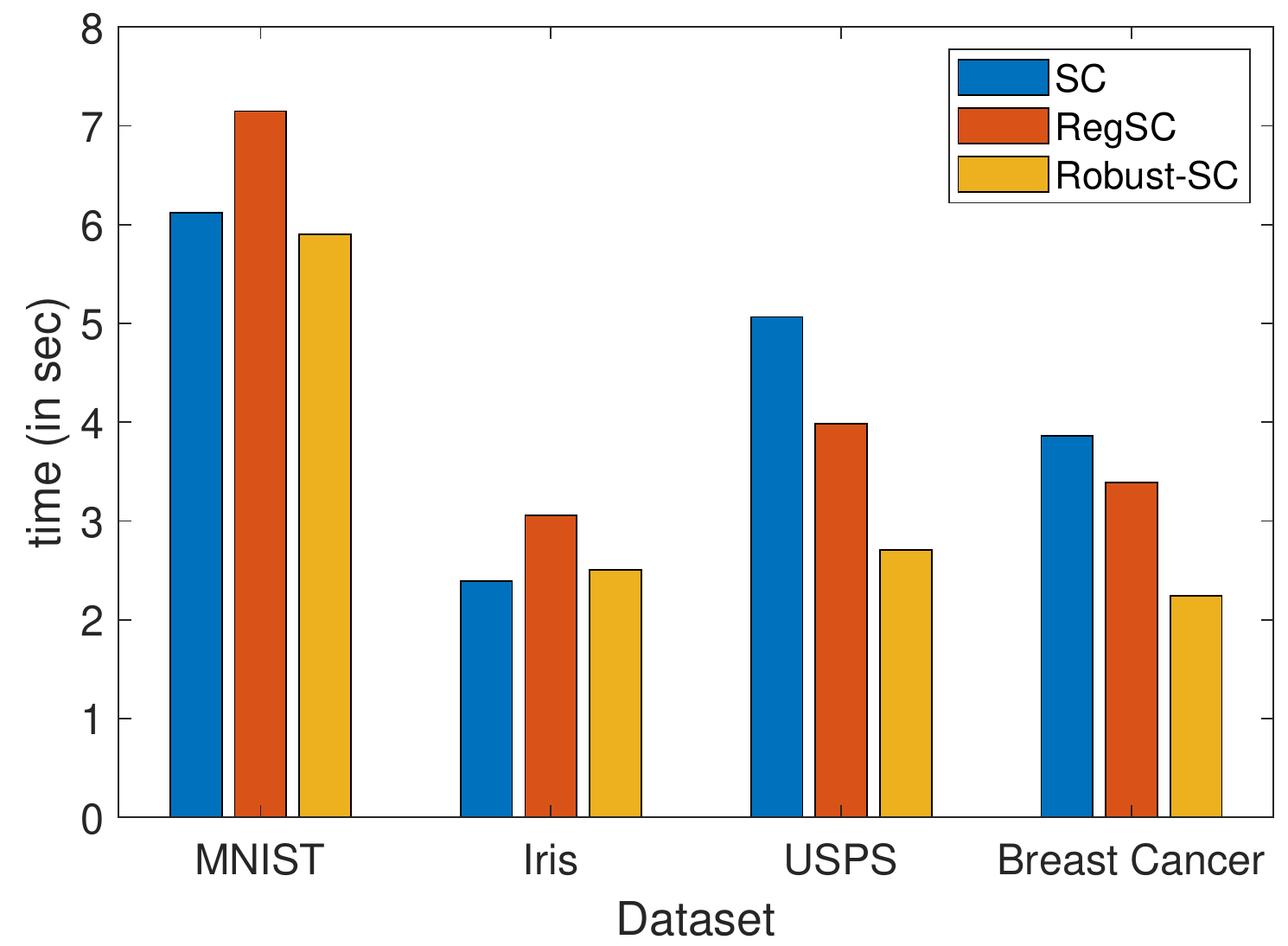}
        \caption{Spectral methods}
    \end{subfigure}
    \begin{subfigure}[t]{0.45\textwidth}
        \centering
        \includegraphics[width=1.03\linewidth]{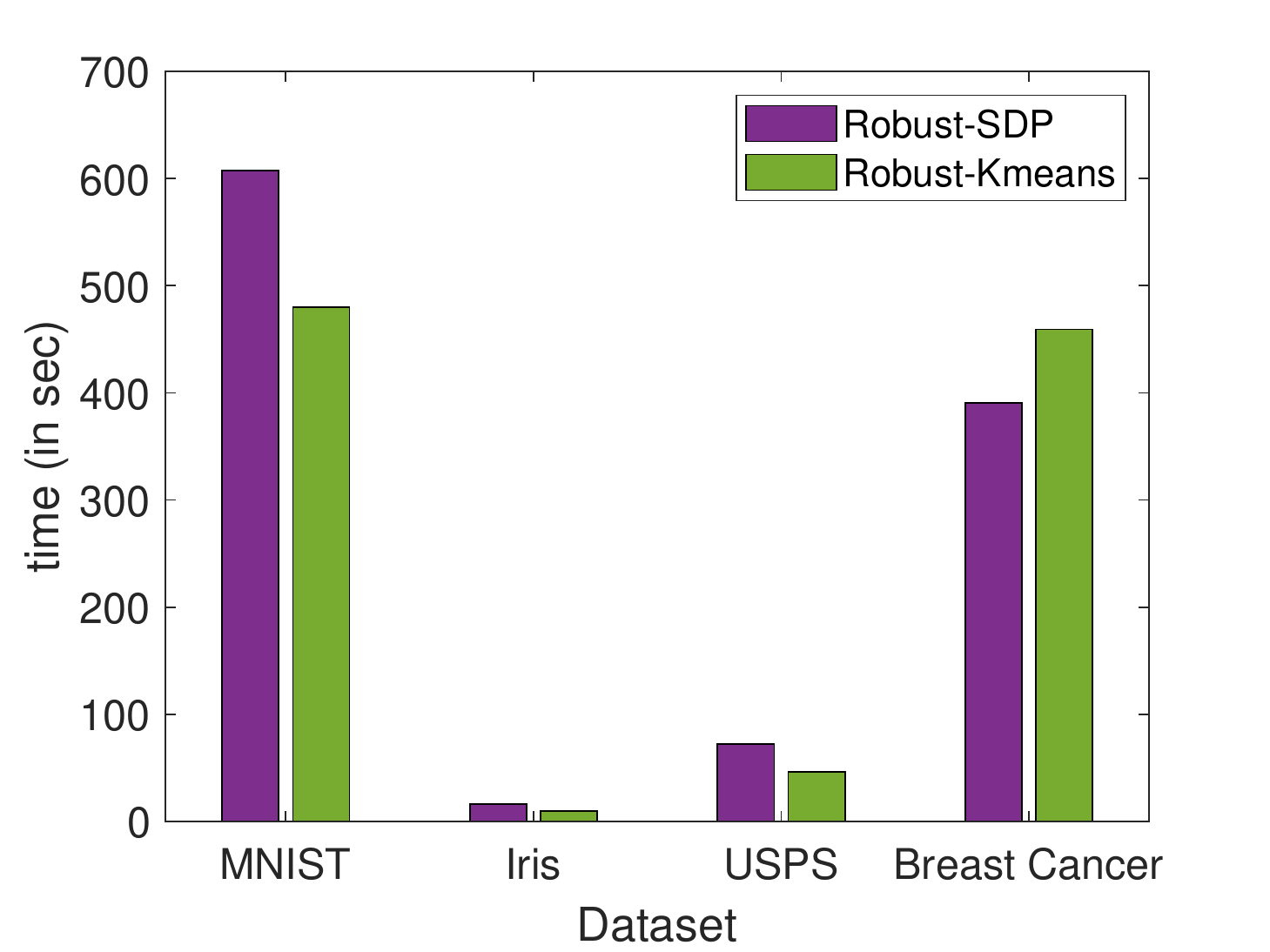} 
        \caption{SDP-based methods} 
    \end{subfigure}
\caption{Solution times (in seconds) for  different algorithms on real-world datasets.}
\label{fig:Times}
\end{figure}
For these real-world datasets, in addition to Robust-Kmeans and CC-Kmeans, we also compare the performances of Robust-SC and Robust-SDP with three other algorithms, namely $k$-means++, vanilla spectral clustering (SC), and regularized spectral clustering (RegSC). 

As we previously discussed, for high-dimensional datasets, some form of a dimensionality reduction procedure is usually needed as an important pre-processing step. In the real-world datasets that we consider in our study, two datasets, namely MNIST and USPS have high-dimensional features. Although none of the other methods that we compare our algorithm against explicitly recommends or analyzes the dimensionality reduction step for high-dimensional setting, for fairness, we apply our proposed dimensionality reduction procedure in \Cref{sec-results} to all the algorithms. For reference, however, we consider a variant of the Robust-Kmeans algorithm, Robust-Kmeans-NoDR, that does not use our proposed dimensionality reduction procedure, but is applied to the actual data in the original high-dimensional space. 

\Cref{Table-RealWorld} summarizes the clustering performance of different algorithms on the real-world datasets in terms of their overall accuracy for each dataset. Based on the values in the table, we infer that both Robust-SC and Robust-SDP consistently perform well across all datasets, and considerably better compared to the other algorithms considered in the study. Additionally, as we previously observed from our simulation studies, the Robust-SC algorithm recovers solutions that are almost as good as the Robust-SDP solutions, and for some datasets (MNIST and Breast Cancer), marginally better in terms of the clustering accuracy, even though Robust-SC is based on a simple rounding scheme, while the Robust-SDP algorithm requires solving the \ref{RelaxedSDPForm} formulation. For this reason, there is a significant disparity in the solution times noted for the two algorithms (refer to \Cref{fig:Times}), with the Robust-SC algorithm being approximately 100 times faster even for moderately-sized problem instances. Additionally, comparing the performance of Robust-Kmeans and Robust-Kmeans-NoDR on the high dimensional datasets - MNIST and USPS, we can easily see that the dimensionality reduction step significantly improves the performance of the algorithm on high-dimensional real-world datasets. 

\subsection{Estimating unknown number of clusters from Robust-SDP formulation}
\label{sec-clusters_num}
In several real-world problems, the number of clusters $\R$ is unknown. In this section, we discuss how we can obtain an estimate $\htt r$ for the number of clusters from the \ref{RelaxedSDPForm} solution~$\htt \X^{\text{SDP}}$. In general, the SDP solution provides a more denoised representation of the kernel matrix as compared to the simple rounding scheme based on the Robust-LP solution. We propose a procedure based on the eigengap heuristic~\citep{von2007tutorial} of the normalized graph Laplacian matrix~$\bm L_{\tilde {\mathc I}}:=~\bm I~-~\bm D_{\tilde {\mathc I}}^{-1/2}\htt \X^{\text{SDP}}_{\tilde {\mathc I}}\bm D_{\tilde {\mathc I}}^{-1/2}$ where $\bm D_{\tilde {\mathc I}}=\Diag(\htt \X^{\text{SDP}}_{\tilde {\mathc I}} \bs 1_{\lvert  \tilde {\mathc I}\rvert} )$ and $\tilde{\mathc I}=\{i:\deg(i)\geq \tilde{\tau} \}$. Here, the threshold $\tilde{\tau}$ corresponds to some quantile $\tilde{\beta}$ of $\{ \deg(i), i=1,\ldots,\N\}$. The key idea behind this heuristic is to select a value of $\htt \R$ such that the $\htt \R$ smallest eigenvalues $\lambda_1 \leq \ldots \leq \lambda_{\htt \R}$ of $\bm L_{\tilde {\mathc I}}$ are extremely small (close to $0$) while $\lambda_{\htt \R +1} $ is relatively large. The main argument for using the eigengap heuristic comes from matrix perturbation theory, which leverages the fact that if a graph consists of $\R$ disjoint clusters, then its graph Laplacian matrix has an eigenvalue of $0$ with multiplicity $\R$ and its ($\R$+1)-st smallest eigenvalue $\lambda_{\R+1}$ is comparatively larger. 

\begin{figure}[htb!]
    \centering
    \begin{subfigure}[t]{0.45\textwidth}
        \centering
        \includegraphics[width=\linewidth]{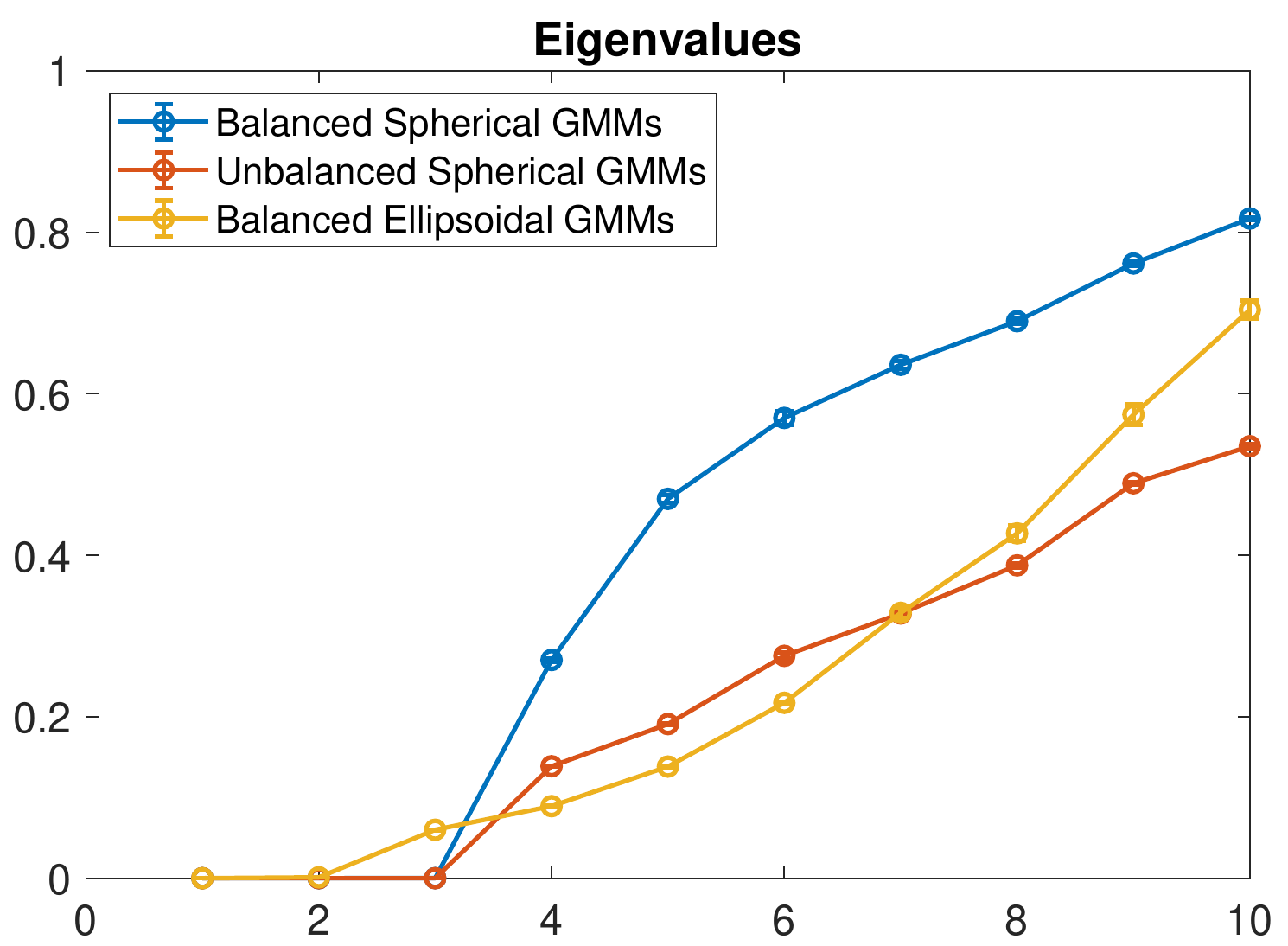}
        \caption{Synthetic Datasets}
    \end{subfigure}
    \begin{subfigure}[t]{0.45\textwidth}
        \centering
        \includegraphics[width=\linewidth]{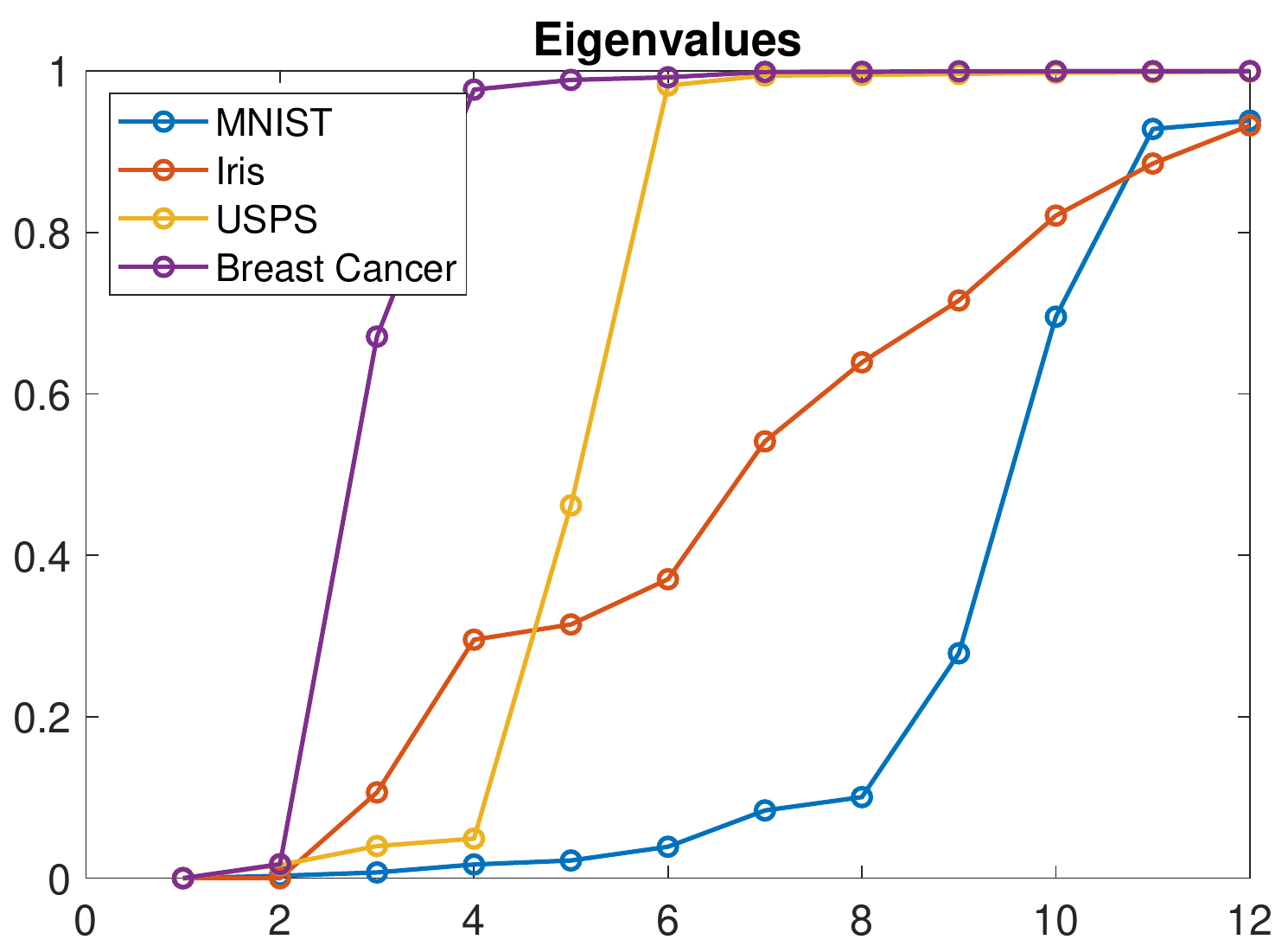} 
        \caption{Real-world datasets} 
    \end{subfigure}
\caption{Eigenvalues of the normalized graph Laplacian matrix $\bm L_{\tilde {\mathc I}}:= \bm I - \bm D_{\tilde {\mathc I}}^{-1/2}\htt \X^{\text{SDP}}_{\tilde {\mathc I}}\bm D_{\tilde {\mathc I}}^{-1/2}$  for synthetic and real-world datasets with $\tilde{\beta}=0.8$.}
\label{fig:Eigenvalues}
\end{figure}

\Cref{fig:Eigenvalues} denotes the eigenvalues of the normalized graph Laplacian matrix for both synthetic and real-world datasets.
From the plot, it is easy to see that the eigengap heuristic correctly predicts the number of clusters for each of the three synthetic datasets. It is important to note the eigengap heuristic for finding the number of clusters usually works better when the signal-to-noise ratio is large, i.e., either when the clusters are well-separated, or when the noise around the clusters is small. However, for many real world datasets, a high signal-to-noise ratio is not always observed. For example, in the MNIST handwritten digits dataset, there are considerable overlaps between clusters that represent digits 1 and 7 as well as digits 4 and 9. Thus, when the eigengap heuristic is applied on the MNIST dataset, it returns $\htt r=8$ as an estimate for the number of clusters. Similarly, for the iris dataset, two of the clusters (Verginica and Versicolor) are known to intersect each other \citep{ana2003robust}. Thus, when the number of clusters is not specified, we get $\htt \R=2$ instead of the actual three clusters in the dataset. 

While it is possible to obtain an estimate of $\R$ by applying the above procedure on the rounded matrix $\hat{\X}$ obtained from the \ref{RelaxedForm} formulation, we see that $\hat{\R}$ obtained from $\htt \X^{\text{SDP}}$ is more accurate.

\subsection*{Acknowledgments}
Grani A.~Hanasusanto is supported by the National Science Foundation grant no.~1752125. Purnamrita Sarkar is supported in part by the National Science Foundation grant no.~1713082.
\end{onehalfspace}

\clearpage
\bibliographystyle{plainnat}
\bibliography{references}
\appendix
\begin{onehalfspace}
\newpage
\numberwithin{equation}{section}
\begin{center}
    \section*{A Robust Spectral Clustering Algorithm for Sub-Gaussian \\ Mixture Models with Outliers\\(Supplementary Material) }
\end{center}

\section{Background on Sub-Gaussian Random Variables and Vectors}
\begin{definition}[Sub-gaussian Random Variable] 
A random variable $X$ with mean $\mu$ is defined to be sub-gaussian if there exists a constant $\psi>0$ such that the following condition holds:
\begin{equation*}
    \mb E[e^{\bm \lambda(X-\mu)}]\leq e^{\psi^2\lambda^2/2}, \quad \forall \lambda \in \mb R.
\end{equation*}
Here, $\psi$ is also called the sub-gaussian parameter.
\end{definition}
\begin{definition}[Sub-gaussian Random Vector] 
A random vector $\bm X \in \mb R^\D$ with mean $\bmu \in \mb R^\D$ is defined to be sub-gaussian if there exists a constant $\psi>0$ such that the following condition holds:
\begin{equation*}
    \mb E[e^{\bs \nu^\top (\bm X-\bmu)}]\leq e^{\psi^2\lVert v \rVert^2/2}, \quad \forall \bs \nu \in \mb R^\D.
\end{equation*}
Here, $\psi$ is also called the sub-gaussian parameter. 
\end{definition}
\noindent  For additional background on sub-gaussian random variables and sub-gaussian random vectors, we refer the reader to \cite{hsu2012tail,wainwright2019high,vershynin2010introduction}.

\section{Proof of Proposition 1}
\begin{proof}
$(\implies)$ For every feasible $\Z$ in $\eqref{eq:outlier-ExactForm-Z}$, we can construct a solution $\X=\Z \Z^\top$. By definition, $\X\succeq 0$ and satisfies the constraint that $\rank(\X)\leq r$ since $\rank(\bm  Z)\leq r$. In addition, since $\Z$ is a binary 0-1 assignment matrix whose each row $\z_i^\top \in \{0,1\}^r$ sums to either $0$ or $1$, we get that $X_{ij}=\z_i^\top \z_j \in \{0,1\}$.  Thus, $\X=\Z \Z^\top$ is feasible for \eqref{eq:outlier-ExactForm-X} and has the same objective function as \eqref{eq:outlier-ExactForm-Z}.

$(\impliedby)$ To prove the converse, we first assume that $\bm  X$ is a feasible solution for \eqref{eq:outlier-ExactForm-X}. Thus, it satisfies the constraints $\bm  X \succeq 0$ and $\text{rank}(\X)=l\leq r$.  These two constraints together imply that $\bm  X$ can be expressed as $\X=\bm G \bm G^\top$ where $\bm G\in\mb R^{\N\times \R}$ is a matrix with $\rank(\bm G)= l$. Next, since $\bm  X$ is a binary 0-1 matrix, we get that $X_{ii} = \bm g_i^\top \bm g_i =\lVert \bm g_i\rVert^2=\lVert \bm g_i\rVert$ equals either $0$ or $1$. This, in turn, implies that each row $\bm g_i^\top$ of $\bm G$ is either a zero vector or a unit vector depending on whether the point is classified as an outlier or an inlier. 

Next, we show that there exists an orthogonal matrix $\bm  O\in \mb R^{\R\times \R}$ such that $\Z=\bm G \bm  O$, where~$\bm  Z\in\{0,1\}^{\N\times r}$ is an assignment matrix with $\rank(\bm  Z)=l\leq r$ whose each row sums to either 0 or 1. Since $\rank(\bm G)=l$, there must exist a set of $l$ (non-zero and distinct) linearly independent row vectors in $\bm G$ that span the row-space of $\bm G$. Assume that $\mathc B=\{ \bm  u_1,\ldots,\bm u_l\}$ represents one such set of $l$ row vectors from $\Z$. We now show that the set $\mathc B$ forms an orthonormal basis for the row-space of $\bm G$. Since each $\bm u_{l'}\in \mathc B$ corresponds to some non-zero row of $\bm G$, it immediately follows that  $\lVert \bm  u_{l'} \rVert=1$ for all $l'\in[l]$ in $\mathc B$. For an angle $\zeta_{l_1 l_2}$ between a pair of distinct basis vectors $\bm  u_{l_1}$ and $\bm  u_{l_2}$, we have that ${\bm u}_{l_1}^\top \bm  u_{l_2}=\lVert \bm  u_{l_1} \rVert \lVert \bm  u_{l_2} \rVert \cos\zeta_{l_1 l_2}= \cos\zeta_{l_1 l_2}$ which must be either $0$ or $1$ since $\bm  X$ is a binary 0-1 matrix. As $\bm  u_{l_1}\neq \bm u_{l_2}$, this implies that $\bm  u_{l_1}^\top \bm  u_{l_2}=0$. Thus, $\mathc B$ forms an orthonormal basis. Next, we show that for all $i\in [N]$, the row vector $\bm g_i$ of $\bm G$ must be one of the basis vectors in the set $\mathc B$. We prove this assertion by contradiction. Suppose that there exists a non-zero row vector $\bm g_i$ in the row-space of $\bm G$ such that $\bm g_i\neq \bm u_{l'}$ for all $l' \in [l]$. Then assuming an angle $\zeta_{il'}$ between $\bm g_i$ and $\bm  u_{l'}$, we get that $\bm g_i^\top \bm  u_{l'}=\lVert \bm g_i \rVert \lVert \bm  u_{l'} \rVert \cos\zeta_{il'}=\cos\zeta_{il'}\notin \{0,1\}$ which is a contradiction since $\X$ is a binary 0-1 matrix. Thus, $\bm G$ can be multiplied by an orthogonal matrix $\bm  O\in \mb R^{\R\times \R}$ so that each of its rows correspond to either a $r$-dimensional standard basis vector or a $r$-dimensional zero vector in the assignment matrix $\Z$. Therefore, for every feasible solution $\bm  X$ in \eqref{eq:outlier-ExactForm-X}, there exists a corresponding solution $\bm  Z=\bm G \bm O$ which is feasible in \eqref{eq:outlier-ExactForm-Z} and has the same objective function value as \eqref{eq:outlier-ExactForm-X}.
\end{proof}

\section{Proof of Lemma \ref{lem-EstimatedApproximatesTrue}}
\begin{proof}
Let $\hat{\X}=\arg\max\limits_{\X \in \mathcal{X}}\langle \bm K - \gamma \bm E_N, \X\rangle$. From Lemma \ref{lem-TrueReferenceMaximizer}, we have that $\X^0=\arg \max\limits_{\X \in \mathcal{X}}\langle\bm R - \gamma \bm E_N, \X\rangle$.  In addition, from the strong assortativity condition, we have  $$ \displaystyle    R^{\tin}_{\min}=\min_{i,j \in  C_k: k\in [\R] } R_{ij} > \max_{ i\in  C_k, j\in  \mathc{C}_l:k,l \in [\R]} R_{ij} = R^{\tout}_{\max}.$$
Therefore, for any $R^{\tin}_{\min}<\gamma<R^{\tout}_{\max}$, we get
\begin{equation}
\begin{aligned}
\langle \bm R - \gamma \bm E_N, \X^0-\hat{\X}\rangle
&=\sum_{k\in [r]}\sum_{i,j\in C_k} (R_{ij}-\gamma) (1-\hat{X}_{ij}) - \sum_{k\neq l} \sum_{i\in  C_k, j\in \mathc{C}_l} (\gamma-R_{ij})(-\hat{X}_{ij}) \\
&\geq (R^{\tin}_{\min}-\gamma) \sum_{k \in [\R]} \sum_{i,j\in \mathc{C}_k} (1-\hat{X}_{ij}) + (\gamma-R^{\tout}_{\max}) \sum_{k\neq l} \sum_{i\in \mathc{C}_k, j \in  \mathc{C}_l}  (\hat{X}_{ij}-0)\\
&\geq \min(R^{\tin}_{\min}-\gamma, \gamma- R^{\tout}_{\max}) \lVert \X^0_{\mathc I}-\hat{\X}_{\mathc I}\rVert_1
\end{aligned}
\end{equation}
From above, we get that the estimation error can be bounded as below:
\begin{equation}
\label{eq:X_diff_bound}
\lVert \X^0_{\mathc I} - \hat{\X}_{\mathc I}\rVert_1 \leq \frac{\langle \bm R - \gamma \bm E_N, \X^0-\hat{\X}\rangle}{\min( R^{\tin}_{\min}-\gamma,\gamma- R^{\tout}_{\max})}\leq\frac{\langle \bm R - \gamma \bm E_N, \X^0-\hat{\X}\rangle}{\min\{\upsilon,1-\upsilon\}(\tau_{\tin}-\tau_{\tout})}.
\end{equation}
Here, the last inequality is obtained by setting $\gamma= \upsilon \tau_{\tin}+(1-\upsilon)\tau_{\tout} \in \big(R^{\tout}_{\max},R^{\tin}_{\min} \big)$, and noting that $\tau_{\tin}\leq R^{\tin}_{\min}$ and $\tau_{\tout}\geq R^{\tout}_{\max}$.
\end{proof}


\section{Proof of Lemma \ref{lem-EstimatedAlmostReferenceMaximizer}}
\begin{proof}
\begin{equation}
\label{eq:ref_inner_prod_bound}
\begin{aligned}
&\langle \bm R - \gamma \bm E_N, \X^0 -\hat{\X}\rangle \\
&=\langle \bm R_{\mathc I} - \gamma \bm E_\smalln, \X^0_{\mathc I} -\hat{\X}_{\mathc I}\rangle \\
&=\langle \bm R_{\mathc I} - \gamma \bm E_\smalln,\X^0_{\mathc I} \rangle -\langle \bm R_{\mathc I} - \gamma \bm E_\smalln,\hat{\X}_{\mathc I}\rangle \\ 
&=\langle \bm R_{\mathc I} - \gamma \bm E_\smalln,\X_{\mathc I}^0\rangle - \langle \bm K_{\mathc I} - \gamma \bm E_\smalln,\htt \X_{\mathc I} \rangle + \langle \bm K_{\mathc I} - \bm R_{\mathc I}, \hat{\X}_{\mathc I}\rangle \\
&\stackrel{(i)}{\leq} \langle \bm R_{\mathc I} - \gamma \bm E_N, \X^0_{\mathc I} \rangle - \langle \bm K_{\mathc I}-\gamma \bm E_\smalln, \X^0_{\mathc I} \rangle + \langle \bm K_{\mathc I} - \bm R_{\mathc I},\hat{\X}_{\mathc I} \rangle\\
&= \langle \bm R_{\mathc I}-\bm K_{\mathc I},\X^0_{\mathc I}\rangle + \langle \bm K_{\mathc I}-\bm R_{\mathc I},\hat{\X}_{\mathc I}\rangle \\
&\stackrel{(ii)}{\leq} 2 \lVert \bm K_{\mathc I}-\bm R_{\mathc I} \rVert_{1} 
\end{aligned}
\end{equation}
Here, inequality $(i)$ follows from the fact that $\langle \bm K_{\mathc I} - \gamma \bm E_\smalln,\hat{\X}_{\mathc I}\rangle \geq \langle \bm K_{\mathc I} - \gamma \bm E_\smalln,\X^0_{\mathc I}\rangle$, while inequality~$(ii)$ is obtained by noting that $0\leq \htt X_{ij}, X^0_{ij} \leq 1$ for all $i,j$. 
\end{proof}
\section{Proof of \Cref{theorem-ErrorRateX}}
\begin{proof}
 We assume $\bm R_{\mathc I}$, the part of the reference matrix defined for the set of inlier points $\mathc I$ to be of the following form:
\begin{equation}
    R_{ij}= 
    \begin{cases}
    \max\big\{K_{ij}, \exp\big(-\frac{r_{\tin}^2}{\theta^2}\big) \big\} \text{ if } i,j \in \mathc C_k \\
    \min\big\{ K_{ij}, \exp\big(-\frac{{r^{kl}_{\tout}}^2}{\theta^2}\big) \big\}\text{ if } i\in \mathc C_k, j\in \mathc C_l
    \end{cases}
\end{equation}
Here, $r_{\tin}$ and $r^{kl}_{\tout}$ are parameters which we determine in the proof. Let $\tau_{\tin}=\exp{\big(-\frac{{r_{\tin}}^2}{\theta^2}}\big)$ and  $\tau_{\tout}^{(k,l)}=\exp\big(-\frac{{r^{kl}_{\tout}}^2}{\theta^2}\big)$. Therefore, 
\begin{equation}
\label{corruptions-bound}
    \begin{aligned}
        \lVert \bm K_{\mathc I} - \bm R_{\mathc I} \rVert_1 &\leq  \sum_{k\in[\R]} \sum_{i,j\in \mathc C_k} \mathbbm{1}_{\{ K_{ij}<\tau_{\tin}\}}  \tau_{\tin} +\sum_{k\neq l:k,l\in[\R]} \sum_{i\in\mathc C_k,j\in\mathc C_l} \mathbbm{1}_{\{ K_{ij}>\tau_{\tout}^{(k,l)}\}}  (1-\tau_{\tout}^{(k,l)})\\
        &\leq \sum_{k\in[\R]} m_c^{(k,k)}  +\sum_{k\neq l:k,l\in[\R]}  m_c^{(k,l)}
    \end{aligned}
\end{equation}
Here, $ m_c^{(k,k)}$ and $m_c^{(k,l)}$ denote the number of corruptions for the $k$-th diagonal block and $(k,l)$-th off-digonal block respectively. Next, we let $\y_i$ and $\y_j$ be sub-gaussian random vectors with means $\bs \mu_k$ and $\bs \mu_l$ along with their respective sub-gaussian norms $\sigma_k$ and $\sigma_l$. Then, we have
\begin{equation}
    \begin{aligned}
        \lVert \y_i - \y_j \rVert^2 &= \lVert (\bs{\mu}_k+\bs \xi_i) -(\bs\mu_l+\bs \xi_j) \rVert^2\\
        &=\lVert \bs\mu_k-\bs\mu_l\rVert^2 + 2(\bs\mu_k-\bs\mu_l)^\top(\bs\xi-\bs\xi_j)+\lVert\bs\xi_i-\bs\xi_j\rVert^2\\
        &=\Delta_{kl}^2+ 2(\bs\mu_k-\bs\mu_l)^\top(\bs\xi-\bs\xi_j)+\lVert\bs\xi_i-\bs\xi_j\rVert^2
    \end{aligned}
\end{equation}

\noindent Assume that $\eta_{ij}$ represents the noise part in the above term. Thus
\begin{equation}
    \eta_{ij}=2(\bs\mu_k-\bs\mu_l)^\top(\bs\xi_i-\bs\xi_j)+\lVert\bs\xi_i-\bs\xi_j\rVert^2
\end{equation}

\noindent Note that if $k=l$, then the above term reduces to $\lVert \y_i - \y_j \rVert^2 = \eta_{ij}=\lVert\bs\xi_i-\bs\xi_j\rVert^2$.  \\

\noindent In the above expressions, $\bs\xi_i$ and $\bs\xi_j$ are both sub-gaussian random vectors with their respective sub-gaussian norms $\sigma_k$ and $\sigma_l$. Next, we use the fact $\bs\xi_i-\bs\xi_j$ is sub-gaussian with sub-gaussian norm at most $\sqrt{2}\sigma_{\max}$.
\begin{equation}
\begin{aligned}
        \mb E[\exp(\bm c^\top(\bs\xi_i-\bs\xi_j))] &=  \mb E[\exp(\bm c^\top\bs\xi_i)] \mb E[\exp(-\bm c^\top\bs\xi_j)] \\
        &\leq \exp\bigg(\frac{\lVert \bm c \rVert^2 \sigma_k^2}{2}\bigg)\exp\bigg(\frac{\lVert \bm c \rVert^2 \sigma_l^2}{2}\bigg) \\ 
        &\leq \exp\bigg(\frac{\lVert \bm c \rVert^2 2\sigma_{\max}^2}{2}\bigg)
\end{aligned}
\end{equation}
Here, the equality follows from the independence of random variables $\bs \xi_i$ and $\bs \xi_j$ while the first and second inequalities are obtained from the definition of the sub-gaussian norm. Using the concentration inequality from \cite{hsu2012tail} for quadratic forms of sub-gaussian random vectors, we have 
\begin{equation}
    \begin{aligned}
         \mb P\left(\lVert \bs \xi_i - \bs \xi_j \rVert^2 >  2\sigma_{\max}^2(\D+2\sqrt{t\D}+2t)\right) \leq \exp (-t)
    \end{aligned}
\end{equation}

 We take $t=\frac{c^2\Delta_{\min}^2}{\sigma_{\max}^2}$ and  assume that $\Delta_{\min}\geq c'\sigma_{\max}\sqrt{\D}$. Therefore, we get $2\sigma_{\max}^2(\D+2\sqrt{t\D}+2t)=2\sigma_{\max}^2\bigg(\D+2\sqrt{\frac{c^2\Delta_{\min}^2\D}{\sigma_{\max}^2}}+2\frac{c^2\Delta_{\min}^2}{\sigma_{\max}^2}\bigg) \leq 2\sigma_{\max}^2\bigg(\frac{\Delta_{\min}^2}{c'^2\sigma_{\max}^2}+\frac{2c\Delta_{\min}^2}{c'\sigma_{\max}^2}+2\frac{c^2\Delta_{\min}^2}{\sigma_{\max}^2}\bigg)= 2\Delta_{\min}^2\bigg(\frac{1}{c'^2}+\frac{2c}{c'}+2c^2\bigg)$.  Putting $c=\frac{1}{c'}$, we have $\lVert \bs\xi_i-\bs\xi_j\lVert^2\leq 10\frac{\Delta_{\min}^2}{c'^2}$ with probability at least $1-\exp\bigg(-\frac{\Delta_{\min}^2}{c'^2\sigma_{\max}^2}\bigg)$:
\begin{equation}
\label{eq:inner_block_bound}
    \begin{aligned}
         \mb P\bigg(\lVert \bs \xi_i - \bs \xi_j \rVert^2 >  \frac{10}{c'^2}\Delta_{\min}^2\bigg) \leq \exp \bigg(-\frac{\Delta_{\min}^2}{c'^2\sigma_{\max}^2}\bigg)
    \end{aligned}
\end{equation}
Setting $r_{\tin}^2=\frac{10}{c'^2}\Delta_{\min}^2$, we can now easily obtain an upper bound on the probability of a violation in the kernel matrix $\bm K$ for the diagonal block by noting that
\begin{equation}
\label{eq:diagonalub}
    \begin{aligned}
         \mb P( K_{ij}<\tau_{\tin}\lvert i,j \in \mathc C_k) &= \mb P\bigg(\lVert \bs\xi_i - \bs\xi_j \rVert^2 >  \frac{10}{c'^2}\Delta_{\min}^2\bigg)\\
         &\leq \exp \bigg(-\frac{\Delta_{\min}^2}{c'^2\sigma_{\max}^2}\bigg).
    \end{aligned}
\end{equation}
Next, we consider the probability of a corruption on the off-diagonal blocks. For this, we consider the random variable  $(\bs\mu_k-\bs\mu_l)^\top(\bs\xi_i-\bs\xi_j)$. Since $\bs\xi_i-\bs\xi_j$ is a sub-gaussian random vector with sub-gaussian norm at most $\sqrt{2}\sigma_{\max}$, we get 
\begin{equation}
\begin{aligned}
        \mb E[\exp(t(\bs\mu_k-\bs\mu_l)^\top(\bs\xi_i-\bs\xi_j)] &\leq \exp\bigg(\frac{t^2 (2\sigma_{\max}^2 \Delta_{kl}^2)}{2}\bigg).
\end{aligned}
\end{equation}
Thus, $(\bs\mu_k-\bs\mu_l)^\top(\bs\xi_i-\bs\xi_j)$ is a sub-gaussian random variable with variance parameter $\sqrt{2}\sigma_{\max} \Delta_{kl}$. Therefore, we have
\begin{equation}
\label{eq:offdiagonalub}
    \begin{aligned}
     \mb P( K_{ij}>\tau_{\tout}^{(k,l)}\lvert i\in\mathc C_k, j\in \mathc{C}_l)
        &=\mb P(\lVert \y_i -\y_j \rVert^2 < {r^{kl}_{\tout}}^2\lvert i\in\mathc C_k, j\in \mathc{C}_l)\\
        &= \mb P\bigg(\lVert \bs \xi_i - \bs \xi_j \rVert^2 + 2(\bs\mu_k-\bs\mu_l)^\top (\bs\xi_i-\bs\xi_j)<{r^{kl}_{\tout}}^2-\Delta_{kl}^2\bigg)\\
        &\leq \mb P\bigg(2(\bs\mu_k-\bs\mu_l)^\top (\bs\xi_i-\bs\xi_j)<{r^{kl}_{\tout}}^2-\Delta_{kl}^2\bigg)\\
         &= \mb P\bigg(2(\bs\mu_k-\bs\mu_l)^\top (\bs\xi_i-\bs\xi_j)>\Delta_{kl}^2-{r^{kl}_{\tout}}^2\bigg)\\
        &\leq \exp\bigg(-\frac{(\Delta_{kl}^2-{r^{kl}_{\tout}}^2)^2}{16\sigma_{\max}^2\Delta_{kl}^2}\bigg).
    \end{aligned}
\end{equation}

\noindent Next, we let $p_{kk}=\mb P( K_{ij} < \tau_{\tin}\lvert i,j\in\mathc C_k)$. Then, $$\displaystyle U_{kk}=\frac{\sum_{\{(i,j):i,j\in \mathc C_k,i <j \}}\mathbbm{1}_{\{ K_{ij}<\tau_{\tin}\}}}{n_k(n_k-1)/2}$$ is an unbiased estimator for $p_{kk}$. Using Bernstein's inequality for one-sample U-statistic \citep{hoeffding1963probability,arcones1995bernstein}, we have
\begin{equation}
\label{onesamplebound}
    \mb  P( U_{kk}-p_{kk}> t_1) \leq  \exp\bigg(-\frac{(n_{k}/2)t_1^2}{c_1\nu_{kk}+c_2 t_1}\bigg),
\end{equation}
where $\nu_{kk}$ is the variance of the indicator random variable~$B^{(k,k)}_{ij}:=\mathbbm{1}_{\{ K_{ij}<\tau_{\tin}\}}$ where $i,j \in \mathc C_k$, and $c_1,c_2>0$ are constants. Taking $t_1=\max\big\{p_{kk}, \frac{c_3\log n_{\min}}{ n_{\min}}\big\}$ where $c_3=2(c_1+c_2)>0$ and noting that $\nu_{kk} = p_{kk}(1-p_{kk})\leq p_{kk}\leq t_1$, we note that \eqref{onesamplebound} simplifies to 
\begin{equation}
    \mb  P( U_{kk}-p_{kk}> t_1) \leq  \exp\bigg(-\frac{n_{k}t_1}{c_3}\bigg)\leq \exp\bigg(-\frac{n_{k}\log n_{\min}}{n_{\min}}\bigg)\leq \frac{1}{n_{\min}}.
\end{equation}
Therefore, with probability at least $1-\frac{1}{n_{\min}}$, we get
\begin{equation}
\label{corruptions-onesample}
     m_c^{(k,k)}\leq 2\cdot\frac{n_k(n_k-1)}{2}\cdot \bigg(p_{kk} + \max\bigg\{p_{kk},\frac{c_3 \log n_{\min}}{n_{\min}}\bigg\}\bigg) \leq 2\cdot \max\bigg\{p_{kk},\frac{c_3 \log n_{\min}}{n_{\min}}\bigg\}  n_k^2.
\end{equation}

\noindent Similarly, we assume $p_{kl}=\mb P( K_{ij}> \tau_{\tout}^{(k,l)}\lvert i\in\mathc C_k, j\in \mathc C_l)$. Then, we have that $U_{kl}$ defined as below: $$U_{kl}=\frac{\sum_{i\in\mathc C_k, j \in \mathc C_l}\mathbbm{1}_{\{ K_{ij}>\tau_{\tout}^{(k,l)}\}}}{n_k n_l}$$ is a U-statistic for $p_{kl}$. Following the arguments provided in \cite{pitcan2017note} for the proof of~\eqref{onesamplebound} and ideas in \cite[Section 5b]{hoeffding1963probability} used to prove the Hoeffding bound for two-sample U-statistics, we obtain the following Bernstein inequality for our two-sample U-statistic $U_{kl}$:
\begin{equation}
\label{twosamplebound}
    \mb  P( U_{kl}-p_{kl}> t_2) \leq  \exp\bigg(-\frac{\min\{n_k,n_l \}t_2^2}{c_4\nu_{kl}+c_5 t_2}\bigg),
\end{equation}
where $\nu_{kl}$ is the variance for the indicator variable $B^{(k,l)}_{ij}:=\mathbbm{1}_{\{ K_{ij}>\tau_{\tout}^{(k,l)}\}}$ where $i\in \mathc C_k, j\in \mathc C_l$, and $c_4,c_5>0$ are constants. 
Putting $t_2=\max\big\{p_{kl},\frac{2(c_4+c_5) \log n_{\min}}{n_{\min}}\big\}$ and noting that  $\nu_{kl}=p_{kl}(1-p_{kl})\leq p_{kl}\leq t_2$, following the steps similar to~\eqref{corruptions-onesample}, we have that with probability at least~$1-\frac{1}{n_{\min}^2}$,
\begin{equation}
\label{corruptions-twosample}
    m_c^{(k,l)}\leq  \bigg(p_{kl} + \max\bigg\{p_{kl},\frac{2(c_4+c_5)\log n_{\min}}{n_{\min}}\bigg\}\bigg)  n_k n_l \leq 2\max\bigg\{p_{kl},\frac{2(c_4+c_5)\log n_{\min}}{n_{\min}}\bigg\} n_k n_l.
\end{equation}

\noindent Let $\rho_{\min}:=\min\{\upsilon,1-\upsilon\} (\tau_{\tin}-\tau_{\tout})$. Then, by applying union bound, we get that with probability at least $1-r/n_{\min}-r^2/n_{\min}^2$,
\begin{equation}
\label{Xintbound}
    \begin{aligned}
        \lVert \htt \X_{\mathc I} -  \X^0_{\mathc I} \rVert_1
        &\stackrel{(i)}{\leq} \frac{2}{\rho_{\min}}\cdot \lVert \bm K_{\mathc I} - \bm R_{\mathc I} \rVert_1\\  
        &\stackrel{(ii)}{\leq} \frac{2}{\rho_{\min}} \bigg( 2\cdot\sum_{k\in[\R]} \max\bigg\{p_{kk},\frac{c_3 \log n_{\min}}{n_{\min}}\bigg\}  n_k^2  + 4 \cdot \sum_{k>l:k,l\in[\R]} \max\bigg\{p_{kl},\frac{2(c_4+c_5)\log n_{\min}}{n_{\min}}\bigg\} n_k n_l \bigg)\\
        &\leq \frac{4}{\rho_{\min}} \cdot \max_{k,l\in[\R]}\bigg( \max\bigg\{p_{kl},\frac{c_6\log n_{\min}}{n_{\min}}\bigg\}\bigg) \cdot \bigg( \sum_{k\in[\R]}n_k^2  + 2 \cdot \sum_{k>l:k,l\in[\R]}n_k n_l \bigg)\\ 
        &=\frac{4}{\rho_{\min}}\cdot \max_{k,l\in[\R]}\bigg( \max\bigg\{p_{kl},\frac{c_6\log n_{\min}}{n_{\min}}\bigg\}\bigg) \cdot \bigg( \sum_{k\in[\R]}n_k\bigg)^2\\ 
        &=\frac{4n^2}{\rho_{\min}}\cdot \max\bigg\{ \max_{k,l\in[\R]}p_{kl},\frac{c_6\log n_{\min}}{n_{\min}}\bigg\}\\ 
    \end{aligned}
\end{equation}
\noindent In the above equation, $c_6=\max\{c_3,2(c_4+c_5)\}$. We get $(i)$ using the results of \Cref{lem-EstimatedApproximatesTrue} and \Cref{lem-EstimatedAlmostReferenceMaximizer}, while $(ii)$ is obtained by combining the results in \eqref{corruptions-bound}, \eqref{corruptions-onesample} and \eqref{corruptions-twosample}. Next, we note that
\begin{equation}
\label{maxp}
\begin{aligned}
        \max_{k,l\in[\R]} p_{kl} &\stackrel{(iii)}{\leq}  \max\bigg\{\exp \bigg(-\frac{\Delta_{\min}^2}{c'^2\sigma_{\max}^2}\bigg),\max_{k\neq l : k,l\in[\R]} \exp\bigg(-\frac{{(\Delta_{kl}^2-{r^{kl}_{\tout}}^2)}^2}{16\sigma_{\max}^2\Delta_{kl}^2}\bigg) \bigg\}\\
        &\stackrel{(iv)}{\leq} \max \bigg\{\exp \bigg(-\frac{\Delta_{\min}^2}{c'^2\sigma_{\max}^2}\bigg), \exp\bigg(-\frac{\Delta_{\min}^2}{64\sigma_{\max}^2}\bigg)\bigg\}
    \end{aligned}
\end{equation}
\noindent Here, $(iii)$ is obtained by substituting the upperbounds for the violation probabilities $p_{kk}$ and $p_{kl}$ for diagonal and off-diagonal respectively from \eqref{eq:diagonalub} and \eqref{eq:offdiagonalub} under the assumption that~$\Delta_{\min}\geq c'\sigma_{\max}\sqrt{\D}$. We obtain $(iv)$ by assuming~${r^{kl}_{\tout}}^2=\frac{\Delta_{kl}^2}{2}$. Next, we show that $\rho_{\min}>0$ and finite provided $c'$ is appropriately chosen and~$\theta=\Theta(\Delta_{\min})$, i.e., $\theta=\kappa \Delta_{\min}$ for some $\kappa>0$. We first note that $\min\{\upsilon,1-\upsilon\}>0$ for any $0<\upsilon<1$. Thus, for the condition $\rho_{\min}>0$ to hold, it suffices to show that $\tau_{\tin}-\tau_{\tout}>0$.
\begin{equation}
    \begin{aligned}
        \tau_{\tin}-\tau_{\tout}&= \tau_{\tin}-\max_{k\neq l}\tau_{\tout}^{(k,l)}\\ 
        &= \exp{\bigg(-\frac{10\Delta_{\min}^2}{c'^2\theta^2}}\bigg) - \max_{k\neq l} \exp\bigg(-\frac{\Delta_{kl}^2}{2\theta^2}\bigg)\\
        &\geq\exp{\bigg(-\frac{10\Delta_{\min}^2}{c'^2\theta^2}}\bigg) - \exp\bigg(-\frac{\Delta_{\min}^2}{2\theta^2}\bigg)\\
        &\geq \exp\bigg(-\frac{10}{c'^2\kappa^2}\bigg) - \exp\bigg(-\frac{1}{2\kappa^2}\bigg).
    \end{aligned}
\end{equation}
Thus, $\rho_{\min}> 0$  provided $c'^2>20$. Taking $c'^2=64$ and combining the results in \eqref{Xintbound} and \eqref{maxp}, we get
\begin{equation}
\label{eqn-tempXbound}
        \lVert \htt \X_{\mathc I} -  \X^0_{\mathc I} \rVert_1
        \leq \frac{4n^2}{\rho_{\min}}\cdot \max \bigg\{ \exp\bigg(-\frac{\Delta_{\min}^2}{64\sigma_{\max}^2}\bigg),\frac{c_6\log n_{\min}}{n_{\min}}\bigg\}.
\end{equation}

\noindent Setting $C=\max\{\frac{4}{\rho_{\min}},c_6\}$, we obtain the bound for the inlier part of $\htt \X$ in \eqref{equation-ErrorRateInlierX}. Next, using the  bound in \eqref{eqn-tempXbound} and assuming $n>r$, we get that with probability at least $ 1-2r/n_{\min}$,
\begin{equation}
\begin{aligned}
      \frac{\lVert \htt \X -  \X^0 \rVert_1}{\lVert \X^0 \rVert_1}
        &\leq \frac{\lVert \htt \X_{\mathc I} -  \X^0_{\mathc I} \rVert_1}{\lVert \X^0 \rVert_1}+\frac{2m\N}{\lVert \X^0 \rVert_1} \\
        &\stackrel{(i)}{\leq} \frac{4r}{\rho_{\min}} \cdot  \max \bigg\{ \exp\bigg(-\frac{\Delta_{\min}^2}{64\sigma_{\max}^2}\bigg),\frac{c_6\log n_{\min}}{n_{\min}}\bigg\} + \frac{2m(\smalln+m)}{\big(\frac{\smalln^2}{\R}\big)}\\
        &\stackrel{(ii)}{\leq} \frac{4r}{\rho_{\min}}\cdot \max \bigg\{ \exp\bigg(-\frac{\Delta_{\min}^2}{64\sigma_{\max}^2}\bigg),\frac{c_6\log n_{\min}}{n_{\min}}\bigg\} + \frac{4m\R}{\smalln}\\
        &\leq \frac{4r}{\rho_{\min}} \cdot \exp\bigg(-\frac{\Delta_{\min}^2}{64\sigma_{\max}^2}\bigg)  + \frac{4c_6}{\rho_{\min}}\cdot \frac{r\log n_{\min}}{n_{\min}} + \frac{4m\R}{\smalln}\\
        &\leq C' r \exp\bigg(-\frac{\Delta_{\min}^2}{64\sigma_{\max}^2}\bigg)  + C'' r \max\bigg\{\frac{\log n_{\min}}{n_{\min}},\frac{m}{\smalln}\bigg\} 
\end{aligned}
\end{equation}
Here, $C'=\frac{4}{\rho_{\min}}$ and $C''=8\max\big\{\frac{c_6}{\rho_{\min}},1\}$. Inequality $(i)$ is obtained by substituting $\N=\smalln+m$ and using the fact that $\lVert \X^0 \rVert_1=\lVert \X^0_{\mathc I} \rVert_1 = \sum_{k\in[r]} n_k^2 \geq \frac{n^2}{\R}$, while inequality $(ii)$ follows from the assumption that $m<\smalln$. 
\end{proof}

\section{Proof of \Cref{prop-outliers}}
\begin{lemma}
\label{lem-outlier_neighbors}
Let $\mathc O_g:=\{i\in \mathc O: \displaystyle\min_{k\in [r]} \ \lVert \bm y_i -\bmu_k \rVert \geq c' \Delta_{\min} \}$ denote the set of ``good" outlier points. Assume that for any $i\in \mathc O_g$, $\lVert \bm \y_i -\bs \mu_k \rVert \geq r_0= \sqrt{2} \Delta_{\min}$ holds for all $k\in [r]$ . Then, for any $i\in \mathc O_g$, the cardinality of the set of inlier neighboring points  $\mathc N_{\mathc{ I}}{(i)}:=\{j\in \mathc I: \lVert \y_i -\bm y_j\rVert \leq r =\frac{\Delta_{\min}}{\sqrt{2}}\}$ is upper bounded as
\begin{equation}
        \lvert \mathc N_{\mathc I}{(i)} \rvert \leq n \bigg( \exp\bigg(-\frac{\Delta_{\min}^2}{10\sigma_{\max}^2}\bigg) +\sqrt{\frac{\ \log  n_{\min} }{2 n_{\min}}}\bigg)
\end{equation}
with probability at least $1-\frac{r}{n_{\min}}$.
\end{lemma}
\begin{proof} Consider  an inlier point $\bm y_j$ for some $j\in \mathc I$ and a ``good" outlier point $\bm y_i$ where $i\in \mathc{O}_g$. Then, a lower bound on the probability of the event $\{K_{ij}\leq\gamma \}$ can be obtained as below: 
\begin{equation*}
\begin{aligned}
        p_k^{(i)}&:= \mb P (K_{ij}\leq\gamma \lvert i\in \mathc{O}_g, j \in \mathc C_k )\\
        &\stackrel{(i)}{\geq}\mb P_{\bm y_j}(\lVert \y_i -\bm y_j\rVert > r \lvert \ i\in \mathc{O}_g, j \in \mathc C_k )\\
        &\stackrel{(ii)}{\geq} \mb P_{\bm y_j}(\lVert \bm y_j -\bs \mu_k \rVert< \lVert \bm \y_i -\bs \mu_k \rVert - r\lvert i\in \mathc{O}_g, j \in \mathc C_k )\\
        &\stackrel{(iii)}{\geq} \mb P_{\bm y_j}(\lVert \bm y_j -\bs \mu_k \rVert< r_0 - r)\\
        &= 1-\mb P_{\bm y_j}(\lVert \bm y_j -\bs \mu_k \rVert\geq r_0 - r)\\
        &= 1- \mb P_{\bm y_j}\bigg(\lVert \bm y_j -\bs \mu_k \rVert\geq \frac{\Delta_{\min}}{\sqrt{2}}\bigg)\\
        &\stackrel{(iv)}{\geq} 1 - \exp\bigg(-\frac{\Delta_{\min}^2}{10\sigma_{\max}^2}\bigg) 
\end{aligned}
\end{equation*}
Here, $(i)$ is obtained by noting that $\gamma \geq \exp\big(-\frac{\Delta_{\min}^2}{2\theta^2}\big) $, and accordingly setting $r=\frac{\Delta_{\min}}{\sqrt{2}}$, $(ii)$ follows from the triangle inequality, $(iii)$ makes use of the fact that for any $i \in  \mathc{O}_g$, $\displaystyle\min_{k\in[r]} \ \lVert \bm \y_i -\bs \mu_k \rVert\geq r_0$. Finally, $(iv)$ is obtained by applying the tail-bound \citep{hsu2012tail} for $\bm\y_j -\bs \mu_k$, which we note is a sub-gaussian random vector, and substituting $r_0=\frac{2\Delta_{\min}}{\sqrt{2}}$.

Next, we define $\mathc N_k{(i)}$ for each outlier point $i\in \mathc O_g$ as the set of neighbor points of $i$ from cluster $\mathc C_k$, i.e., $\mathc N_k{(i)}:=\{j\in \mathc C_k: \lVert \y_i -\bm y_j\rVert \leq \frac{\Delta_{\min}}{\sqrt{2}}\}$. We obtain an upper bound on the cardinality of $\mathc N_k{(i)}$ by applying Hoeffding's inequality, which gives 
\begin{equation*}
    \mb P (\lvert  \lvert \mathc N_{k}{(i)} \rvert - n_k (1-p_k^{(i)}) \rvert \leq n_k \epsilon) \geq 1-2\exp( -2\epsilon^2 n_k).
\end{equation*}
Therefore, we have that
\begin{equation*}
    \mb P (\lvert \mathc N_{k}{(i)} \rvert\leq  (1-p_k^{(i)}+\epsilon) n_k) \geq 1-2\exp( -2\epsilon^2 n_k)
\end{equation*}
Putting $\epsilon= \sqrt{\frac{\log n_k }{2 n_k}}$, we get that with probability at least $1-2\exp( -2\epsilon^2, n_k)=1-\frac{1}{n_k}$
\begin{equation*}
    \lvert \mathc N_{k}{(i)} \rvert \leq  (1-p_k^{(i)}+\epsilon) n_k \leq \bigg( \exp\bigg(-\frac{\Delta_{\min}^2}{10\sigma_{\max}^2}\bigg) +\sqrt{\frac{\log n_k }{2 n_k}}\bigg) n_k .
\end{equation*}

\noindent Therefore, by applying union bound, we get that with probability at least $1-\frac{r}{n_{\min}}$
\begin{equation}
\label{eq:inlier_nb_bound}
\begin{aligned}
    \lvert \mathc N_{\mathc I}{(i)} \rvert =\sum_{k\in[r]}\lvert \mathc N_{k}{(i)} \rvert &\leq \sum_{k\in[r]} \bigg( \exp\bigg(-\frac{\Delta_{\min}^2}{10\sigma_{\max}^2}\bigg) +\sqrt{\frac{\log n_k }{2 n_k}}\bigg) n_k\\
    &\leq n \bigg( \exp\bigg(-\frac{\Delta_{\min}^2}{10\sigma_{\max}^2}\bigg) +\sqrt{\frac{\ \log n_{\min} }{2 n_{\min}}}\bigg) . 
\end{aligned}
\end{equation} 

\end{proof}

\noindent We are now in a position to prove \Cref{prop-outliers}.
\begin{proof}[Proof of \Cref{prop-outliers}.]
From \eqref{equation-ErrorRateInlierX}, we get that with probability at least $1-2r/n_{\min}$,
\begin{equation}
\label{eq:resultrep}
        \lVert \htt \X_{\mathc I} -  \X^0_{\mathc I} \rVert_1
        \leq Cn^2\cdot \max \bigg\{ \exp\bigg(-\frac{\Delta_{\min}^2}{64\sigma_{\max}^2}\bigg),\frac{\log n_{\min}}{n_{\min}}\bigg\}.
\end{equation}
Next, we use the result in~\Cref{lem-outlier_neighbors} and the assumptions in \Cref{prop-outliers} to obtain a bound on the entire $\htt \X$  below:
\begin{equation*}
\begin{aligned}
\lVert \htt \X -  \X^0 \rVert_1&\leq  
        \lVert \htt \X_{\mathc I} -  \X^0_{\mathc I} \rVert_1 + 2 \cdot \lvert \mathc O_g \rvert \cdot \max_{i\in \mathc O_g} (\lvert \mathc N_{\mathc I}{(i)} \rvert + \lvert \mathc N_{\mathc O}{(i)} \rvert) + 2 \cdot \lvert \mathc O_b \rvert \cdot n\\
        &\leq  \lVert \htt \X_{\mathc I} -  \X^0_{\mathc I} \rVert_1 + 2 \cdot  \lvert \mathc O_g \rvert \cdot \max_{i\in \mathc O_g} (\lvert \mathc N_{\mathc I}{(i)} \rvert + \lvert \mathc N_{\mathc O}{(i)} \rvert) + 2 \cdot  \lvert \mathc O_b \rvert \cdot n\\
        &\stackrel{(i)}{\leq} Cn^2\cdot \max \bigg\{ \exp\bigg(-\frac{\Delta_{\min}^2}{64\sigma_{\max}^2}\bigg),\frac{\log n_{\min}}{n_{\min}}\bigg\} + 2 \cdot \lvert \mathc O_g \rvert \cdot n\cdot \bigg( \exp\bigg(-\frac{\Delta_{\min}^2}{10\sigma_{\max}^2}\bigg) +\sqrt{\frac{\ \log n_{\min} }{2 n_{\min}}}\bigg)\\ &+2 \cdot \lvert \mathc O_g \rvert  \cdot o(n) +  2 \cdot  \lvert \mathc O_b \rvert \cdot n \\
            &\leq C' n^2\cdot \max \bigg\{ \exp\bigg(-\frac{\Delta_{\min}^2}{64\sigma_{\max}^2}\bigg),\sqrt{\frac{\ \log n_{\min} }{n_{\min}}}\bigg\} +2 \cdot \lvert \mathc O_g \rvert  \cdot o(n) + 2 \cdot  \lvert \mathc O_b \rvert \cdot n
\end{aligned}
\end{equation*}
Here, inequality $(i)$ holds with probability at least $1-3r/n_{\min}$, and is obtained by applying union bound to the results obtained in eqs. \eqref{eq:inlier_nb_bound} and \eqref{eq:resultrep}. Thus, the relative estimation error can be upper bounded as 
\begin{equation*}
\begin{aligned}
\frac{\lVert \htt \X -  \X^0 \rVert_1}{\lVert\X^0 \rVert_1}
            &\leq C' r\cdot \max \bigg\{ \exp\bigg(-\frac{\Delta_{\min}^2}{64\sigma_{\max}^2}\bigg),\sqrt{\frac{\ \log n_{\min} }{n_{\min}}}\bigg\} + \frac{2r \lvert \mathc O_b \rvert}{n} ,
\end{aligned}
\end{equation*}
where $C'>0$ is a universal constant. 
\end{proof}

\section{Proof of \Cref{theorem-ErrorRateZ}}

\begin{lemma}[Approximate $k$-means bound (\citet{lei2015consistency}, Lemma 5.3)]
\label{lem-Lei}
Define $\mb{M}_{n,r} \subseteq \{0,1\}^{n \times r}$ be the set of membership matrices, such that any element of it has only exactly one 1 on each row. Consider two matrices $\bm V, \hat{\bm V} \in \mb{R}^{n,r}$ such that $\bm V = \bs \Theta^* \bm B^*$ with $\bs\Theta^* \in \mb{M}_{n,r}$,  $\bm B^* \in \mb{R}^{r \times r}$. Let $\mathc G_{k} = \{i:\Theta_{ik}^* = 1\}$, i.e., the points in the $k\text{-}{th}$ cluster induced by $\bs \Theta^*$. Consider the k-means problem

\begin{equation} \label{lem-G5eq1}
    \arg \underset{{\bs\Theta \in \mb{M}_{n,r}, \bm B \in \mb{R}^{r \times r}}}{\min} \lVert \hat{\bm V} - \bs\Theta \bm B \rVert_{\F}^2.
\end{equation}

\noindent Let $(\hat{\bs\Theta}, \hat{\bm B})$ be a $(1+ \epsilon)$ approximate solution to \eqref{lem-G5eq1} for $\epsilon > 0$:

\begin{equation} \label{lem-G5eq2}
    \lVert \hat{\bm V} - \hat{\bs\Theta} \hat{\bm B} \rVert_{\F}^2 \leq (1+ \epsilon) \underset{{\bs \Theta \in \mb{M}_{n,r}, \bm B \in \mb{R}^{r \times r}}}{\min} \lVert \hat{\bm V} - \bs \Theta \bm B \rVert_{\F}^2
\end{equation}
\noindent Let $\bar{\bm V} = \hat{\bs\Theta} \hat{\bm B}$. For any $\delta_k \leq \min_{l \neq k} \lVert \bm b^*_{l} - \bm b^*_{k}\rVert$, define 

\begin{equation} \label{lem-G5eq3}
   \mathc S_k = \Big\{ i \in \mathc G_k: \lVert \bar{\bm v}_{i}  - \bm v_{i} \rVert \geq \frac{\delta_k}{2} \Big\}.
\end{equation}
\noindent Then

\begin{equation} \label{lem-G5eq4}
    \sum_{k=1}^{r} \ \lvert \mathc S_{k} \rvert \delta_{k}^2 \leq 4(4+ 2\epsilon) \lVert \bm V - \hat{\bm V} \rVert_{\F}^2.
\end{equation}

\noindent Moreover, if $(16+8\epsilon) \lVert \bm V - \hat{\bm V} \rVert_{\F}^2 \leq n_k \delta_{k}^2 $ for all $k \in [r]$, then there exists a $r \times r$ permutation matrix~$\bm J$ such that $\hat{\bs\Theta}_{\mathc G*} = \bs\Theta_{\mathc G*} \bm J$, where $\mathc G = \cup_{k=1}^{r} (\mathc G_k \setminus \mathc S_k)$.
\end{lemma}

\begin{theorem}[ Davis-Kahan Theorem
(\citep{yu2014useful}, Theorem 2)] Let $\bSigma,\hat{\bSigma}\in\mathbb{R}^{p\times p}$ be symmetric with eigenvalues $\lambda_1\geq\ldots\geq\lambda_p$ and $\hat{\lambda}_1\geq\ldots\geq\hat{\lambda}_p$ respectively. Fix $1\leq s \leq r \leq p$ and assume that min$(\lambda_{s-1} - \lambda_s, \lambda_{r}-\lambda_{r+1})>0$, where $\lambda_0:=\infty$ and $\lambda_{p+1}:=-\infty$. 
Let $d=r-s+1$, and let $\bm U=(\bm u_s,\bm u_{s+1},\ldots,\bm u_r) \in\mathbb{R}^{p\times d}$ and $ \hat{\bm U}=(\hat{\bm u}_s,\hat{\bm u}_{s+1},\ldots,\hat{\bm u}_r)\in\mathbb{R}^{p\times d}$ have orthonormal columns satisfying $\bSigma \bm u_j=\lambda_j \bm u_j$ and $\hat{\bSigma} \hat{\bm u}_j=\hat{\lambda}_j \hat{\bm u}_j$ for $j=s,s+1,\ldots,r$. Then, there exists an orthogonal matrix $\hat{\bm O}\in\mathbb{R}^{d\times d}$ such that
\begin{equation}
    \lVert \bm U-\hat{\bm U}\hat{\bm O} \rVert_{\F} \leq \frac{2^{3/2}\lVert\hat{\bSigma}-\bSigma\rVert_{\F}}{\min(\lambda_{s-1}-\lambda_{s},\lambda_r-\lambda_{r+1})}
\end{equation}
\end{theorem}

\begin{proof}[Proof of \Cref{theorem-ErrorRateZ}.]
Let $\bm U^0,\htt {\bm U} \in \mathbb{R}^{N \times \R}$ denote the top~$\R$~eigenvectors of $\X^0$ and $\htt \X$ respectively. Then, using the fact that the top~$\R$ eigenvectors of~$\X^0$ are essentially indicator vectors for the~$\R$~clusters with associated eigenvalues that correspond to the cluster cardinalities $\smalln_1, \ldots, \smalln_\R$  in decreasing order, we note that $\bm U^0$ can expressed as follows:
\begin{equation*}
    \bm U^0=\Z^0 \Diag(1/\sqrt{\smalln_1},\ldots,1/\sqrt{\smalln_\R}).
\end{equation*}

\noindent Next, we apply the Davis-Kahan theorem \citep{yu2014useful} to obtain the bound below:
\noindent
\bk
\begin{equation}
    \begin{aligned}
    \lVert \htt \U - \U^0 \ocap\rVert_{\F}^2 & \stackrel{(i)}{\leq} \frac{8 \ \lVert \htt \X - \X^0\rVert_{\F}^2}{(\lambda_\R(\X^0) - \lambda_{\R+1}(\X^0))^2} \stackrel{(ii)}{=} \frac{8 \ \lVert \htt \X - \X^0\rVert_{\F}^2}{\smalln_{\min}^2}\stackrel{(iii)}{\leq} \frac{8 \ \lVert \htt \X - \X^0\rVert_1}{\smalln_{\min}^2}.
    \end{aligned}
\end{equation}
Here, inequality $(i)$ follows from the Davis-Kahan theorem, $(ii)$ is obtained by using the fact that there is an eigengap $\smalln_{\min}$ between the $r$-th and $(r+1)$-th eigenvalues of $\X^0$, and $(iii)$ holds since $0\leq X_{ij}, \htt X_{ij}\leq 1$ for all $i,j$. 

We now obtain a bound on the number of mis-classified data points in each cluster using the result stated in \Cref{lem-Lei}. To obtain our result, we first relate the relevant quantities of interest. We assume that the true clustering matrix $\Z^0$ corresponds to $\bm \Theta^*$, $\bm V=\bm U^0 \bm O$ and $\htt{\bm V}=\htt {\bm U}$. Next, we let $\bm B^*=\Diag(1/\sqrt{\smalln_1},\ldots,1/\sqrt{\smalln_\R})\bm O$. Based on the assumption $\lVert\bm b^*_l-\bm b^*_k\rVert=\sqrt{\frac{1}{\smalln_k}+\frac{1}{n_l} }$. Setting $\delta_k^2 =\frac{1}{n_k} $, we get 
\begin{equation*}
\begin{aligned}
    \sum_{k\in[\R]} \frac{\lvert \mathc S_k \rvert} {\smalln_k}=\sum_{k\in[\R]} \lvert \mathc S_k \rvert \delta_k^2
    \leq 64(2+\epsilon) \frac{\lVert \X^0-\htt \X  \rVert_1}{n_{\min}^2}
\end{aligned}
\end{equation*}
To ensure that $(16+8\epsilon) \lVert \bm V - \hat{\bm V} \rVert_{\F}^2 \leq n_k \delta_{k}^2=1$ for all $k \in [r]$ with high probability, we note that 
\begin{equation}
    \begin{aligned}
        (16+8\epsilon) \lVert \bm V - \hat{\bm V} \rVert_{\F}^2 &= (16+8\epsilon) \lVert \htt \U - \U^0 \ocap\rVert_{\F}^2\\
        &\leq 64 (2+\epsilon) \frac{ \lVert \htt \X - \X^0\rVert_1}{\smalln_{\min}^2}
    \end{aligned}
\end{equation}
Next, from \Cref{theorem-ErrorRateX}, we have that suppose the separation condition $\Delta_{\min}\geq 8\sigma_{\max} \sqrt{\D}$ holds, then with probability at least $1-2r/n_{\min}$
\begin{equation*}
    \begin{aligned}
        \lVert \htt \X - \X^0\rVert_1 &\leq \frac{n^2}{r} \bar{\epsilon}.
    \end{aligned}
\end{equation*}
Here, we define  $\bar{\epsilon} := C r \exp\bigg(-\frac{\Delta_{\min}^2}{64\sigma_{\max}^2}\bigg)  + C' r \max\big\{\frac{\log n_{\min}}{n_{\min}},\frac{m}{\smalln}\big\}$. Next, we require that the following condition holds: 
\begin{equation}
    \begin{aligned}
       \frac{64(2+\epsilon)\bar{\epsilon}}{n_{\min}^2}\frac{n^2}{r}\leq 1.
    \end{aligned}
\end{equation}
This, in turn, ensures that $(16+8\epsilon) \lVert \bm V - \hat{\bm V} \rVert_{\F}^2 \leq 1$ is satisfied with high probability. 
    
\end{proof}

\section{Proof of Remark \ref{rem-ExpConst} }
\label{RemarkProof}
\begin{proof} Using the result from~\eqref{maxp}, we get 
\begin{equation}
\label{maxpNew}
\begin{aligned}
        \max_{k,l\in[\R]} p_{kl} &\stackrel{(iii)}{\leq}  \max\bigg\{\exp \bigg(-\frac{\Delta_{\min}^2}{c'^2\sigma_{\max}^2}\bigg),\max_{k\neq l : k,l\in[\R]} \exp\bigg(-\frac{{(\Delta_{kl}^2-{r^{kl}_{\tout}}^2)}^2}{16\sigma_{\max}^2\Delta_{kl}^2}\bigg) \bigg\}\\
        &\stackrel{(iv)}{=}\max\bigg\{\exp \bigg(-\frac{\Delta_{\min}^2}{c'^2\sigma_{\max}^2}\bigg),\max_{k\neq l : k,l\in[\R]} \exp\bigg(-\frac{(1-c'')^2\Delta_{kl}^2}{16\sigma_{\max}^2}\bigg) \bigg\} \\
        &\leq \max \bigg\{\exp \bigg(-\frac{\Delta_{\min}^2}{c'^2\sigma_{\max}^2}\bigg), \exp\bigg(-\frac{(1-c'')^2\Delta_{\min}^2}{16\sigma_{\max}^2}\bigg)\bigg\}.
    \end{aligned}
\end{equation}
\noindent Here, $(iii)$ is obtained by substituting the upperbounds for the violation probabilities $p_{kk}$ and $p_{kl}$ for diagonal and off-diagonal respectively from \eqref{eq:diagonalub} and \eqref{eq:offdiagonalub} under the assumption that~$\Delta_{\min}\geq c'\sigma_{\max}\sqrt{\D}$. We obtain $(iv)$ by assuming~${r^{kl}_{\tout}}^2=c''\Delta_{kl}^2$ for some $c''\in (0,1)$. Next, we show that $\rho_{\min}>0$ and finite provided $c'$ is appropriately chosen and~$\theta=\Theta(\Delta_{\min})$, i.e., $\theta=\kappa \Delta_{\min}$ for some $\kappa>0$. We first note that $\min\{\upsilon,1-\upsilon\}>0$ for any $0<\upsilon<1$. Thus, for the condition $\rho_{\min}>0$ to hold, it suffices to show that $\tau_{\tin}-\tau_{\tout}>0$. 
\begin{equation}
    \begin{aligned}
        \tau_{\tin}-\tau_{\tout}&= \tau_{\tin}-\max_{k\neq l}\tau_{\tout}^{(k,l)}\\ 
        &= \exp{\bigg(-\frac{10\Delta_{\min}^2}{c'^2\theta^2}}\bigg) - \max_{k\neq l} \exp\bigg(-\frac{c''\Delta_{kl}^2}{\theta^2}\bigg)\\
        &\geq\exp{\bigg(-\frac{10\Delta_{\min}^2}{c'^2\theta^2}}\bigg) - \exp\bigg(-\frac{c''\Delta_{\min}^2}{\theta^2}\bigg)\\
        &\geq \exp\bigg(-\frac{10}{c'^2\kappa^2}\bigg) - \exp\bigg(-\frac{c''}{\kappa^2}\bigg).
    \end{aligned}
\end{equation}
\noindent Thus, $\rho_{\min}> 0$  provided $c'^2>10/c''$. Next, we assume $c'^2$ to be of the form $c'^2=t/c''$ for some $t>10$ and obtain the tightest bound in \eqref{maxpNew} by optimizing over $c'$ and $c''$. We note that subject to the constraints $0<c''<1$ and $c'^2=t/c''$, the optimal $c'^*,c''^*$ satisfy
\begin{equation}
\label{opt_const}
    \begin{aligned}
        c'^*,c''^*=&\arg \min_{c',c''}  \max \bigg\{\exp \bigg(-\frac{\Delta_{\min}^2}{c'^2\sigma_{\max}^2}\bigg), \exp\bigg(-\frac{(1-c'')^2\Delta_{\min}^2}{16\sigma_{\max}^2}\bigg)\bigg\}\\
        =&\arg \min_{c',c''}  \max \bigg\{c'^2, \frac{16}{(1-c'')^2}\bigg\}\\
        =&\arg \max_{c',c''}  \min \bigg\{\frac{1}{c'^2}, \frac{(1-c'')^2}{16}\bigg\}.
    \end{aligned}
\end{equation}
Since $t$ can be arbitrarily close to 10, we set  $c'^2=10/c''$ in \eqref{opt_const} and analytically solve for the following one-dimensional problem in $c''$, which yields $c''^*=\frac{9-2\sqrt{14}}{5}\geq 0.3033$. Putting $c''=0.3033$, we get ${c'^*}^2\approx 32.97$. Taking $c'^2=33$ and combining the results in \eqref{Xintbound} and \eqref{maxpNew}, we get
\begin{equation}
\label{eqn-tempXNewBound}
        \lVert \htt \X_{\mathc I} -  \X^0_{\mathc I} \rVert_1
        \leq \frac{4n^2}{\rho_{\min}}\cdot \max \bigg\{ \exp\bigg(-\frac{\Delta_{\min}^2}{33\sigma_{\max}^2}\bigg),\frac{c_6\log n_{\min}}{n_{\min}}\bigg\}.
\end{equation}
\noindent Setting $C=\max\{\frac{4}{\rho_{\min}},c_6\}$, we obtain the bound for the inlier part of $\htt \X$ in \eqref{equation-ErrorRateInlierX}. Next, using the  bound in \eqref{eqn-tempXNewBound} and assuming $n>r$, we get that with probability at least $ 1-2r/n_{\min}$,
\begin{equation}
\begin{aligned}
      \frac{\lVert \htt \X -  \X^0 \rVert_1}{\lVert \X^0 \rVert_1}
        &\leq \frac{\lVert \htt \X_{\mathc I} -  \X^0_{\mathc I} \rVert_1}{\lVert \X^0 \rVert_1}+\frac{2m\N}{\lVert \X^0 \rVert_1} \\
        &\stackrel{(i)}{\leq} \frac{4r}{\rho_{\min}} \cdot  \max \bigg\{ \exp\bigg(-\frac{\Delta_{\min}^2}{33\sigma_{\max}^2}\bigg),\frac{c_6\log n_{\min}}{n_{\min}}\bigg\} + \frac{2m(\smalln+m)}{\big(\frac{\smalln^2}{\R}\big)}\\
        &\stackrel{(ii)}{\leq} \frac{4r}{\rho_{\min}}\cdot \max \bigg\{ \exp\bigg(-\frac{\Delta_{\min}^2}{33\sigma_{\max}^2}\bigg),\frac{c_6\log n_{\min}}{n_{\min}}\bigg\} + \frac{4m\R}{\smalln}\\
        &\leq \frac{4r}{\rho_{\min}} \cdot \exp\bigg(-\frac{\Delta_{\min}^2}{33\sigma_{\max}^2}\bigg)  + \frac{4c_6}{\rho_{\min}}\cdot \frac{r\log n_{\min}}{n_{\min}} + \frac{4m\R}{\smalln}\\
        &\leq C' r \exp\bigg(-\frac{\Delta_{\min}^2}{33\sigma_{\max}^2}\bigg)  + C'' r \max\bigg\{\frac{\log n_{\min}}{n_{\min}},\frac{m}{\smalln}\bigg\} 
\end{aligned}
\end{equation}
Here, $C'=\frac{4}{\rho_{\min}}$ and $C''=8\max\big\{\frac{c_6}{\rho_{\min}},1\}$. Inequality $(i)$ is obtained by substituting $\N=\smalln+m$ and using the fact that $\lVert \X^0 \rVert_1=\lVert \X^0_{\mathc I} \rVert_1 = \sum_{k\in[r]} n_k^2 \geq \frac{n^2}{\R}$, while inequality $(ii)$ follows from the assumption that $m<\smalln$. 
\end{proof}

\section{Proof of \Cref{lem-projection}}
Before we prove \Cref{lem-projection}, we first prove two lemmas, which yield useful results under the assumption that the inlier part of the data is centered at the origin. 
\begin{lemma}
\label{mu_bound}
If $\sum_{k\in[\R]} \pi_k \bs \mu_k = 0$, then $\lVert \bs \mu_k \rVert \leq \Delta_{\max}$ for all $k\in[\R]$. 
\end{lemma}
\begin{proof}
For any fixed $k$, we let $g(k)=\arg \min_{l\neq k} \big\langle \bs \mu_l, \frac{\bs \mu_k}{\lVert \bs \mu_k \rVert} \big\rangle $. Therefore, we have that
\begin{equation}
    \bigg\langle \bs \mu_{g(k)}, \frac{\bs \mu_k}{\lVert \bs \mu_k \rVert} \bigg\rangle=\min_{l\neq k} \bigg\langle  \bs \mu_l, \frac{\bs \mu_k}{\lVert \bs \mu_k \rVert} \bigg\rangle \leq \sum_{l\in[\R]} \pi_l \bigg\langle \bs \mu_l, \frac{\bs \mu_k}{\lVert \bs \mu_k \rVert} \bigg\rangle =\bigg\langle \sum_{l\in[\R]} \pi_l\bs \mu_l, \frac{\bs \mu_k}{\lVert \bs \mu_k \rVert} \bigg\rangle =0. 
\end{equation}
Thus, for every $k\in[\R]$, we have that 
\begin{equation}
    \Delta_{\max} \geq \lVert \bs \mu_k -\bs \mu_{g(k)} \rVert \geq \bigg\langle \bs \mu_k - \bs \mu_{g(k)}, \frac{\bs \mu_k}{\lVert \bs \mu_k \rVert}\bigg\rangle \geq \lVert \bs \mu_k \rVert.
\end{equation} 
\end{proof}

\begin{lemma}
\label{lem-meanbound}
If $\sum_{k\in[\R]} \pi_k \bs \mu_k = 0$, then $ \lVert \bs \Sigma\rVert \leq 2 \Delta_{\max}^2$. 
\end{lemma}
\begin{proof}
It can be easily worked out that when the mean $\bs \mu$ for mixture of sub-gaussians is assumed to be at the origin, i.e., when $\sum_k \pi_k \bs \mu_k=0$, the covariance matrix $\bs \Sigma$ can be expressed as follows:

\begin{equation}
\begin{aligned}
    &&\bs \Sigma&= \sum_{k\in[\R]} \pi_k \bs \mu_k \bs \mu_k^\top +\sum_{k\in[\R]} \pi_k \bSigma_k  \\
\end{aligned}
\end{equation}

Using the above expression for $\bSigma$, we obtain a bound on $\lVert \bSigma \rVert_2$ as follows:
\begin{equation}
\begin{aligned}
    \lVert \bSigma \rVert_2 &\leq  \big\lVert \sum_{k\in[\R]} \pi_k \bs \mu_k  \bmu_k^\top \big\rVert_2+ \big \lVert\sum_{k\in[\R]} \pi_k \bSigma_k \big \rVert_2\\
    &\leq \sum_{k\in[\R]} \pi_k \lVert \bs \mu_k\rVert_2^2 +\sum_{k\in[\R]}\pi_k\lVert \bSigma_k \rVert_2 \\
    &= \sum_{k\in[\R]} \pi_k (\lVert \bs \mu_k\rVert_2^2 +\sigma_k^2) \\
    &\leq \Delta_{\max}^2 +\sigma_{\max}^2\leq 2\Delta_{\max}^2.
\end{aligned}
\end{equation}
Here, the first inequality is obtained using triangle inequality, and the second inequality follows from the convexity of norms and using \Cref{lem-meanbound}. 
\end{proof}

Next, we show that provided the number of outliers $m$ is small relative to the number of inlier points $N$, the operator norm $\lVert \htt {\bs \Sigma} - \bs\Sigma \rVert_2$ is also small. 
\begin{lemma}
\label{lem-CovarianceEst}
Let $\bm Y'\in \mb R^{N' \times d}$ denote the portion of the data matrix $\bm Y$ obtained by sampling $N'$ points randomly without replacement from $\bm Y$. Suppose ${\bs\Sigma}$ and $\htt {\bs\Sigma}$ denote respectively the true covariance matrix for the SGMM (without outliers) and the sample covariance matrix for $\bm Y'$. Then, assuming $\sum_k \pi_k \bs\mu_k = 0$, we have 
$$\lVert \htt {\bs\Sigma} - \bs\Sigma \rVert_2 \leq C_1\sqrt{ \frac{2\D N\log N'}{n N'}} + C_2\bigg( \frac{m}{N}+\sqrt{\frac{\log N'}{N'}}\bigg) \max\bigg\{\Delta_{\max}^2, \lVert \Y^{\mathc{O'}}\rVert^2_{2,\infty}\bigg\}$$
with probability at least $1-O(N'^{-1})$.
\end{lemma}

\begin{proof}
Let $\mathc{I'}$ and $\mathc{O'}$ denote respectively the sets of indices that correspond to the inlier and outlier data points in the $N'$ randomly sampled set of points for dimensionality reduction with their respective cardinalities $n'$ and $m'$. Assume that $\bm Y'\in \mb R^{N' \times d}$ denotes the portion of the data matrix $\bm Y$ that corresponds to these $N'$ data points, while $\Y^{\mathc{I'}}$ and $\Y^{\mathc{O'}}$ denote the inlier and outlier parts of the data matrix with their respective sample means represented by $\overline{\y}^{\mathc{I'}}$ and $\overline{\y}^{\mathc{O'}}$. Then, we note that the sample covariance matrix~$\htt \bSigma$ can be expressed as below:
\begin{equation*}
\label{estimated_covariance}
    \begin{aligned}
        &\htt {\bs \Sigma}\\ &=\frac{1}{n'+m'}\sum_{i\in \mathc{I'} \cup \mathc{O'}} (\y_i-\overline{\bm y})(\y_i-\overline{\bm y})^\top \\
        &=\frac{1}{n'+m'} {\bm Y'}^\top \bm Y' - \overline{\bm y}\overline{\bm y}^\top \\
        &=\frac{1}{n'+m'} {(\Y^{\mathc{I'}}}^\top\Y^{\mathc{I'}} +{\Y^{\mathc{O'}}}^\top\Y^{\mathc{O'}}) - \bigg(\frac{n'}{n'+m'} \overline{\y}^{\mathc{I'}}+\frac{m'}{n'+m'} \overline{\y}^{\mathc{O'}}\bigg) \bigg(\frac{n'}{n'+m'} \overline{\y}^{\mathc{I'}}+\frac{m'}{n'+m'} \overline{\y}^{\mathc{O'}}\bigg)^\top\\
        &=\frac{1}{n'+m'} {(\Y^{\mathc{I'}}}^\top\Y^{\mathc{I'}} +{\Y^{\mathc{O'}}}^\top\Y^{\mathc{O'}}) - \bigg(\frac{n'}{n'+m'} \bigg)^2 {\overline{\y}^{\mathc{I'}} \overline{\y}^{\mathc{I'}}}^\top-\bigg(\frac{m'}{n'+m'}\bigg)^2 {\overline{\y}^{\mathc{O'}} \overline{\y}^{\mathc{O'}}}^\top- \frac{m'n'}{(n'+m')^2} {(\overline{\y}^{\mathc{I'}}\overline{\y}^{\mathc{O'}}}^\top+{\overline{\y}^{\mathc{O'}}\overline{\y}^{\mathc{I'}}}^\top)\\
        &=\frac{n'}{n'+m'} \bigg(\frac{1}{n}{\Y^{\mathc{I'}}}^\top \Y^{\mathc{I'}} - {\overline{\y}^{\mathc{I'}} \overline{\y}^{\mathc{I'}}}^\top \bigg)+\bigg(\frac{n'}{n'+m'} \bigg) \bigg({\overline{\y}^{\mathc{I'}} \overline{\y}^{\mathc{I'}}}^\top-\bigg(\frac{n'}{n'+m'} \bigg) {\overline{\y}^{\mathc{I'}} \overline{\y}^{\mathc{I'}}}^\top\bigg) +\frac{1}{n'+m'}{\Y^{\mathc{O'}}}^\top \Y^{\mathc{O'}}\\
        & \ \ \ -\bigg(\frac{m'}{n'+m'}\bigg)^2 {\overline{\y}^{\mathc{O'}} \overline{\y}^{\mathc{O'}}}^\top - \frac{m'n'}{(n'+m')^2} {(\overline{\y}^{\mathc{I'}}\overline{\y}^{\mathc{O'}}}^\top+{\overline{\y}^{\mathc{O'}}\overline{\y}^{\mathc{I'}}}^\top).
    \end{aligned}
\end{equation*}

\noindent From the above expression for $\htt {\bm \Sigma}$, we obtain a bound for $\lVert \bm\Sigma -\htt {\bm\Sigma} \rVert_2$ by using triangle inequality as follows:
\begin{equation*}
    \begin{aligned}
    \lVert \bs\Sigma -\htt {\bs\Sigma} \rVert_2
    &\leq\frac{n'}{n'+m'} \bigg\lVert \bs\Sigma -\bigg(\frac{1}{n}{\Y^{\mathc{I'}}}^\top \Y^{\mathc{I'}} - {\overline{\y}^{\mathc{I'}} \overline{\y}^{\mathc{I'}}}^\top \bigg) \bigg\rVert_2+ \frac{m'}{n'+m'} \lVert \bs \Sigma \rVert_2  + \frac{m'n'}{(n'+m')^2} \big\lVert {\overline{\y}^{\mathc{I'}} \overline{\y}^{\mathc{I'}}}^\top \big\rVert_2\\
    &\quad+\frac{1}{n'+m'}\big\lVert {\Y^{\mathc{O'}}}^\top\Y^{\mathc{O'}} \big\rVert_2+\bigg(\frac{m'}{n'+m'}\bigg)^2 {\big\lVert \overline{\y}^{\mathc{O'}} \overline{\y}^{\mathc{O'}}}^\top \big\rVert_2 + \frac{m'n'}{(n'+m')^2} {\big(\big\lVert \overline{\y}^{\mathc{I'}}\overline{\y}^{\mathc{O'}}}^\top \big\rVert_2+{\big\lVert\overline{\y}^{\mathc{O'}}\overline{\y}^{\mathc{I'}}}^\top\big\rVert_2\big).
\end{aligned}
\end{equation*}

\noindent Next, we note that $\htt {\bs\Sigma}_{\mathc{I'}}=\frac{1}{n}{\Y^{\mathc{I'}}}^\top\Y^{\mathc{I'}} - {\overline{\y}^{\mathc{I'}} \overline{\y}^{\mathc{I'}}}^\top$ is the sample covariance matrix for the set of inlier points. Using Lemma 7 in \cite{yan2016convex}, we have that  with probability at least $1-O(n'^{-\D})$, $\lVert\bm\Sigma -\htt{\bm\Sigma}_{\mathc{I'}}\rVert_2 \leq C_1\sqrt{\frac{\D \log n'}{n'}}$, where $C_1$ is some constant. 
\begin{equation}
\label{CovarianceEstimation}
\begin{aligned}
    \lVert \bs\Sigma -\htt {\bs \Sigma} \rVert_2  &\leq  C_1\sqrt{ \frac{\D \log n'}{n'}} + \frac{m'}{N'} \lVert \bs\Sigma \rVert_2 + \frac{m'}{N'} \lVert \overline{\y}^\mathc{I'}\rVert_2^2 +\frac{1}{\N'} \lVert \Y^{\mathc{O'}}\rVert_2^2 +\frac{m'^2}{\N'^2} \lVert \overline{\y}^\mathc{O'}\rVert_2^2 + \frac{2m'}{\N'} \lVert \overline{\y}^\mathc{I} \rVert_2 \lVert \overline{\y}^\mathc{O} \rVert_2 \\
    &\leq C_1\sqrt{ \frac{\D \log n'}{n'}} + \frac{m'}{N'} \lVert {\bs \Sigma} \rVert_2 + \frac{m'}{N'} \frac{\lVert \Y^\mathc{I'}\rVert_2^2}{n'} +\frac{1}{\N'} \lVert \Y^{\mathc{O'}}\rVert_2^2 +\frac{m'^2}{\N'^2} \frac{\lVert \Y^\mathc{O'}\rVert_2^2}{m'} + \frac{2m'}{\N'} \frac{\lVert \Y^\mathc{I} \rVert_2 \lVert \Y^\mathc{O} \rVert_2}{\sqrt{\smalln' m'}}\\
     &= C_1\sqrt{ \frac{\D \log n'}{n'}} + \frac{m'}{N'} \lVert \bs\Sigma \rVert_2 + \frac{m'}{\N'\smalln'} \lVert \Y^\mathc{I'}\rVert_2^2+\frac{1}{\N'} \lVert \Y^{\mathc{O'}}\rVert_2^2 +\frac{m'}{\N'^2} \lVert \Y^\mathc{O'}\rVert_2^2 + \frac{2\sqrt{m'}}{\N'\sqrt{\smalln'}} \lVert \Y^\mathc{I} \rVert_2 \lVert \Y^\mathc{O} \rVert_2
    \end{aligned}
\end{equation}

For the set of inliers points $\mathc{I'}$, we note that the data matrix ${\Y^{\mathc{I'}}}=\bm M'+\bs\Xi'$, where $\bm M'\in \mb R^{\smalln'\times \D}$ denotes the signal part of the data with the $i$-th row-vector corresponding to the mean $\bmu_{\phi_i}$ for the $i$-th data point, and $\bs\Xi\in \mb R^{n\times \D}$ denotes the noise part with its row $\bs \xi_i^\top$ representing the sub-gaussian noise for the $i$-th datapoint. To get a bound on $\lVert {\Y^{\mathc{I'}}}\rVert_2$, we first obtain a bound on $\lVert\bm M\rVert_2$ and $\lVert \bs \Xi \rVert_2$ separately and then apply triangle inequality. 
We note that $\lVert\bm M\rVert_2$ can be bounded as below:
\begin{equation}
\label{tau_bound}
\lVert \bm M' \rVert_2 \leq \lVert \bm M' \rVert_{\F}\leq  \sqrt{\smalln'} \max_{i \in \mathc {I}'} \lVert \bmu_{\phi_i} \rVert \leq \sqrt{\smalln'} \Delta_{\max}.
\end{equation}
Here, the last inequality follows from the bound obtained in Lemma \ref{mu_bound}.  Next, we obtain a high-probability bound for $\lVert \bs \Xi \rVert_2$. For this, we use the result obtained in Corollary 5.39 in \cite{vershynin2010introduction} for the operator norm of a random matrix whose rows consist of independent sub-gaussian isotropic random vectors. However, since any row vector $\bs \xi_i$ of $\bs \Xi$ is not necessarily an isotropic random vector, we first represent it as $\bs \xi_i= \bSigma_{\phi_i}^{1/2} \bar{\bs \xi_i}$ where $\bar{\bs \xi_i}$ is a sub-gaussian isotropic random vector that constitutes the $i$-th row of $\overline{\bs\Xi}$. Using the corollary along with the fact $\lVert \bs\Xi \rVert_2\leq \sigma_{\max} \lVert \overline{\bs\Xi} \rVert_2$, we get that with probability at least $1-2e^{-c_2 n'}$
\begin{equation}
\label{xi_bound}
    \lVert \bs\Xi \rVert_2 \leq (c_1\sqrt{\D}+\sqrt{\smalln'})\sigma_{\max},
\end{equation}
where $c_1$ and $c_2$ are constants that depend on the sub-gaussian norms $\{\sigma_k\}_{k=1}^\R$. In addition, if specifically $\bs \xi_i$ are Gaussian random vectors and  $\bar{\bs \xi_i}$ are standard normal random vectors, then $c_1$ and $c_2$ are constants independent of $\{\sigma_k\}_{k=1}^\R$. 
Combining \eqref{tau_bound} and \eqref{xi_bound}, we get that $\lVert \Y^{\mathc{I'}} \rVert \leq C_2'\sqrt{n} \Delta_{\max}$ for some constant $C_2'>0$. Therefore, from \eqref{CovarianceEstimation}, we get that with probability at least $1-O(n'^{-\D})$
\begin{equation}
\label{eqn:CovarianceEstimationOpBound}
    \begin{aligned}
    \lVert \bs\Sigma -\htt {\bs \Sigma} \rVert_2
    &\leq C_1\sqrt{ \frac{\D \log n'}{n'}} + \frac{2m'}{\N'} \Delta_{\max}^2 + C_2'\frac{m'}{N'} \Delta_{\max}^2+\frac{1}{\N'} \lVert \Y^{\mathc{O'}}\rVert_2^2 +\frac{m'}{\N'^2} \lVert \Y^{\mathc{O'}}\rVert_2^2 + \frac{2\sqrt{m'}}{\N'}  \Delta_{\max} \lVert \Y^{\mathc{O'}}\rVert_2\\
    &\leq C_1\sqrt{ \frac{\D \log n'}{n'}} + C_2''\frac{m'}{\N'} \Delta_{\max}^2+\frac{2}{\N'} \lVert \Y^{\mathc{O'}}\rVert_2^2 + \frac{2\sqrt{m'}}{\N'}  \Delta_{\max} \lVert \Y^{\mathc{O'}}\rVert_2\\
    &\leq C_1\sqrt{ \frac{\D \log n'}{n'}} + C_2\frac{m'}{\N'} \max\bigg\{\Delta_{\max}^2, \lVert \Y^{\mathc{O'}}\rVert^2_{2,\infty}\bigg\},
    \end{aligned}
\end{equation}
\noindent for some constants $C_2, C_2''>0$ . Next, we note that $n'\sim \text{Hypergeometric}(N,n,N')$ with mean $p_nN'$, where $p_n=\frac{n}{N}$. Therefore, using the tail bound for hypergeometric distribution from \cite{chvatal1979tail}, we have
\begin{equation*}
    \mb P(n' \leq  (p_n-\epsilon) N') \leq \exp(- 2\epsilon^2 N').
\end{equation*}
Thus, we get that $n' \geq \frac{nN'}{2N}$ with probability at least  $1-\exp(- p_n^2 N'/2)$. Similarly, $m'~\sim~\text{Hypergeometric}(N,m,N')$ with mean $p_mN'$, where $p_m=\frac{m}{N}$. Again, using the tail bounds from \citet{chvatal1979tail}, we obtain that
\begin{equation*}
    \mb P(m' \geq  (p_m+\epsilon) N') \leq \exp(- 2\epsilon^2 N').
\end{equation*}
Therefore, we have $m' \leq \frac{mN'}{N}+\sqrt{\frac{\log N'}{N'}}N'\leq 2\max\bigg\{\frac{m}{N},\sqrt{\frac{\log N'}{N'}}\bigg\}N'$ with probability at least  $1-\frac{1}{N'^2}$. 

Next, we note that $\frac{\log N'}{N'}$ is a monotonically decreasing function in $N'$. Therefore, to obtain an upper bound on the right hand side of \eqref{eqn:CovarianceEstimationOpBound}, we use the fact that $n' \geq \frac{nN'}{2N}$  and $\frac{m'}{N'}\leq \frac{m}{N}+\sqrt{\frac{\log N'}{N'}}$ with probability at least $1-\frac{1}{N'^2}-\exp(- p_n^2 N'/2)$. Combining this with the result obtained from \eqref{eqn:CovarianceEstimationOpBound}, we get that probability at least $1-O(N'^{-1})$
\begin{equation}
\label{eqn:CovarianceEstimationOpBound2}
    \begin{aligned}
    \lVert \bs\Sigma -\htt {\bs \Sigma} \rVert_2
    &\leq C_1\sqrt{ \frac{2\D \log N'}{p_n N'}} + C_2\bigg( \frac{m}{N}+\sqrt{\frac{\log N'}{N'}}\bigg) \max\bigg\{\Delta_{\max}^2, \lVert \Y^{\mathc{O'}}\rVert^2_{2,\infty}\bigg\}.
    \end{aligned}
\end{equation} 
 
\end{proof}

\begin{proof}[Proof of \Cref{lem-projection}.] We follow the approach discussed in Lemma 8 and Lemma 9 in \cite{yan2016convex} to obtain the final result in \Cref{lem-projection} by setting $R=S-\htt S$ with $\lVert R \rVert \leq \epsilon = C_1\sqrt{ \frac{2\D N\log N'}{n N'}} + C_2\bigg( \frac{m}{N}+\sqrt{\frac{\log N'}{N'}}\bigg) \max\bigg\{\Delta_{\max}^2, \lVert \Y^{\mathc{O'}}\rVert^2_{2,\infty}\bigg\}$ where $C_1$ and $C_2$ are constants as derived in \eqref{eqn:CovarianceEstimationOpBound2} in  \Cref{lem-CovarianceEst}.  
\end{proof}

\section{Extension to weakly separated clusters}
\begin{proof}[Proof of \Cref{proposition-OverlapCluster}.]
Based on the definition of the new reference matrix in \eqref{eq:NewReference}, we note that the solution~$\tilde{\X}$ obtained from the reference optimization problem has the same form as specified in~\eqref{eq:Xtilde}, where the weakly separated clusters form a single merged cluster. From \eqref{eq:X_diff_bound} and \eqref{eq:ref_inner_prod_bound}, we have
\begin{equation}
\lVert \htt \X_{\mathc I} -  \tilde{\X}_{\mathc I} \rVert_1 \leq \frac{\langle \bm R - \gamma \bm E_N, \tilde{\X}-\hat{\X}\rangle}{\min( R^{\tin}_{\min}-\gamma,\gamma- R^{\tout}_{\max})}\leq \frac{2\cdot\lVert \bm K_{\mathc I}-\bm R_{\mathc I}\rVert_1}{\min\{\upsilon,1-\upsilon\}(\tau_{\tin}-\tau_{\tout})}.
\end{equation}
As before, we obtain a high-probability bound on $\lVert \bm K_{\mathc I}-\bm R_{\mathc I}\rVert_1$ in terms of the number of corrupted entries on the diagonal and off-diagonal blocks. We note that equations~\eqref{eq:diagonalub}~and~\eqref{eq:offdiagonalub} still hold when~$\Delta_{\min}$ is replaced by~$\tilde{\Delta}_{\min}$ (defined in the statement of the theorem), and thus, provide us with upper bounds on probabilities of corruptions $p_{kk}$ and $p_{kl}$ defined respectively for the kernel entries $K_{ij}$ on the $k$-th diagonal block and $(k,l)$-th off-diagonal block where $(k,l)\in\mathc S_{\ov}^c$. Additionally, we obtain bounds on the number of corrupted entries $m_c^{(k,k)}$ and $m_c^{(k,l)}$ for $(k,l)\in \mathc S_{\ov}^c$ on the diagonal and off-diagonal blocks of well separated clusters from equations \eqref{corruptions-onesample} and \eqref{corruptions-twosample} respectively. 

Next, we define the probability of a corrupted kernel entry for a block of weakly separated clusters as $p_{kl}:=\mb P( K_{ij} < \tau_{\tin}\lvert i\in\mathc C_k, j\in \mathc C_l)$ where $(k,l)\in \mathc S_{\ov}$. Consider $i\in \mathc C_k$ and $j\in \mathc C_l $ for $(k,l)\in \mathc S_{\ov}$. Then, we note that  
\begin{equation}
\begin{aligned}
    \mb P(K_{ij} <\tau_{\tin}\lvert i \in \mathc{C}_k, j\in \mathcal{C}_l)&=\mb P(\lVert \bm y_i-\bm y_j \rVert^2>r_{\tin}^2\lvert i \in \mathc{C}_k, j\in \mathcal{C}_l)\\
    &= \mb P(\lVert \bs \xi_i-\bs \xi_j \rVert^2+2(\bs \mu_k -\bs \mu_l)^\top (\bs \xi_i-\bs \xi_j) + \lVert\bs \mu_k -\bs \mu_l \rVert^2 >r_{\tin}^2\lvert i \in \mathc{C}_k, j\in \mathcal{C}_l)\\
    &\stackrel{(i)}{\leq}  \mb P(\lVert \bs \xi_i-\bs \xi_j \rVert^2+2 \Delta_{kl} \lVert\bs \xi_i-\bs \xi_j\rVert + \Delta_{kl}^2 >r_{\tin}^2\lvert i \in \mathc{C}_k, j\in \mathcal{C}_l)\\
    &=\mb P((\lVert \bs \xi_i-\bs \xi_j \rVert + \Delta_{kl})^2 >r_{\tin}^2\lvert i \in \mathc{C}_k, j\in \mathcal{C}_l)\\ 
    &=\mb P(\lVert \bs \xi_i-\bs \xi_j \rVert >r_{\tin} - \Delta_{kl}\lvert i \in \mathc{C}_k, j\in \mathcal{C}_l)\\
    &\stackrel{(ii)}{=} \mb P(\lVert \bs \xi_i-\bs \xi_j \rVert^2 > (r_{\tin} - \Delta_{kl})^2\lvert i \in \mathc{C}_k, j\in \mathcal{C}_l)\\
    &\stackrel{(iii)}{\leq } \exp \bigg(-\frac{(r_{\tin} - \Delta_{kl})^2}{10\sigma_{\max}^2} \bigg)
\end{aligned}
\end{equation}
Here, $(i)$ follows from the application of Cauchy-Schwartz inequality, whereas $(ii)$ holds true under the assumption that $\Delta_{kl}\leq r_{\tin}=\frac{\sqrt{10}}{8} \tilde{\Delta}_{\min}$. We obtain the final bound in $(iii)$. Next, we obtain a bound on the number of corruptions on the  off-diagonal blocks for weakly separated clusters. For this, we note that $U_{kl}$ defined for $(k,l)\in \mathc S_{\ov}$ as below: 
$$U_{kl}=\frac{\sum_{i\in\mathc C_k, j \in \mathc C_l}\mathbbm{1}_{\{ K_{ij} < \tau_{\tin}\}}}{n_k n_l}$$ 
is a U-statistic for $p_{kl}$. Therefore, from \eqref{twosamplebound}, we have 
\begin{equation}
    \mb  P( U_{kl}-p_{kl}> t_2) \leq  \exp\bigg(-\frac{\min\{n_k,n_l \}t_2^2}{c_4\nu_{kl}+c_5 t_2}\bigg),
\end{equation}
where $\nu_{kl}$ is the variance for the indicator variable $B^{(k,l)}_{ij}:=\mathbbm{1}_{\{ K_{ij}<\tau_{\tin}\}}$ where $i\in \mathc C_k, j\in \mathc C_l$ for $(k,l)~\in~\mathc S_{\ov}$, and $c_4,c_5>0$ are constants. Putting $t_2=\max\big\{p_{kl},\frac{2(c_4+c_5) \log n_{\min}}{n_{\min}}\big\}$ and noting that  $\nu_{kl}=p_{kl}(1-p_{kl})\leq p_{kl}\leq t_2$, we follow the steps in~\eqref{corruptions-twosample} to obtain that with probability at least~$1-\frac{1}{n_{\min}^2}$, the number of corruptions $m_{c}^{(k,l)}$ for $(k,l)\in\mathc S_{\ov}$ is bounded as 
\begin{equation}
    m_{c}^{(k,l)}\leq 2\max\bigg\{p_{kl},\frac{2(c_4+c_5)\log n_{\min}}{n_{\min}}\bigg\} n_k n_l,
\end{equation}
where $p_{kl}$ is the corruption probability for block $(k,l) \in \mathc S_{\ov}$. Using the above result and following the steps in \eqref{Xintbound}, we get
\begin{equation}
\label{eq:overlap_in}
    \begin{aligned}
        \lVert \htt \X_{\mathc I} -  \X^0_{\mathc I} \rVert_1
        &\leq \frac{2}{\rho_{\min}}\cdot \lVert \bm K_{\mathc I} - \bm R_{\mathc I} \rVert_1\\ 
        &\leq \frac{2}{\rho_{\min}}\cdot \bigg(\sum_{k\in[r]} m_c^{(k,k)}+ \sum_{(k,l)\in \mathc S_{\ov}} m_{c}^{(k,l)}  + \sum_{(k,l)\in \mathc S_{\ov}^c} m_c^{(k,l)} \bigg)\\
        &\leq \frac{4n^2}{\rho_{\min}}\cdot \max\bigg\{ \max_{k,l\in[\R]}p_{kl},\frac{c_6\log n_{\min}}{n_{\min}}\bigg\}\\ 
    \end{aligned}
\end{equation}
Next, we evaluate the first term within the max expression  
\begin{equation}
\label{eq:overlap_maxp}
\begin{aligned}
        \max_{k,l\in[\R]} p_{kl} &\stackrel{(i)}{\leq }  \max\bigg\{\exp \bigg(-\frac{\tilde{\Delta}_{\min}^2}{c'^2\sigma_{\max}^2}\bigg), \max_{(k,l)\in \mathc S_{\ov}} \exp \bigg(-\frac{(r_{\tin} - \Delta_{kl})^2}{10\sigma_{\max}^2} \bigg),\max_{(k,l)\in \mathc S^c_{\ov}} \exp\bigg(-\frac{{(\Delta_{kl}^2-{r^{kl}_{\tout}}^2)}^2}{16\sigma_{\max}^2\Delta_{kl}^2}\bigg) \bigg\}\\
        &\stackrel{(ii)}{\leq } \max_{(k,l)\in \mathc S_{\ov}} \exp \bigg(-\frac{(r_{\tin} - \Delta_{kl})^2}{10\sigma_{\max}^2} \bigg)\\
        &\stackrel{(iii)}{\leq } \exp\bigg(-\frac{(\tilde{c}\tilde{\Delta}_{\min}-\Delta')^2}{10\sigma_{\max}^2}\bigg)\\
        &=\exp\bigg(-\frac{(\tilde{\Delta}_{\min}-\Delta'/\tilde{c})^2}{64\sigma_{\max}^2}\bigg)
    \end{aligned}
\end{equation}
 Here, inequality $(i)$ is obtained by plugging the bounds from~\eqref{maxp} for $p_{kk}$ and $p_{kl}$ where $(k,l)\in \mathc{S}_\ov^c$. Inequalities $(ii)$ and $(iii)$ follow from the fact that $r_{\tin}=\tilde{c}\tilde{\Delta}_{\min}$ where $\tilde{c}=\frac{\sqrt{10}}{8}$ and $\Delta_{kl}\leq \Delta' \leq \tilde{\Delta}_{\min}$ for~$(k,l)\in \mathc{S_{\ov}}$. Combining the results in \eqref{eq:overlap_in} and \eqref{eq:overlap_maxp}, we obtain the desired result in \eqref{eq:overlap_inlier_thm}. From~\eqref{eq:overlap_maxp}, the result in \eqref{eq:overlap_inlier_thm2} is derived by following the same steps as discussed in the proof of \Cref{theorem-ErrorRateX}. 
 \end{proof}

\end{onehalfspace}

\end{document}